%% file: arxiv_submitted.tex
\newif\ifjmiv    \jmivfalse % \jmivtrue %    %%%%% TEST JULIEN AVEC ICML

\ifjmiv
    \RequirePackage{fix-cm}
    \documentclass[twocolumn]{svjourarxiv}          % twocolumn
\else %%%%% TEST JULIEN AVEC ICML
    \documentclass{article}
    \usepackage[accepted]{icml2013}

    \icmltitlerunning{Convex Histogram-Based Joint Image Segmentation with Regularized Optimal Transport Cost}
    
\fi%%%%%%%%%%%%%%

\usepackage[english]{babel}
\usepackage[utf8]{inputenc}
\usepackage{graphicx}
\usepackage[dvipsnames]{xcolor}
\usepackage{cancel}
\usepackage{amsmath,amsthm} % provide proof
\usepackage{amsfonts} 
\usepackage{amssymb,bbm}
\usepackage{multirow}

\usepackage{gpeyre}
\usepackage{epstopdf}
\usepackage{epsfig}

\usepackage{url}
\usepackage{hyperref}
\hypersetup{
    hypertexnames=true,
    bookmarks=false,         % show bookmarks bar?
    unicode=false,          % non-Latin characters in Acrobat’s bookmarks
    pdftoolbar=true,        % show Acrobat’s toolbar?
    pdfmenubar=true,        % show Acrobat’s menu?
    pdffitwindow=false,     % window fit to page when opened
    pdfstartview={FitH},    % fits the width of the page to the window
    pdftitle={Convex Histogram-Based Joint Image Segmentation with Regularized Optimal Transport Cost},    % title
    pdfauthor={Nicolas Papadakis and  Julien Rabin};     % author
    pdfsubject={Optimal Transport distance for image segmentation},   % subject of the document
    pdfcreator={Nicolas Papadakis and  Julien Rabin},   % creator of the document
    pdfproducer={Nicolas Papadakis and  Julien Rabin}, % producer of the document
    pdfkeywords={color segmentation} {optimal transport} {image processing}, % list of keywords
    pdfnewwindow=true,      % links in new window
    backref=true,               % Permet d'ajouter des liens dans
    pagebackref=true,       % les bibliographies     
    colorlinks=true,       % false: boxed links; true: colored links
    linkcolor=blue,          % color of internal links (change box color with linkbordercolor)
    citecolor=red,        % color of links to bibliography
    filecolor=magenta,      % color of file links
    urlcolor=blue,           % color of external links
}
\usepackage[hyperpageref]{backref} 

\usepackage{bclogo}
\usepackage{graphicx}
\graphicspath{{./figures_small/}}
\newcommand{\U}{\mathbf{1}}
\renewcommand{\O}{\mathbf{0}}
\newcommand \Idc {\mathbbm{1}}
\newcommand{\idc}[1]{\mathbbm{1}_{\{#1\}}}
\newcommand{\R}{\RR}

\newcommand{\MK}{\textbf{MK}}

\newcommand{\cmt}[1]{\emph{\red #1}}

\newcommand{\A}{{\cal A}}
\newcommand{\B}{{\cal B}}
\newcommand{\C}{{\cal C}}
\newcommand{\F}{{\cal F}}
\renewcommand{\H}{{\cal H}}
\renewcommand{\L}{{\cal L}}
\renewcommand{\S}{{\Delta}} % simplex pour des paires d'histogramme, avec limite sur la masse totale
\newcommand{\PS}{{\cal S}} % simplex de probabilité pour un histogramme

\newcommand{\X}{{\cal X}}
\newcommand{\Trp}{{\mathsf{T}}} % for transpose symbol
\newcommand{\Diff}{{D}} % for finite difference matrix
\newcommand{\Dist}{{S}} % for similarity metric between histograms

\newcommand{\lbd}{\lambda}
\newcommand{\sig}{\sigma}

\newcommand{\TV}{\mathop{{TV}}} % \text
\newcommand{\Proj}{{\text{Proj}}} % \mathop \text

\newtheorem{proposition}{Proposition}
\newtheorem{corollary}{Corollary}
\newcommand{\smartqed}{\qed}

\begin{document}

\ifjmiv

    \title{Convex Histogram-Based Joint Image Segmentation\\ with Regularized Optimal Transport Cost%
    }
    \subtitle{}
    
    
    \author{
    	Nicolas Papadakis 
    	\and
            Julien Rabin 
            %etc.
    }
    
    
    \institute{N. Papadakis \at
                  IMB, UMR 5251, Universit\'e de Bordeaux, F-33400 Talence, France\\
                   %IMB, UMR 5251, CNRS, F-33400 Talence, France \\
                  Tel.: +33 (0) 5 40 00 21 16\\
                  Fax: +33 (0) 5 40 00 21 23\\
                  \email{nicolas.papadakis@math.u-bordeaux.fr}           %  \\
               \and
               J. Rabin \at
               Normandie Univ, UNICAEN, ENSICAEN, CNRS, GREYC\\
               14000 Caen, France\\
                    Tel.: +33 (0)  2 31 45 29 23\\
                  Fax: +33 (0) 2 31 45 26 98\\
                   \email{julien.rabin@unicaen.fr}     
    }
    
    \date{Received: date / Accepted: date}
    
    \maketitle

\else
    
    \twocolumn[
    \icmltitle{Convex Histogram-Based Joint Image Segmentation\\ with Regularized Optimal Transport Cost}
    
    % unless you've provided the [accepted] option to the icml2013
    \icmlauthor{Nicolas Papadakis}{nicolas.papadakis@math.u-bordeaux.fr}
    \icmladdress{IMB, UMR 5251, Universit\'e de Bordeaux, F-33400 Talence, France}
    \icmlauthor{Julien Rabin}{julien.rabin@unicaen.fr}
    \icmladdress{Normandie Univ, UNICAEN, ENSICAEN, CNRS, GREYC,
               14000 Caen, France\\}
    
    \icmlkeywords{Optimal transport, Wasserstein distance, Sinkhorn distance, image segmentation, convex optimization}
    
    \vskip 0.3in
    ]
\fi

\begin{abstract}
We investigate in this work a versatile convex framework for multiple image segmentation, relying on the regularized optimal mass transport theory.
In this setting, several transport cost functions are considered and used to match statistical distributions of features.
In practice, global multidimensional histograms are estimated from the segmented image regions, %  corresponding to different labels 
and are compared to referring models that are either fixed histograms given \emph{a priori}, or directly inferred in the non-supervised case. % (co-segmentation).
The different convex problems studied are solved efficiently using primal-dual algorithms.
The proposed approach is generic and enables multi-phase segmentation as well as co-segmentation of multiple images.

\ifjmiv
\fi
\end{abstract}
%%%%%%%%%%%%%%%%%%%%%%%%%%%%%%%%%%%%%%%%%%%%%%%%%%%
\section{Introduction}

\paragraph{Optimal transport in imaging}

Optimal transport theory has received a lot of attention during the last decade as it provides a powerful framework to address problems which embed statistical constraints. In contrast to most distances from information theory (e.g. the Kullback-Leibler divergence), optimal transport takes into account the spatial location of the density mode and define robust distances between empirical distributions.
 The geometric nature of optimal transport, as well as the ability to compute optimal displacements between densities through the corresponding transport map, make this theory progressively mainstream in several applicative fields. 
In image processing, the warping provided by the optimal transport has been used for video restoration~\cite{Delon-midway}, color transfer~\cite{Pitie07}, texture synthesis~\cite{Jungpeyre_wassestein11}, optical nanoscopy~\cite{brune20104d} and medical imaging registration~\cite{AHT04}. It has also been applied to interpolation in computer graphics~\cite{Bonneel-displacement,2015-solomon-siggraph} and surface reconstruction in computational geometry~\cite{digne-reconstruction}. 

The  optimal transport distance has also been successfully used in various image processing and machine learning tasks, image retrieval~\cite{Rubner2000,Pele-ICCV}, image segmentation \cite{bresson_wasserstein09}, image decomposition \cite{TV_Kantorovich_SIAM14} or texture synthesis~\cite{Rabin_mixing_SSVM11}.

Some limitations have been also shown and partially addressed, such as time complexity~\cite{Cuturi13,Bonneel_radon15,Schmitzer_jmiv16}, regularization and relaxation of the transport map~\cite{Ferradans_siims14} for imaging purposes.

\paragraph{Image segmentation}

Image segmentation has been the subject of active research for more than 20 years (see \emph{e.g.}  \cite{Aubert_Kornprobst02,Cremers07} and references therein). 
For instance, we can refer to the seminal work of Mumford and Shah \cite{Mumford_Shah89}, or to its very popular approximation with level sets developed by Chan and Vese in \cite{Chan_Vese01}. This last work provides a very flexible algorithm to segment an image into two homogeneous regions, each one being characterized by its mean gray level value.

In the case of textured images, a lot of extensions of \cite{Chan_Vese01} have been proposed to enhance the mean value image segmentation model by considering other kind of local information. For instance, local histograms are used in \cite{Zhu95,bresson_wasserstein09}, Gabor filters in \cite{Chan_Vese02},  wavelet packets in \cite{Aujol03} and textures are characterized thanks to the structure tensor in \cite{Brox03,Rousson03}. 

Advanced statistical based image segmentation models using first parametric models (such as the mean and variance), and then empirical distributions combined with adapted statistical distances such as the Kullback-Leibler divergence, have been thoroughly studied in the literature. 
One can for instance refer to the works in~\cite{Aubert02,Kim05,Brox06,Herbulot06} that consider the global histograms of the regions to segment and are also based on the Chan and Vese model \cite{Chan_Vese01}. 
%Recent works make use of the Wasserstein distance ~\cite{PeyreWassersteinSeg12} to compare globally the histograms. 
It is important to notice that this class of approaches involves complex shape gradient computations  for the level set evolution equation.
Moreover, as these methods all rely on the evolution of a level set function \cite{Osher88}, it leads to non-convex methods that are sensitive to the initialization choice and only a local minimizer of the associated energy is computed. 

Recently, convexification methods have been proposed to tackle this problem, as in \cite{Chambolle042,Nikolova06,CD08,Berkels09,Pock10,Chambolle11,Brown11,Yildi12}. The original  Chan and Vese model  \cite{Chan_Vese01} can indeed be convexified, and a global solution can be efficiently computed, for instance with a primal-dual algorithm. By means of the coarea formula, a simple thresholding of this global solution provides a global minimizer of the original non-convex problem. 
The multiphase segmentation model based on level sets \cite{Chan_Vese02} can also be treated with convexification methods.
However, the thresholding of the estimated global minima does not anymore ensure to recover a global optimal multiphase segmentation~\cite{Pock09}.
Notice that such approaches have not been developed yet for global histogram segmentation with length boundary regularization.

Other models as in \cite{rother2006,vicente2009joint,ayed2010graph,gorelick2012segmentation,Yuan12}  use graph-based methods and max-flow formulations \cite{punithakumar2012convex} in order to obtain good minima without level-set representation. 
Nevertheless, these approaches are restricted to bin-to-bin distances (for instance $\ell_1$~\cite{rother2006}, Bhattacharyya~\cite{Yuan12}, {$\ell_2$~\cite{TVL2_segmentation} or $\chi^2$~\cite{object_cosegmentation}}) between features' histograms
that are not robust enough to deal with non-uniform quantification or data outliers.

\paragraph{Optimal Transport and image segmentation}
The use of Optimal Transport for image segmentation has been first investigated in~\cite{bresson_wasserstein09} for comparing local $1D$ histograms. In \cite{PeyreWassersteinSeg12,mendoza}, active contours approaches using the Wasserstein distance for comparing global multi-dimensional histograms of the region of interest have  been proposed. Again, these non-convex active contours methods are sensitive to the initial contour. Moreover, their computational cost is very large, even if they include some approximations of the Wasserstein distance as in~\cite{PeyreWassersteinSeg12}. 

In order to deal with global distance between histograms while being independent of any initialization choice,
convex formulations have been designed~\cite{papa_aujol,Swoboda13}. In  \cite{papa_aujol}, a $\ell_1$ norm between cumulative histograms is considered and gives rise to a fast algorithm. This is related to optimal transport only for 1D histograms of grayscale images. The authors of \cite{Swoboda13}  proposed to rely on the Wasserstein distance. In order to be able to optimize the corresponding functional, it requires to make use of sub-iterations to compute the proximity operator of the Wasserstein distance, which use is restricted to low dimensional histograms.
Hence, we considered in \cite{Rabin_ssvm15} a fast and convex approach involving regularization of the optimal transport distance through the entropic regularization of \cite{Cuturi13}.
In this paper we investigate in detail this regularized model and look at its extension to multi-phase segmentation.

\paragraph{Co-segmentation}
As already proposed in \cite{Swoboda13}, the studied convex framework can be extended to deal with the unsupervised co-segmentation of two images. The problem of co-segmentation \cite{Vicente_co-seg_eccv10} consists in segmenting simultaneously multiple images that contain the same object of interest without any prior information.
When the proportion between the size of the object and the size of the image is the same in all images, the model of \cite{Swoboda13} can be applied. 
It aims at finding regions in different images having similar color distributions.
However, this model is not suited for cases where the scale of the object vary. % in the general context,  
In the literature, state-of-the art approaches rely on graph representation. They are able to deal with small scale changes \cite{Rubio} or large ones by considering global information on image subregions \cite{Hochbaum} pre-computed with dedicated algorithms. 
Notice that convex optimization algorithms involving partial duality of the objective functional have been used for co-segmentation based on local features \cite{Joulin}. Such approach is able to deal with scale change of objects but it relies on high dimensional problems scale with $O(N^2)$, where $N$ is the total number of pixels. % $N$ is the sum of all pixels of all images.

The use of robust optimal transport distances within a low dimensional formulation for the global co-segmen\-tation of objects of different scales is thus an open problem that is addressed in this paper. % tackled

\paragraph{Contributions}
The global segmentation models presen\-ted in this paper are  based on the convex formulation for two-phase image segmentation of~\cite{papa_aujol} involving $\ell_1$ distances between histograms.
Following \cite{Swoboda13,Rabin_ssvm15}, we consider the use of Wasserstein distance for global segmentation purposes.
As in \cite{Rabin_ssvm15}, we rely on  the entropic regularization~\cite{Cuturi13,CuturiDoucet14} of optimal transport distances in order to deal with accurate discretizations of histograms.
Hence, this paper shares some common features with the recent work of \cite{CuturiPeyre_smoothdualwasserstein} in which the authors investigate the use of the Legendre-Fenchel transform of regularized transport cost for imaging problems.

With respect to the preliminary version of this work presented in a conference \cite{Rabin_ssvm15}, the contributions of this paper are the following:
\begin{itemize}
\item we give detailed proofs of the computation of the  functions and operators involved by the entropic regularization of optimal transport between non-normalized histograms.
\item we generalize the framework to the case of multi-phase segmentation in order to find a partition of the images with respect to several priors;
\item we provide numerous experiments exhibiting the properties of our framework;
%strengths and the weaknesses of our framework;
\item we extend our model to the co-segmentation of multiple images. Two convex models are proposed. The first one is   able to co-segment an object with constant size in two images for general ground costs. The second one  can deal with different scales of a common object contained in different images for a specific ground cost.
\end{itemize}
This paper is also closely related to the framework proposed in \cite{Swoboda13}. With respect to this method, our contributions are:
\begin{itemize}
\item  the use of regularized optimal transport distances for dealing with high dimensional histograms;
\item  the generalization of the framework to  multi-phase segmentation;
\item  the definition of co-segmentation model for more than $2$ images  dealing with scale changes of objects.
\end{itemize}

\section{Convex histogram-based image segmentation}

\subsection{Notation and definitions}

We consider here vector spaces equipped with the Euclidean inner product $\dotp{.\,}{.}$ and the $\ell_2$ norm
$\lVert . \rVert = \sqrt{\dotp{.\,}{.}}$.
The conjugate linear operator of $A$ is denoted by $A^*$ and satisfies $\langle A.,. \rangle = \langle .,A^* . \rangle$.
We denote as $\U_n \text{ and } \O_n \in \R^n$ the $n$-dimensional vectors filled with ones and zeros respectively, $x^\Trp $ the transpose of $x$, %and $\nabla$ the discrete gradient operator, 
while $\text{Id}$ stands for the identity operator.
The concatenation of the vectors $x$ and $y$ into a vector is denoted $(x;y)$.
Operations and functions on vectors and matrices are meant component-wise, such as inequalities:  
\eq{
	X \le Y \quad \Leftrightarrow \quad X_{ij} \le Y_{ij} \quad \forall \,i,j
}
or exponential and logarithm functions:
\eq{
	\left(\exp X \right)({i,j}) = \exp X_{i,j}
	\qquad
	\log X= \left(\log X_{i,j} \right)_{i,j}
	.
}

We refer to $\ell_p$ norm as ${\lVert x \rVert }_{p} = \left( \sum_i {\lvert x_i\rvert}^p \right)^{\frac1p}$.
The  norm of a linear operator $A$ is 
${\lVert A \rVert } %= \sup_{x \not= \O} \frac{{\lVert Ax \rVert } }{{\lVert x \rVert } } 
= \sup_{\lVert x \rVert= 1} {\lVert Ax \rVert } $.
The operator $\diag(x)$ defines a square matrix whose diagonal is the vector $x$. The identity matrix is $\Id_n = \diag(\U_n)$.
The functions $\Idc_S$ and $\chi_S$ are respectively the indicator and characteristic functions of a set $S$
\begin{equation*}
\begin{split}
\Idc_S(x) = \begin{cases}
	1 & \text{ if } x \in S\\
	0 & \text{ otherwise }
\end{cases}
\!\!,\;
%\quad
% \\
\chi_S(x) = \begin{cases}
	0 & \text{ if } x \in S\\
	\infty & \text{ otherwise }
\end{cases}
.
\end{split}
\end{equation*}
The Kronecker $\delta$ symbol is $\delta_{i,j} = 1$ if $i=j$, and $\delta_{i,j} =0$ otherwise.

A histogram with $n$ bins is a vector $h \in \R^n_+$ with non-negative entries.
The set 
\eql{\label{eq:simplex}
	\PS_{m,n} := \{x \in \R^{n}_+, \dotp{x}{\U_n}=m\}
}
is the simplex of histogram vectors % with $n$ bins
of total mass $m$ ($\PS_{1,n}$ being the $n$-dimensional probability simplex).

The operators $\Prox$ and $\Proj$ stand respectively for the Euclidean proximity and projection operators:%
\begin{equation*}
\begin{split}
\Prox_{f}(x) &= \argmin_{y} \tfrac12\norm{y-x}^2 + f(x)
\\
\Proj_{S}(x) &= \argmin_{y \in S} \norm{y-x} = \Prox_{\;\chi_S}(x)
\;.
\end{split}
\end{equation*}
Functions $f$ for which the proximity operator is known in closed form, or at least that can be evaluated at a given point explicitly, are usually referred to as \emph{simple}.

The Legendre-Fenchel conjugate $f^*$ of a lower semicontinuous convex function $f$ writes $f^*(y) = \sup_{x} \dotp{x}{y} - f(x)$, and satisfies the equality: $f^{**} = f$. 
% Such a transform also satisfies
%We also recall the Moreau's identity 
%$$
%	\Prox_{f} + \Prox_{f^*} = \Id
%	\;.
%$$

%\cmt{pour être consistant, changer tous les sup des problem d'optim en max ?}

\subsection{General formulation of distribution-based image segmentation}
For sake of simplicity, we first describe the binary segmentation problem. The following framework can be extended to multi-phase segmentation, as lately shown in Section~\ref{sec:multiphase}.

Let $I : x \in \Om \mapsto I(x) \in \RR^{d}$ be a multi-dimensional image, defined over the $N$-pixel domain $\Om$ ($N = \lvert \Om\rvert$), and $\F$ a feature-transform into $n$-dimensional descriptors: $\F I(x) \in \RR^n$.
{The border of the domain is denoted $\partial \Om$.}
We would like to define a binary segmentation $u:\Om \mapsto \{0,1\}$ of the whole image domain, using two fixed probability distributions of features $a$ and $b$.
Following the variational model introduced in \cite{papa_aujol}, we consider the energy
\eql{\label{eq:segmentation_energy}
E(u)  = \rho \TV(u) + \Dist( a , h(u) ) + \Dist( b , h(\U-u) )
}
where $\rho\ge 0$ is the regularization parameter, and
\begin{itemize}\setlength\itemsep{0.3em}
\item[\textbullet] the fidelity terms are defined using  ${\Dist} ( .,. )$, a dissimilarity measure between distributions of features;
%= \lVert x-y \rVert_1$, which is the L1 norm (special case of $\MK$ cost when the ground cost function is uniform, except for identical bins: $C(i,j) = \frac{1}{2} (1- \delta(i-j))$);
\item[\textbullet] $h(u)$ is the empirical discrete probability distribution of features $\F I$ using the binary map~$u$, which is written as a sum of Dirac masses % (when not using a kernel estimator)
\eql{\label{eq:dist}
	\!\!
	h(u) : y \in \R^n \mapsto \frac{1}{\sum_{x \in \Om} u(x)} \sum_{x \in \Om} u(x) \delta_{\F I(x)}(y) \,;
}
\item[\textbullet] $\TV(u)$ is the total variation norm of the binary image $u$, which is related to the perimeter of the region $R_1(u) := \{x \in \Om \,\vert\, u(x)=1\}$ by the co-area formula.
\end{itemize}
Observe that this energy is highly non-convex, $h$ being a non linear operator, 
and that we would like to find a minimum $u^\star \in \{0,1\}^N$ over a non-convex set.

\subsection{Convex relaxation of histogram-based segmentation energy}

The authors of \cite{papa_aujol} propose some relaxations and a reformulation in order to handle the minimization of energy \eqref{eq:segmentation_energy} using convex optimization tools.

\subsubsection{Probability map} 
First, it consists in considering solutions from the convex enveloppe of the binary set,
\text{i.e.} using a segmentation variable $u:\Om \mapsto [0,1]$ which can be interpreted as a weight function (probability map). 
%The first relaxation consists in using a segmentation variable $u:\Om \mapsto [0,1]$ which is a weight function (probability map). 
A threshold is therefore required to obtain a binary segmentation of the image into the region corresponding to the prior distribution $a$
\eql{\label{def:region}
	R_t(u) := \{x \in \Om \,\vert\, u(x) \ge t\},
}
its complement $R_{t}(u)^c$ corresponding to prior distribution $b$.
Other post-processing partition techniques may be considered and are discussed later.% {ajouter un mot à ce propos : on ne peux pas optimiser le seuil $t$ facilement car on utilise le transport}. 

%Whereas for

It is worth mentioning that for the specific $\TV$-$\ell_1$ approach of \cite{Nikolova06}, where the dissimilarity measure $\Dist(u,u_0) = \norm{u - u_0}_1$ is the $\ell_1$ distance between the segmentation variable $u$ and a given \emph{prior} {\em binary} segmentation variable $u_0$, %{ that is without using any local feature (operators $h$ and $\cal F$) from the image $I$}, 
such a relaxation still guaranties to find a global solution for the non-convex problem.
However, there is no such a property in our general setting.

\subsubsection{Feature histogram}\label{sec:def_feature_histogram}

Considering the continuous domain of the feature space, as done for instance in \cite{PeyreWassersteinSeg12}, may become numerically intractable for high di\-men\-sional descriptors.
We consider instead histograms, as already proposed in~\cite{bresson_wasserstein09,papa_aujol}.

The feature histogram of the probability map is denoted $H_\X(u)$ and defined as the \textbf{quantized, non-normalized, and weighted histogram} of the feature image $\F I$ using the relaxed variable $u:\Om \mapsto [0,1]$ and a feature set $\X=\{X_i \in \R^n\}_{1\le i\le M_X}$ composed of $M_X$ bins indexed by $i \in \{1,\ldots M_X\}  $
\eql{\label{op:H}
\left(H_\X(u)\right)_i =  \sum_{x \in \Om} u(x) \Idc_{\C_\X(i)} (\F I(x)), 
%\quad
%\forall\, i \in \{1,\ldots M_X\}  
}
%where $i$ a bin index, 
where $X_i$ is the centroid of the corresponding bin $i$, and $\C_\X(i) \subset \RR^n$ 
is the corresponding set of features (\emph{e.g.} the Voronoï cell obtained from 
nearest-neighbor assignment). %\emph{hard assignment} method).
We can write $H_\X$ as a linear operator
\eql{\label{eq:histogram_matrix_def}
H_\X: u \in \R^{N} \mapsto H_\X \cdot u \in \R^{M_X}, % j'ai remplacé \Idc_{\X} par H pour être cohérent avec la suite
}
with matrix notation $H(i,x) := 1$ if $\F I(x) \in \C_\X(i)$ and  $0$ otherwise.
Note that $H_\X \in \RR^{M_X \times N}$ is a fixed \emph{hard assignment} matrix that indicates which pixels of $\F I$ contribute to each bin of the histogram.
As a consequence, we have the property
\begin{equation}\label{eq:histogram_sum}
	\dotp{H_\X \, u}{\U_{N}} % =\sum_{i=1}^{M_X}\sum_{j \in \Om} \Idc_\X(i,j) 
	  = \sum_{x \in \Om} u(x) = \dotp{u}{\U_N},
\end{equation}
so that
$H_\X (u) \in \PS_{\dotp{u}{\U}, M_X}$.
This linear operator computing the histogram of a particular region of the image is illustrated in Figure \ref{illust_histo} for RGB color feature.

\begin{figure}[ht!]
%\begin{center}
\centering
\begin{tabular}{ccc}
	\includegraphics[height=1.3cm]{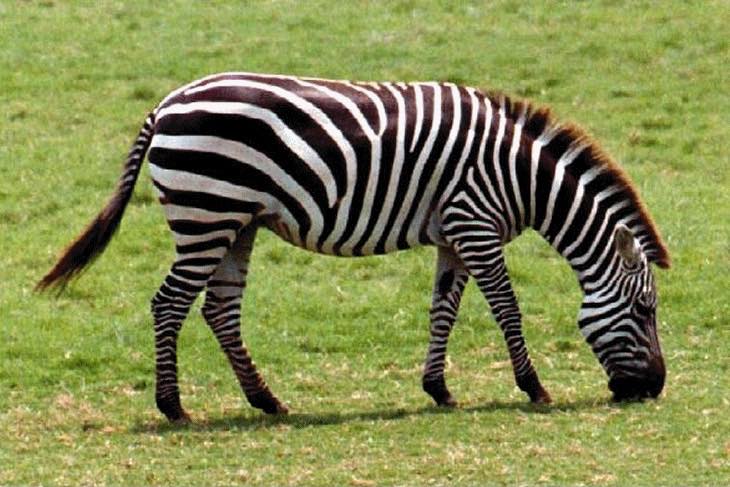}&%
	\includegraphics[height=1.3cm]{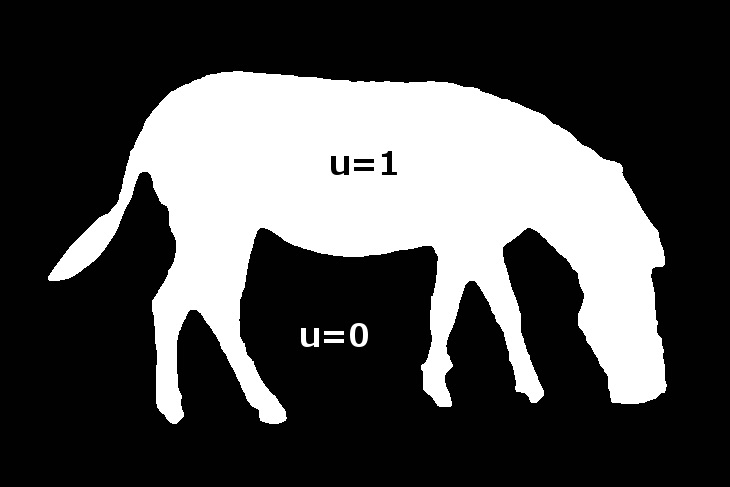}&%
	\hspace{-0.1cm}%
	\includegraphics[height=1.3cm]{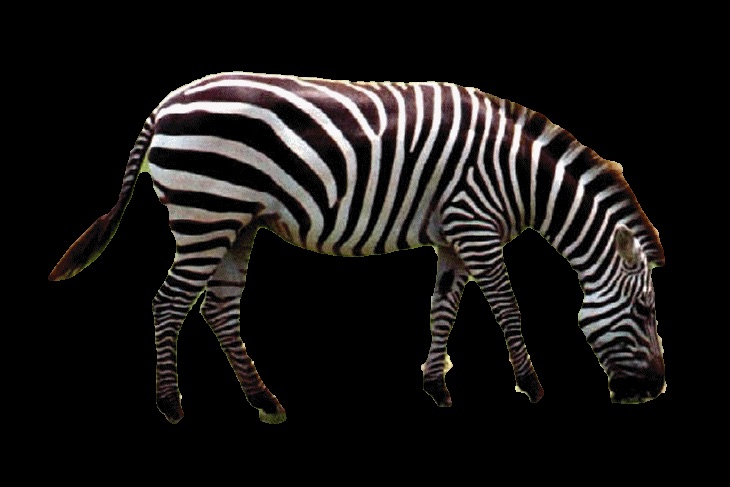}
	\\
	(a) Image $I$&%
	(b) Segmentation  $u$&%
	(c) Region $I\odot u$ %\vspace{0.3cm}
\end{tabular}

\vspace*{4mm}

\centering
{\tiny % \scriptsize % \footnotesize % \small
\noindent
\hspace{-0.4cm}
	 \begin{tabularx}{0.485\textwidth}{ll}
	   \multicolumn{1}{m{5.8cm}}{\scriptsize
	      $ \underbrace{\begin{bmatrix}1&0 &0 & 1 &\cdots & 0\\0&1&0 & 0 &\cdots & 1\\
	       &&&\vdots\\
	       0&0 &0 & 0 &\cdots & 0\\
		0&0 &1 & 0 &\cdots & 0 \end{bmatrix}}_{H_\X}\times\underbrace{\begin{bmatrix}{0}\\{1}\\
		\vdots\\
		{1}\\
	    %   0&0 &0 & 0 &\cdots & 1\\
		{0} \end{bmatrix} }_{{u}}\hspace{-0.05cm}
		=
		\hspace{-0.05cm}\underbrace{\begin{bmatrix}n_1\\n_2\\
	\vdots\\
	    %   0&0 &0 & 0 &\cdots & 1\\
		n_{M-1}\\n_M \end{bmatrix} }_{H_\X \, u}=$}
	\hspace{-0.57cm}
	&
	\multicolumn{1}{m{2.cm}}{
		\includegraphics[height=2.4cm]{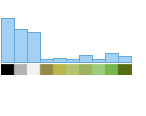}
	}
	\end{tabularx}}

\caption{\textbf{Illustration of the color histogram computed from a binary region.} 
(a) Image $I$. 
(b) Binary segmentation map $u$. 
(c) Corresponding region $I \odot u (x) = I(x) . u(x)$. 
The hard assignment linear operator $H_\X$ encodes the position of each pixel in the clustered color space. 
The histogram value $n_i$ represents here the number of pixels $I(x)$ of the region characterized by $u(x)=1$ that belongs to the feature cluster $C_\X(i)$.
\vspace*{3mm}
}
%\end{center}
\label{illust_histo} 
\end{figure}

\subsubsection{Exemplar histograms}\label{sec:convexification_histogram}
The segmentation is driven by two fixed histograms $a \in \PS_{1,M_a}$ and $b \in \PS_{1,M_b}$, which are normalized (\emph{i.e.} sum to $1$), have respective dimension $M_a$ and $M_b$, and are obtained using the respective sets of features $\A$ and $\B$.
In order to measure the similarity between the non-normalized histogram $H_\A (u)$ and the normalized histogram $a$, while obtaining a convex formulation, we follow \cite{papa_aujol} and consider ${\Dist} \left( {a \dotp{u}{\U_N} , H_\A (u)} \right)$ as  fidelity term,
where the constant vector $a$ has been scaled to $H_\A (u) \in \PS_{\dotp{u}{\U}, M_a}$.

{Note that this approach, based on the comparison of unnormalized histogram pairs as a data fidelity term, %shares some common features with
is also used in~\cite{coseg_hist_matching,coseg_hist_matching_L2} for co-segmentation.
We will further discuss the consequence of such a choice for this problem in the dedicated Section~\ref{sec:coseg}.
}

\subsubsection{Segmentation energy}
% Observe that
Using the previous modifications to formulation~\eqref{eq:segmentation_energy}, the convex segmentation problem can now be written as finding the minimum of the following energy
\begin{equation*}
%\label{eq:convex_segmentation_energy0}
\begin{split}
J(u)  &= \rho \TV(u) 
+ \tfrac{1}{\gamma} \,S \left( {a \dotp{u}{\U_N} , H_\A u} \right)\\ % \sum_{x \in \Om} u(x) 
&\hspace*{12mm} + \tfrac{1}{N-\gamma} \, S\left({b \dotp{\U_N-u}{\U_N}, H_\B (\U_N-u) } \right). % \sum_{x \in \Om} (1-u(x)) 
\end{split}
\end{equation*}
The constant $\gamma \in (0,N)$ is meant to compensate for the fact that 
the binary regions $R_t(u)$ and $R_t(u)^c$ may have different size. 
This model now compares the histograms of the regions with the rescaled reference histograms, 
instead of normalized distributions defined in Eq.~\eqref{eq:dist}.

As we are interested in a discrete probability segmentation map, we consider the following constrained problem:
\eql{\label{modele_seg}%\vspace*{-2mm}
\min_{u \in [0,1]^{N}}  J(u) = 
\min_{u \in \R^{N}}  J(u)  + \chi_{[0,1]^{N}}(u).
}

\subsubsection{Simplification of the setting}

From now on, and without loss of generality, we will assume that all histograms are computed using the same set of features, \emph{namely} $\A=\B$. We will also omit unnecessary subscripts and consider $M_a=M_b=M$ in order to simplify the notation. 
Moreover, we also omit the parameter $\gamma$ since its value seems not to be critical in practice, as demonstrated in \cite{papa_aujol}.

We introduce the linear operators 
\begin{equation}\label{eq:prior_histogram_operator}
	A := a\, \U_{N}^\Trp  \in \R^{M \cdot N}
	\quad \text{ and }  \quad
	B := b\, \U_{N}^\Trp  \in \R^{M \cdot N}
\end{equation}
such that
$Au = (a  \U_{}^\Trp )u  = a \langle u, \U_{} \rangle $, % = a  (\U_{N}^\Trp  . u)
and $\Diff: \R^N \mapsto \R^{2N}$, the finite difference operator on the bi-dimensional cartesian grid $\Om$. 
The gradient at a pixel coordinate $x=(i,j)\in \Om$ is  ${\Diff} u(x) = v(x) = (v_1(x); v_2(x))$ 
where one has $\forall\,x \in \Om \backslash \partial \Om$ (\emph{i.e.} excluding of the domain's border):
\eq{
%\begin{array}{ll}
%	v_1(i,j) & = u(i,j) - u(i-1,j) 
%	\\
%	v_2(i,j) & = u(i,j) - u(i,j-1)
%\end{array}
	v_1(i,j)  = u(i,j) - u(i-1,j)
	, \;
	v_2(i,j) = u(i,j) - u(i,j-1)
	.
}
{On $\partial \Om$, we use homogeneous Dirichlet conditions:
\eq{
	v_1(0,j)  = u(0,j)
	, \;\;
	v_2(i,0) = u(i,0)
	, \quad \forall i,j
}
meaning that a pixel $x$ outside $\Om$ is considered as background ($u(x) = 0$). %, matching histogram $b$.
}
%\cmt{ATTENTION : définition differente de Chambolle 2004}
%
%\cmt{To define the gradient at the image border, we consider $u$ to be null outside the image domain $\Om$. En dire plus ? on change la définition de Chambolle utilisée plus tard (preuve norme K) : $v_1(0,j) = 0$ et $v_2(i,0) = 0$ }
%

The usual discrete definition of the isotropic total variation used in the problems~\eqref{eq:segmentation_energy} and~\eqref{modele_seg} is
%\cmt{ERREUR : c'est la norme $\norm{.}_{2,1}$ }
\eql{\label{eq:TV}
\TV(u) :=  {\lVert {\Diff u }\rVert}_{1,2}  
%= \left (\sum_{x \in \Om} {\lVert \nabla u(x) \rVert}^2_2\right)^{\frac12}
 =  \sum_{x \in \Om} {\lVert \Diff u(x) \rVert}_2
 \;,
}
where the $\ell_{1,2}$ norm for a gradient field $v=(v_1;v_2)$ corresponds to
\eq{
{\norm{v}}_{1,2} 
%= \sum_{x \in \Om_k} {\norm{(v_1(x), v_2(x))}}_2
:= \sum_{x \in \Om} \sqrt{v_1^k(x)^2 + v_2^k(x)^2}.
\vspace*{-2mm}
}

We finally have the following minimization problem:
\eql{\label{eq:convex_segmentation_problem}
	\begin{split}
	\min_{u \in [0,1]^N} & % \in [0,1]^N
		\hspace*{2mm} \rho \; {\lVert \Diff u\rVert}_{1,2} %\TV(u) 
		+  \Dist(Au , H u)
		\\
		& \hspace*{17mm} +  \Dist(B(\U-u) , H(\U-u) ) 
	% + \chi_{[0,1]^{N}}(u)
	\end{split}
	.
}

\noindent
Notice that the matrix $H \in\R^{M\cdot N}$ is sparse (with only $N$ non-zero values) and $A$ and $B$ are of rank $1$, so that storing or manipulating these matrices is not an issue.

In \cite{papa_aujol}, the distance function ${\Dist}$ was defined as the $\ell_1$ norm.
In the sections~\ref{sec:MK_distance} and \ref{sec:sinkorn_distance}, we investigate the use of similarity measures based on optimal transport cost, which is known to be more robust and relevant for histogram comparison~\cite{rabin_JMIV}. 
In the next paragraphs, we first investigate some extensions of the previous framework and then we describe the optimization method used to solve the proposed variational problems.

% multi-phase segmentation
\subsection{Convex multi-phase formulation}\label{sec:multiphase}

Let $a_1,\ldots a_K$ be $K\ge  2$ input histograms. The previous framework can be extended to estimate a partition of the domain $\Omega$ of an input image with respect to these histograms.
%
%\cmt{j'ai changé les notations pour être homogène}
%
%\cmt{ je propose de parler d'abord du modèle de
%\cite{Zach_multilabel_08}
%}

\paragraph{Multiple probability map}

A simple way to extend the binary model defined in Formula~\eqref{eq:segmentation_energy} %{eq:convex_segmentation_problem} 
is to describe the partition of the image into $K$ regions for each pixel $x \in \Omega$ by a binary variable $u(x) \in \{0,1\}^K$:
$$	u(x)=(u_1(x); \ldots u_K(x)) 
\in  \{0,1\}^{K}, \text{ s.t. } \dotp{u(x)}{\U_K} = 1
%\in [0,1]^{K},
$$ 
where $u_k(x)$ states whether the pixel at $x$ belongs to the region indexed by $k$ or not. %, $k=1\cdots K$. 

The extension of the convex optimization problem \eqref{eq:convex_segmentation_problem} is then obtained by the 
relaxation of $u$ into a probability vector map, as done for instance in \cite{Zach_multilabel_08,PYAC13}:
$u(x) \in \PS_{1,K}$ so that $u_k(x)$ defines the probability the pixel at $x$ belongs to the region indexed by $k$
\begin{equation}
	\label{eq:convex_multi-segmentation_problem_probability}
	\min_{\substack{u = (u_k \in \R^N)_{1 \le k \le K}\\\text{s.t. } u(x) \in \PS_{1,K}}} \; % \in [0,1]^N
	\sum_{k=1}^K  {\lVert \Diff u_k\rVert}_{1,2} %  \rho_k \TV(u) 
	+  \Dist(A_k u_k , H_k u_k) 	%+ \sum_k  \chi_{\PS_{1,K}}(u(x)),
	\;,
\end{equation}
where 
$H_k = H$ for all $k$ in the simplified setting, and where
$A_k$ indicates the linear operator that multiplies histogram $a_k$ by the total sum of the entries of $u_k$, as previously defined in Eq.~\eqref{eq:prior_histogram_operator}.

Notice that other convex relaxations for multi-phase segmentation with non-ordered labels and total variation regularization have been proposed in the literature~\cite{Lellmann2009,Pock09}. 
The one proposed in \cite{Lellmann2009} is nevertheless less tight than \cite{Zach_multilabel_08}.
On the other hand, the convexification of~\cite{Pock09} is even tighter but harder to optimize, while giving very close results, even on pathological cases after thresholding (see \cite{Pock09} for a detailed comparison).

\subsection{Other model variations}\label{sec:variant}

In addition to the multiple labelling extension, some other variations of the previous framework are discussed in this section.

\paragraph{Soft assignment histogram}
% the linear operator $H_{\X}$ is an \emph{hard assignment} operator 
%  hard assignment to design the operator $H_{\X}$ 
For simplicity, we have assumed previously that each operator $H_k$ is an \emph{hard assignment} operators (see definition \eqref{eq:histogram_matrix_def}).
%In definition \eqref{eq:histogram_matrix_def}, \emph{hard assignment} is used to design the linear operator $H_{\X}$ and may be replaced by soft assignment in order to reduce quantization noise.
In the proposed framework, these histogram operators could be instead defined from soft assignment, which might reduce quantization noise.
However, the property \eqref{eq:histogram_sum} would not hold any longer for non binary variables $u$, so that the definition of operators $A_k$ should also change accordingly: $A_k \,u_k = \dotp{H u_k}{\U}\, a$.
%Notice that special care must be taken for the optimization procedure with respect to the conditioning of the matrix $H_{\X}$, since some rows of $H_{\X}$ can then be arbitrarily close to zero.
Observe also that special care should also be taken regarding the conditioning of the matrix $H_{}$, as some rows of $H_{}$ could be arbitrarily close to zero.

\paragraph{Supervised soft labelling} % soft or hard labelling

In our framework, prior histograms $\{a_k\}_K$ may be given \emph{a priori} but can also be defined from
\emph{scribbles} drawn on the input image by the user. In the experiments, we will consider binary scribbles 
$s_k : \Om \mapsto \{0,1\}$ so that prior histograms are defined as (assuming that condition \eqref{eq:histogram_sum} is fullfilled)
\eq{
a_k = \frac{H_k\,  s_k}{\dotp{s_k}{\U}} \;.
}
This approach makes it possible for the model \eqref{eq:convex_multi-segmentation_problem_probability} to correct potentially user mislabelling, as the segmentation variables $u_k$ are not subject to verify the user labelling.
%solving  \eqref{eq:convex_multi-segmentation_problem_probability} may correct
%may contradict the user partial labelling, and 
Considering such hard labelling constraints would not increase the model complexity.

\paragraph{Multi-image segmentation}

The framework enables to segment multiple images with the same prior histograms that can be defined by scribbles from different images.
Without adding interaction terms to measure the similarity between the segmentation variables of each image, the corresponding optimization problems can be solved separately for each image.

%\cmt{A DIRE ICI ? As already mentioned in the introduction, various convex methods have been proposed to incorporate shape prior for image segmentation, such as \cite{cremers_convex_shape_prior_segmentation,Schmitzer_jmiv13}.
%However, to the best of our knowledge, no convex formulation has been proposed to compare shape for multiple joint segmentation.
%}

%\cmt{Idée pour la suite : comparer les cartes de segmentation, avec du transport non normalisé (ça reste convexe)
%$$
%TV(u_1) + \Dist(H u_1, A u_1) + \Dist(H u_2, A u_2) + S(u_1,u_2) + ...
%$$
%où $S$ mesure la similarité des formes des régions $u_1$ et $u_2$
%}

\subsection{Optimization}\label{sec:optim}

%\cmt{Une alternative serait de répertorier directement ici les différents cas de figure : (il manque par exemple le cas D lisse \& lipschitz : L2 ou chi2 par exemple)
%\newline
%Commencer par dire que si on choisi D=L22, on se retrouve avec le modèle TV-L2 très étudié (ROF) qui n'est pas très robuste pour la segmentation.
%\newline
%Dire que Yildizoglu \& Papadakis \& Aujol considere le cas "D non lisse" avec la norme l1 et Swoboda \& Schnoor exploite la version "D non lisse" avec le prox wasserstein, puis expliquer que l'on explore les autres voies dans ce travail.}

Every convex segmentation problems studied in this work are addressed using primal-dual forward-backward optimization schemes.
Depending on the properties of the convex function ${\Dist}$ chosen to measure similarity between histograms, several algorithms can be considered.

In particular, when ${\Dist}$ is a Lipschitz-differentiable function (using for instance quadratic $\ell_2$, Huber loss or $\chi^2$ distance), even simpler forward-backward algorithm can be used.
However, such a choice of function is known to be not very well suited for histogram comparison (see for instance~\cite{rabin_JMIV}) and more robust distances are therefore preferred, such as the $\ell_1$ norm in~\cite{Yildi12}.
%For instance, for $\Dist(x) = \norm{x}^2$, we 

As a consequence, and without loss of generality, we do not address this specific case in the following
% we assumme that the function ${\Dist}$ is not smooth, as done for instance in \cite{Yildi12} : not necessary (on peux facilement calculer le prox de L2^2 !)
%To that end, we resort to the Legendre-Fenchel transform ${\Dist}^*$ of ${\Dist}$ 
and consider the most general setting, without any assumptions on ${\Dist}$ (or ${\Dist}^*$, its Legendre-Fenchel transform)
{aside from being convex and lower semi-continuous}.

%We first consider the two phase segmentation problem.

\paragraph{Two-phase segmentation model} 

In order to reformulate~\eqref{eq:convex_segmentation_problem} as a primal-dual optimization problem, 
% we consider the following dualization of the problem
we resort to variable splitting, 
using the Legendre-Fenchel transforms of the discrete $\TV$ norm and the function ${\Dist}$ to obtain
\eql{\label{pb:dual_type}
%\small
\begin{split}
%\hspace*{-5mm}
	\min_{u}%\in \R^N}
	\max_{\substack{v\\p_A^{} ,\, p_B^{}\\q_B^{} ,\, q_A^{}}} 
	%\in \R^M\\p_B^{},q_B^{} \in \R^M
	%\\p^{}_C \in \R^{2N}}}
	\hspace{0.25cm}
	 & \dotp{ \Diff u}{v}+ \dotp{A^{}u}{q_A}+\dotp{B^{}(\U-u)}{q_B^{}}  
	 \\[-3mm]
	 &+ \dotp{Hu}{p_A} +\dotp{H(\U -u)}{p_B^{}}  
	 \\
 	& %\hspace{-5mm}  
 	- {\Dist}^*(p_A^{},q_A^{}) - {\Dist}^*(p_B^{},q_B^{})
	\\
	& + \chi_{[0,1]^{N}}(u) - \chi_{{\norm{.}}_{\infty, 2} \le \rho}(v)
 	 \\
\end{split}
}
where the primal variable is $u = {\left(u(x)\right)}_{x \in \Om}\in\R^N$ (corresponding to the segmentation map), and dual variables are $v = {\left(v(x)\right)}_{x \in \Om}\in\R^{2N}$ (related to the gradient field) and $p_A^{},p_B^{},q_A^{},q_B^{}\in\R^M$ (related to the histograms).  Notice that  ${\Dist}^*$ is the convex conjugate  of the function ${\Dist}$. 
In this new problem formulation, $\chi_{{\lVert . \rVert}_{\infty,2} \le \rho}$ is the characteristic function of the convex $\ell_{\infty,2}$ ball of radius $\rho$, as we have for the discrete isotropic $\TV$ norm
\eq{
\begin{split}
    \TV(u) 
    %= \sup_{\norm{v^k}\le 1} \dotp{v^k}{\Diff u^k}
    &= \sum_{x \in \Om} \sup_{\norm{v(x)}\le1} \dotp{v(x)}{\Diff u(x)} 
    \\
    &= \sup_{v}  \;  \dotp{v}{\Diff u} - \sum_{x \in \Om} \chi_{\norm{.}\le 1}(v(x))
    \\
    &= \sup_{v}  \;  \dotp{v}{\Diff u} - \chi_{{\norm{.}}_{\infty,2}\le 1}(v).
\end{split}
}
%, où $\ell_\infty$ est la norme duale de $\ell_1$.

In order to accommodate the different models studied in this paper, we assume here that ${\Dist}^*$ is a sum of two convex functions ${\Dist}^*={\Dist}_1^*+{\Dist}_2^*$, where ${\Dist}_1^*$ is non-smooth and ${\Dist}_2^*$ is differentiable with Lipschitz continuous gradient. 

We recover a general primal-dual problem of the form 
%\cmt{ERREUR 1 : I(x) est l'image à segmenter : remplacé par T (déjà utilisé pour la matrice Tau)}
%\cmt{ERREUR 2 : à la fin on a besoin du prox d'une fonction et non seulement du proj sur C, remplacé par R}
\begin{equation}\label{pb:primaldual}
\begin{split}
\min_{u} \; \max_{p} \; &
	\dotp{Ku}{p} + R(u) + T(u) % \chi_{\cal C}(u)  
	% \chi_{[0,1]^{N}}(u)
	\\
	& \quad - F^*(p)-G^*(p)
% sur une seule ligne :
%\hspace*{-5mm} \min_{u} \; \max_{p} \; & \dotp{Ku}{p} + R(u) + T(u) - F^*(p)-G^*(p)
\end{split}
\end{equation}%
with primal variable $u\in \R^N$ and dual variable % vector 
%$p=[p_A^\Trp ,q_A^\Trp ,p_B^\Trp ,q_B^\Trp ,v^\Trp ]^\Trp \in \R^{4M+2N}$%
$p=\left(p_A; q_A; p_B; q_B; v\right) \in \R^{4M+2N}$%
, where
\begin{itemize}\setlength\itemsep{0.3em}
\item[\textbullet] $K=[H^\Trp,A^\Trp,-H^\Trp,-B^\Trp,\Diff^\Trp]^\Trp \in \R^{(4M+2N)\times N}$ is a sparse linear matrix; % operator
%\\[1mm]

\item[\textbullet] {$T$} is convex and smooth, with Lipschitz continuous gradient $\nabla T$ with constant $L_T$.
For now, we have $T(u) = 0$ and $L_T=0$ in the setting of problem~\eqref{pb:dual_type}.
%but one could choose $T=S$ whenever $S$ satisfies the same properties.
%[j'ai remplacé le numéro, avant c'était \ref{pb:primaldual}]
%\cmt{  where function ${\Dist}$ is not smooth} 
%\cmt{(NP) pas compris ce  commentaire : (JR) OK c'était mal dit, c'était pour dire que si S est smooth on choisit I=S} 
%;

\item[\textbullet] %$\chi_{[0,1]^{N}}(u)$ is convex and non-smooth;
{$R$} is convex and non-smooth. 
In problem~\eqref{pb:dual_type}, we have $R = \chi_{\cal C}$ the indicator function of the convex domain ${\cal C} = [0,1]^{N}$ ;
%\\[1mm]

\item[\textbullet] $F^*(p)={\Dist}_1^*(p_A^{},q_A^{})+{\Dist}_1^*(p_B^{},q_B^{}) + \chi_{\lVert . \rVert_{\infty,2} \le \rho}(v)$ is convex and non-smooth;
%\\[1mm]

\item[\textbullet] 
$G^*(p)={\Dist}_2^*(p_A^{},q_A^{})+{\Dist}_2^*(p_B^{},q_B^{})- \dotp{H\U_{N}}{p_B^{}}- \dotp{B\U_{N}}{q_B^{}}$ is convex and differentiable, with Lipschitz continuous gradient with constant $L_{G^*}$. From definition of $H$ and $B$, one have
{$B\U_N = N b$ and $H \U_N = N h_I$ where $h_I$ is the normalized histogram of feature of the image $I$.}
\end{itemize}

To solve this problem, we consider the primal dual algorithm of~\cite{VuPrimalDual,CP14}
%
%\begin{equation}\label{algo:primaldual}
%\left\{
%\begin{array}{ll}
%u^{k+1}\hspace{-2pt}&\hspace{-1pt}=\text{Proj}_{[0,1]^{N}}\left(u^k\hspace{-1pt}-\hspace{-1pt}\tau (K^\Trp p^k + \nabla I(u^k)) \right)\\[3mm]
%p^{k+1}\hspace{-2pt}&\hspace{-1pt}=\text{Prox}_{\sig F^*}\hspace{-1pt}\left(p^k\hspace{-1pt}+\hspace{-1pt}\sig (K (2u^{k+1}\hspace{-1pt}-\hspace{-1pt}u^k)\hspace{-1pt}-\hspace{-1pt}\nabla G^*(p^k))\right)
%\end{array}\right.
%\end{equation}
%
\begin{equation}\label{algo:primaldual}
%\small
%\hspace*{-5mm} 
\! \left\{
\begin{array}{ll}
	u^{(t+1)} 
	& =  \Prox_{\, \tau R} \! % \Proj_{\, \cal C}
	\left(u^{(t)} - \tau (K^\Trp p^{(t)} + \nabla T(u^{(t)})) \right)
	\\[3mm] % [0,1]^{N}
	p^{(t+1)} &
	%= \text{Prox}_{\sig F^*} \left(p^{(t)} + \sig (K (2u^{(t+1)} - u^{(t)}) - \nabla G^*(p^{(t)}))\right)
	= \Prox_{\, \sig F^*}  \left(p^{(t)} + \sig K (2u^{(t+1)} - u^{(t)}) \right.
	\\[2mm]
	& \hspace*{25mm}   \left. - \si \nabla G^*(p^{(t)})\right)
\end{array}\right.
\end{equation}
where the notation $u^{(t)}$ indicates the variable at discrete time indexed by $t$.
For problem~\eqref{pb:dual_type}, one have $\Prox_{\, \tau R} =  \Proj_{\,[0,1]^N}$.
The application $\Prox_{\, \sig F^*}$ depends on the non-smooth part of similarity function $S$ and writes due to separability
\eq{\begin{split}
	& \Prox_{\, \sig F^*}(p) = \\
	& \; \left(
	\Prox_{\, \sig S_1^*}(p_A,q_A) ;\,
	\Prox_{\, \sig S_1^*}(p_B,q_B) ;\,
	\Proj_{{\norm{.}}_{\infty,2}\le \rho}(v)  \right),
\end{split}}
%
%
%\cmt{on devrait définir $\Proj_{{\norm{.}}_{\infty,2}\le \rho} $ ! } 
%AU GRETSI ON AVAIT : 
where 
%\eq{
%	\Proj_{[0;1]^N}(u) 
%	%= \left({\proj_{[0,1]}(u_i)}\right)_i 
%	= \min(\max(u,\O_N),\U_N).
%}
\eql{\label{eq:proj_Linf2}
	\Proj_{{\norm{.}}_{\infty,2} \le \rho}(v)(x) = \frac{v(x)}{\max\left\{{\norm{v(x)}}_{}/{\rho},1\right\}}
	.
}
%
%\left({ 
%\Proj_{{\norm{.}}_{\infty,2} \le \rho}(v^1),
%\Proj_{{\norm{.}}_{\infty,2} \le \rho}(v^2)
%}\right)
%$$
%
%	$$ \text{with }
%	\Proj_{{\norm{.}}_{\infty,2} \le \rho}(v^k) 
%	= \left( 
%	\frac{v^k(x)}{\max\left\{{\norm{v^k(x)}}/{\rho},1\right\}}
%	\right)_{x\in \Om_k}  .
%	$$

The algorithm~\eqref{algo:primaldual} is guaranteed to converge from any initialization of $(u^{(0)},p^{(0)})$ to a saddle point of \eqref{pb:primaldual} as soon as the step parameters $\sig$ and $\tau$ satisfy (see for instance~\cite{CP14}[Eq. 20])
\begin{equation}\label{time_step}
	\left(\tfrac1\tau  - {L_T}\right)\left(\tfrac1\sig - L_{G^*}\right)\ge  \lVert K \rVert^2.
\end{equation}
The worst case estimate for this norm is
\eq{
	\norm{K} = 4 \sqrt{N}+ \sqrt{8}.
}
\begin{proof}
See appendix \ref{proof:norm_K}.
\end{proof}

\paragraph{Preconditioning} 
%As one can see, 
As a consequence of the large value of $\norm{K}^2$ scaling with the primal variable dimension,
the gradient step parameters $(\tau,\sigma)$ may be very small to satisfy Eq.~\eqref{time_step}, 
which results in a slow convergence.

Fortunately, this algorithm can benefit from the recent framework proposed in \cite{Condat_icip14,Lorenz_Pock_inertial_FB_JMIV2015}, using preconditioning.
The idea is to change the metric by using --fixed or variable-- matrices $\mathbf T$ and $\mathbf \Sigma$ in lieu of scalar parameters $\tau$ and $\sig$
in \eqref{algo:primaldual}.
 
Following the guideline proposed in \cite{Lorenz_Pock_inertial_FB_JMIV2015} to design diagonal and constant conditioning matrices, we define
\eq{
\begin{split}
	\mathbf T:= &\diag \left( \boldsymbol \tau \right) % {\tau} \, \Id_N
	\quad \text{ and } \quad 
	\\
	\mathbf \Sigma :=  &\diag \left( \boldsymbol \sig \right)  
	=  \diag \left( \boldsymbol \sig_{H} ,  \boldsymbol \sig_{a} ,  \boldsymbol \sig_{H} ,  \boldsymbol \sig_{b} , \boldsymbol \sig_D \right)  %\in \R^{2N+M}
	\\
\end{split}
}
where
%(with the convention that the inverse is computed element-wise)
%$\frac{1}{\boldsymbol \sig} = \left(\frac{1}{\sig(\ell)}\right)_{1 \le \ell \le M}$ 
\begin{equation}\label{eq:diag_precond}
\begin{split}
	\frac{1}{\boldsymbol \tau(x)} &= \frac{L_T}{\ga} + r \sum_{i=1}^{4M+2N} \abs{K_{i,x}} \;,
	\\
	\frac{1}{\boldsymbol \sig(i)} &= \frac{L_{G^*}}{\delta} + \frac{1}{r} \sum_{x=1}^N \abs{K_{i,x}} \;. % x \in \Om
	\\
\end{split}
\end{equation}
For the setting of problem~\eqref{pb:dual_type}, considering an hard assignment matrix $H$ and writing the operator $D$ in matrix form, we have
\begin{equation*}
\begin{split}
	\frac{1}{\boldsymbol \tau(x)}  & = 4r + r \sum_{y=1}^{2N} \abs{D_{y,x}} \le 8 r \;\text{} 
	%\;\forall\; 1 \le j \le N 
	\\
	%\quad \text{} \forall\; x \in \Om \backslash \partial \Om
	%
	%\frac{1}{\boldsymbol \sig_{H}} &=  \frac{L_{G^*}}{\delta}\U_M  + \frac{1}{r} {H}{\U_{N}}
	\frac{1}{\boldsymbol \sig_{H}} &=  \frac{L_{G^*}}{\delta}\U_M  + \frac{N}{r} {h_{I}} \quad \text{with }  h_I = \tfrac1N H \U_N 
		\\
	\frac{1}{\boldsymbol \sig_{h}} &=  \frac{L_{G^*}}{\delta}\U_M  + \frac{N}{r} h \quad \text{for histogram } h=a \text{ and } b
	\\
	\frac{1}{\boldsymbol \sig_D(y)} &=  \frac{1}{r}  \sum_{x=1}^{N} \abs{D_{y,x}}  \le \frac{2}{r} .
	%\text{ if } y \text{ corresponds to a pixel inside the domain}, 0 \text{ otherwise}
\end{split}
\end{equation*}
The scaling parameters $r>0$ and $\delta \in (0,2)$ enable to balance the update between the primal and the dual variables.
We observed that the preconditioning allows for the use of very unbalanced histograms (that is far from being uniform) that otherwise could make the convergence arbitrarily slow.

\bigskip

Other acceleration methods, such as variable metric~\cite{Condat_icip14} and
inertial update~\cite{Lorenz_Pock_inertial_FB_JMIV2015}, may be considered.
% and have not been investigated.

\paragraph{Multiphase optimisation} 

The algorithm used to minimize problem \eqref{eq:convex_multi-segmentation_problem_probability} is the same as in \eqref{algo:primaldual}.
% The only difference is that each vector $u(x)$ is now projected onto the $K$ dimensional simplex $ {\cal C} = \PS_{1,K}$, whereas it was projected onto the interval $[0,1]$ for the binary segmentation problem.
The only two differences are the size of the variables and the convex constraint set ${\cal C}$. 
First, we consider now multi-dimensional primal and dual variables, 
\emph{i.e.} respectively $u : x \in \Om \mapsto (u_k(x))_{k=1}^K$ and 
%$\textstyle \left(p_k =[{p_A^k}{}^\Trp ,{q_A^k}{}^\Trp ,{p_B^k}{}^\Trp ,{q_B^k}{}^\Trp ,{v^k}{}^\Trp ]^\Trp \right)_{k=1}^K$.\cmt{un peu lourd comme notation...}
$%\eq{
	p = \left(p_k\right)_{k=1}^K
$
	%\quad \text{ with }
with $
	p_k = (p_A^k; q_A^k; p_B^k; q_B^k; v^k)
	.
$ %}
Furthermore, the constraint set ${\cal C}$ for the primal variable $u$ is defined for each pixel $u(x)$ as the simplex $\PS_{1,K}$ (defined in Eq.~\eqref{eq:simplex}), so that:
\eq{
	R(u) = 
	%\chi_{\cal C}(u) = 
	\sum_{x \in \Om} \chi_{\PS_{1,K}}(u(x)) \;.
}
%\cmt{C'est peut etre pas une bonne idée d'appeler ceci $\S_{1,K}$. La notation $\Delta$ sert partout ailleurs pour des couples d'histogrammes  (sauf pour le barycentre/ co-segementation à la toute fin).$->$ j'ai fait la modif}
In this setting, the definition of the diagonal preconditionners \emph{for each phase} $k$ is the same as in~\eqref{eq:diag_precond}.

\bigskip
Eventually, the primal variable $u^\star = u^{(\infty)}$ provided by the algorithm~\eqref{algo:primaldual} only solves the relaxed segmentation problem and has to be post-processed to obtain a partition of the image, as discussed in the next paragraph.

\subsection{Binarization of the relaxed solution}\label{sec:threshold}
%\paragraph{Thresholding}

% previous version
%The final binary segmentation can be obtained by thresholding the minimizer with respect to $0.5$. Notice that in general,  there is no correspondence between this thresholded solution and the global minimizer of the non relaxed problem  over binary variables $u\in\{0;1\}^N$.

%As previously discussed, 
The solution $u^\star$ of the relaxed segmentation problems studied before is a probability map, \emph{i.e.} $u^\star(x) \in [0,1]$.
Although in practice we have observed (see the experimental section~\ref{sec:seg_exp}), as already reported in \cite{Nikolova06,Zach_multilabel_08} for other models, that the solution is often close to be binary, \emph{i.e.} $u^\star(x) \approx 0$ or $1$, some thresholding is still required to obtain a proper labelling of the image.

Following for instance \cite{Zach_multilabel_08}, 
we simply select for every pixel $x$ the most likely label based on probability maps solutions $\{u_k^\star\}_{1 \le k \le K}$, %probabilities $\{u_k(x)\}_{1 \le k \le K}$, 
that is
\eql{\label{eq:multilabel_threshold}
	%\L(x) = 
	x \mapsto \argmin_{k} \; \{u^\star_k(x) \}_{1 \le k \le K} \;. % \{u_1(x), \ldots u_K(x) \}
}
Recall that in general, there is no correspondence between this thresholded solution and the global minimizer of the non-relaxed problem over binary variables. %  $u\in\{0,1\}^N$.

In the specific case of the $K=2$ phase segmentation problem, the previous processing boils down to using a threshold $t=\tfrac{1}{2}$ to define $u_t(x) = \idc{u^\star(x)>t}$. 
A better strategy would be to optimize the global threshold $t$ such that the objective functional $J(u_t)$ is minimized.
However, due to the complexity of the measures ${\Dist}$ considered in this work, this method is not considered here.

%\cmt{garder la suite ? les variables $\theta$ ne sont même plus définies}
%Finally, notice that following lifting methods for multi-label problems \cite{Pock09,PYAC13}, another option is to use the change of variable
%$u_k(x)=\sum_{j\ge  k} \theta_j(x)$ where $\theta_k > \theta_{k+1} \; \forall\, k$ are ordered variables,  
%%threshold $u_i(x)$ with respect to $0.5$ as $\tilde u_i(x)=1$ if $u_i(x)\ge  0.5$ and $0$ otherwise, and finally take $\mathcal{L}(x)=\sum_i \tilde u_i(x)$. 
%and to define the labelling again with thresholding $\mathcal{L}(x)=\sum_{k} \idc{u^\star_k(x)>t}$.
%We nevertheless observed almost no difference  between this strategy and the one based on the maximal probability on our segmentation experiments.

%%%%%%%%%%%%%%%%%%%%%%%%%%%%%%%%%%%%%%%%%%%%%%%%%%%
\section{Monge-Kantorovitch distance for image segmentation}\label{sec:MK_distance}

We investigate in this section the use of optimal transport costs as a distance function $\Dist$ in the previous framework.

\subsection{Optimal Mass Transportation problem and the Wasserstein Distance}

%\subsubsection
\paragraph{Optimal Transport problem}

Following \cite{Rabin_ssvm15}, we consider in this work the discrete formulation of the Monge-Kantorovitch optimal mass transportation problem (see \emph{e.g.}~\cite{Villani03}) between a pair of normalized histograms $a$ and $b$. % \cmt{j'ai enlevé le simplexe}  \in \PS_{m, M}
Given a fixed assignment cost matrix $C_{\A,\B} \in \R^{M \times M}$ between the corresponding histogram centroids $\A=\{A_i\}_{1\le i \le M}$ and $\B=\{B_j\}_{1\le j \le M}$, an optimal transport plan %$P^*\in {\Pp}(a,b)$ 
minimizes the global transport cost, defined as a weighted sum of assignments $\forall \, (a,b) \in \S$:
\eql{\label{eq:OT}
	\MK(a,b) 
		%= \dotp{P^{\star}}{C} 
		:= \hspace*{-1mm}\min_{P \in {\Pp}(a,b)} \left\{\dotp{P}{C}  =  \textstyle \sum_{i,j=1}^{M} P_{i,j}C_{i,j}\right\}. % \sum_{j=1}^{M} 
}
The set of admissible histograms is
\eql{\label{eq:admissible_histograms}
	\S: = \{a,b \in \R^{M}_+ \,,\;\dotp{a}{\U_{M}} = \dotp{b}{\U_{M}}\},
}
and the polytope of admissible transport matrices reads
\eql{\label{eq:admissible_matrices}
	\hspace*{-1mm} {\Pp}(a,b) := \{P\in \RR_+^{M \times M}\!,  P\U_{M} \!= \!a \text{ and } P^\Trp \U_{M} \! = \!b\}.
}
Observe that the norm of histograms is not prescribed in $\S$, and that we only consider histograms with positive entries since null entries do not play any role.

\paragraph{Wasserstein distance}%\label{sec:wasserstein}
When using $C_{i,j} = {\lVert A_i - B_j \rVert}^{p}$, then we recover the Wasserstein distance 
\eql{\label{eq:wasserstein}\textbf{W}_p(a,b) = \MK(a,b)^{1/p},}
which  is a metric between normalized histograms.
In the general case where $C$ does not verify such a condition, by a slight abuse of terminology we refer to the $\MK$ transport cost function as the Monge-Kantorovich \emph{distance}.

\paragraph{$\ell_1$ distance}\label{sec:wasserstein_vs_l1}

As previously mentioned, the $\ell_1$ norm is a popular metric in statistics and signal processing, in particular for image segmentation. When penalized by a factor $\tfrac12$, it is also known as the \emph{total variational distance} or the \emph{statistical distance} between discrete probability distributions.
As a matter of fact, such a distance can also be seen as a special instance of optimal transport when considering the cost function $C_{i,j} = 2 (1-\delta_{ij})$ and the same set of features $\A=\B$. See Appendix~\ref{sec:proof_l1} for more details.

This relation illustrates the versatility and the advantages of optimal transport for histogram comparison as it allows to adapt the distance between histogram features and to use different features for each histogram, contrarily to usual metric.

\paragraph{Monge-Kantorovich distance}
In the following, due to the use of duality, it is more convenient to introduce the following reformulation for general cost matrix $C$
and $\forall\, a,b \in \R^M$
\eql{\label{eq:OT_indicators}
	%\forall\, a,b \qquad 
	\MK(a,b) = \min_{P \in {\Pp}(a,b)} \dotp{P}{C}  
	+\chi_{\S}(a,b).
}
Notice that the optimal transport matrix $P$ is not necessarily unique.

\paragraph{Linear Programming formulation}
We can rewrite the optimal transport problem as a linear program with vector variables.
The associated primal and dual problems write respectively
\eql{\label{eq:primal_dual_MK}
\begin{split}
	\MK(\alpha) %\hspace{-0.25cm}
	&= \min_{\substack{p\in \RR^{M^2} \text{ s.t. } p \ge \O, \; L^\Trp p = \alpha}} %\hspace{-0.15cm}
	\dotp{c}{p} + \chi_{\S}(\alpha)
	\;
	\\
	&= %  \hspace{-0.15cm}
	\max_{\substack{\beta \in \R^{2M} \text{ s.t. } L\beta \le c}}  \;
	\dotp{\alpha}{\beta},
\end{split}
}
where 
$\alpha=(a;b) \in \R^{2M}$
%$\alpha^\Trp=[a^\Trp,b^\Trp]$
is the concatenation of the two histograms and
the unknown vector $p \in \RR^{M^2}$ corres\-ponds to the bi-stochastic matrix $P$ being read column-wise (\emph{i.e.} $P_{i,j} = p_{k}$ with 1D index $k(i,j)=i+ M(j-1)$).
The  $2M$ linear marginal constraints on $p$ are defined by the matrix
$L^\Trp \in \RR^{2M \times M^2}$ through equation $L^\Trp p = \alpha$, where
\eq{\footnotesize % \small
L^\Trp= 
\begin{bmatrix}
 e_1\U_{M}^\Trp & e_2 \U_{M}^\Trp & \cdots & e_{M}\U_{M}^\Trp \\
\text{Id}_{M}  & \text{Id}_{M}   & \cdots & \text{Id}_{M}  \\
\end{bmatrix}
=
\begin{bmatrix}
\U_M^\Trp &\otimes& \Id_M
\\
\Id_M &\otimes& \U_M^\Trp  
\end{bmatrix}
}
%\cmt{a verifier : erreur de typo à nouveau : c'est $ e_1 \U_{M}^\Trp$ au lieu de $\U_{M} e_1^\Trp$ et à écrire avec des produits tensoriels}
with $e_i(j) = \delta_{ij} \;\; \forall\; j\le M$.
%Note that we have the following property:
As a consequence,  
\eq{
%(L\alpha)_{i+M(j-1)} = 
(L\alpha)_{k(i,j)} = 
\left(L
\small
{ 
	\begin{bmatrix}
		a\\b
	\end{bmatrix}
}
\right)_{k(i,j)}
={ \left(
	a \U^\Trp 
	+ \U b^\Trp
\right)}_{i,j}
=a_i + b_j.
}

From the dual formulation \eqref{eq:primal_dual_MK} that contains a linear objective with inequality constraints, 
one can observe that the function $\MK(\alpha)$ is not strictly convex in $\alpha$
{and not differentiable everywhere}.
We also draw the reader's attention to the fact that the indicator of set $\S$ is not required anymore with the dual formulation, which will later come in handy.

\paragraph{Conjugate Monge-Kantorovich distance} 

% \cmt{j'ai changé 'dual' par 'conjugate'}

From Eq.~\eqref{eq:primal_dual_MK}, we have that the Legendre–Fenchel conjugate of $\MK$ writes simply as the characteristic function of the set $\{\beta \in \R^{2M} \,,\, L\beta - c \le \O\}$
\eql{\label{eq:MK_dual}
	\MK^*{\left( \beta \right)} 
	%{\color{red}= \max \left\{L\beta-c\right\} }  pour le simplexe de proba
	= \chi_{L\beta\le  c}(\beta)
	\qquad
	\forall\, \beta \in \R^{2M},
}
where $c$ denotes the vector representation of the cost matrix $C$ (\emph{i.e.} $C_{i,j} = c_{i+ M(j-1)}$).

% \cmt{J'ai enlevé la preuve, car c'est direct depuis (18)} :
%\begin{remark}
%Observe that the additional simplex constraint is completely absorbed
%\cmt{Reecrire, decouper les calculs?}
%\eq{
%	\begin{split}
%	\MK^*{\left( \beta \right)} 
%		& = \max_{h} \dotp{h}{\beta} - \MK{\left( h \right)} 
%	\\
%		& = \max_{h} \dotp{h}{\beta} - \min_{p \ge \O} \left(\dotp{p}{c} + \chi_{L^\Trp p = h} + \chi_{h\in \S}\right)
%	\\
%		& = \max_{p \ge \O} \max_{h} \dotp{h}{\beta} - \dotp{p}{c} - \chi_{L^\Trp p = h} - \chi_{h\in \S}
%	\\
%		& = \max_{p \ge \O} \dotp{L^\Trp p }{\beta} - \dotp{p}{c} - \chi_{L^\Trp p \in \S}
%	\\
%		& = \max_{p \ge \O} \dotp{p}{L\beta-c} - \chi_{L^\Trp p \in \S}
%	\\
%		& = \max_N \max_{q \ge \O, \dotp{\U}{q}=1}  \mathbb{E}_q(L\beta-c)
%	\\
%	%	& = \max \{L\beta-c\}  = \max_{i,j} \max \{ \beta_{1,i} + \beta_{2,j} - C_{i,j}\}
%	%\\
%		& = \chi_{L\beta\le  c}(\beta) = \max_{p \ge \O} \dotp{p}{L\beta-c}
%	\\
%	\end{split}
%}
%using the fact that $P=\dotp{P}{\U} Q$ where $Q$ is any bi-stochastic matrix, then writing $x = L\beta-c = vec(X)$, we have $\dotp{p}{L\beta-c} =  N \dotp{q}{L\beta-c} = \mathbb{E}_q(X)$ which can't be arbitrarily large if $X\ge \O$ and $N\mapsto \infty$.
%
%\end{remark} 

\subsection{Integration in the segmentation framework}
\label{sec:biconjugaison}

We propose to substitute in problem \eqref{eq:convex_segmentation_problem}
the similarity function $\Dist$ by the convex Monge-Kantorovich optimal transport cost \eqref{eq:OT_indicators}.

\subsubsection{Proximity operator}\label{sec:prox_wasserstein}

In order to apply the minimization scheme described in \eqref{algo:primaldual}, 
{as $\MK^*$ is not differentiable}, 
we should be able to compute the proximity operator of $\MK^*$.
Following~\eqref{eq:MK_dual} it boils down to the projection onto the convex set $\{\beta \,,\, L \beta  \le c\}$.
However, because the linear operator $L$ is not invertible, this projector cannot be computed in a closed form and the corresponding optimization problem should be solved at each iteration of the process~\eqref{algo:primaldual}.

A similar strategy is employed in \cite{Swoboda13} with the qua\-dratic Wasserstein distance (defined in~\eqref{eq:wasserstein}, using $p=2$), where the proximity operator of $\Prox_{{\mathbf{W}_2^2}(.,a)}(H u)$ with respect to the primal variable $u$ is computed using qua\-dratic programming algorithm.
To reduce the resulting time complexity, a reformulation is proposed which does not depends on the size $N$ of the pixel grid, but rather on the number of bins $M$, as in our framework with the computation of $\Prox_{\MK^*}$.

\subsubsection{Biconjugaison}\label{sec:opt:biconj}
To circumvent this problem, we resort to biconjugaison 
to rewrite the $\MK$ transport cost as a primal-dual problem itself. 
First, we can write
$\MK^*( \beta ) = f^*(L\beta)$ with 
%$f^*(r) = \chi_{r\le c}(r)$, 
$f^* = \chi_{. \le c}$, 
so that $f(r) = \dotp{r}{c} + \chi_{.\ge \O}(r)$.
Then, using variable splitting
\eql{\label{eq:primal_dual_MK_splitting}
\begin{split}
%\hspace*{-8mm}
	\MK^*(\beta) &
	= f^*( L\beta ) = \max_r\; \dotp{r}{L\beta} - f(r)\\& = \max_r\; \dotp{r}{L\beta-c} - \chi_{\cdot\ge \O}(r)
	\\
\end{split}
}
\text{and}
\eq{%\label{eq:dual_MK_splitting}
\begin{split}
	\MK(\alpha) &
	= \max_{\beta}\; \dotp{\alpha}{\beta}  - f^*( L\beta )\\
	%&= \max_{\beta}\; \dotp{\alpha}{\beta} + \min_r \dotp{r}{c-L\beta} + \chi_{\cdot \ge \O}(r)
	%\\
	%& 
	%= \max_{\beta} \; \min_{p \in \R^{M_a \cdot M_b}} \; \dotp{\alpha}{\beta} - \dotp{p}{L\beta} + f(p)
	%\\
	& = \min_{r} \; \max_{\beta} \; 
	\dotp{r}{c} + \chi_{\cdot \ge \O}(r) + \dotp{\alpha - L^\Trp  r}{\beta}
	\\
%	& =
%	\\
\end{split}
}
where $\min$ and $\max$ are swapped in virtue of the minimax theorem (the characteristic function being lower semi-continuous for variable $r$).
%
%With this formulation,
With this {\emph{augmented} representation of the transportation problem, 
it is no longer necessary to compute the proximity operator of $\MK^*$.

\subsubsection{Segmentation problem} 
Plugging the previous expression into Eq.~\eqref{pb:dual_type} enables us to solve it using algorithm~\eqref{algo:primaldual}.
Indeed, introducing new primal variables $r_A,r_B \in \R^{M^2}$ related to transport mappings for the binary segmentation problem, we recover the following primal dual formulation (extension for multi-phase segmentation is straightforward using Section~\ref{sec:optim})
%\eql{\label{pb:dual_type_MK}
%\small
%\begin{split}
%	\hspace{-4mm}
%	\min_{\substack{u \in \R^N\\r_A,r_B \in \R^{M^2}}}
%	\max_{\substack{p_A^{},q_A^{},p_B^{},q_B^{} \in \R^M\\v \in \R^{2N}}}
%	 & \langle Hu,p_A^{}\rangle + \langle A^{}u,q_A^{}\rangle +\langle H(\U -u),p_B^{}\rangle +\langle B^{}(\U-u),q_B^{}\rangle
%	 \\[-4mm]
% 	& %\hspace{-5mm}  
% 	\dotp{r_A}{c-L{\begin{bmatrix}p_A\\q_A\end{bmatrix}}} + % \tiny
% 	\dotp{r_B}{c-L{\begin{bmatrix}p_B\\q_B\end{bmatrix}}}   + \langle \Diff u,v\rangle
% 	%- \dotp{L^\Trp  r_B}{[p_B^\Trp ,  q_B^\Trp ]^\Trp }
% 	 \\%[-3mm]
% 	 & + \chi_{[0,1]^{N}}(u) + \chi_{\cdot \ge \O}(r_A) + \chi_{\cdot \ge \O}(r_B) - \chi_{\lVert.\rVert\le \rho}(v).
%\end{split}
%}
\eql{\label{pb:dual_type_MK}
%\small
\begin{split}
\hspace*{-1mm}
	\min_{\substack{u\\r_A, r_B}}%\substack{u \in \R^N\\r_A,r_B \in \R^{M^2}}}
	\; \max_{\substack{v\\p_A^{}, q_A^{}\\p_B^{}, q_B^{}}}
	\;
	 %\quad 
	 & \quad \dotp{ \Diff u}{v} - \chi_{\lVert.\rVert_{\infty,2} \le \rho}(v) 
	 \\[-5mm]
	 & + \dotp{ A^{}u}{q_A^{}} + \dotp{ B^{}(\U-u)}{q_B^{}}
	 \\
	 & + \dotp{ Hu}{p_A^{}} + \dotp{ H(\U -u)}{p_B^{}} 
	 \\
 	 & + \dotp{r_A}{c-L{\small \begin{bmatrix}p_A\\q_A\end{bmatrix}}} 
 	    + \dotp{r_B}{c-L{\small \begin{bmatrix}p_B\\q_B\end{bmatrix}}}   
 	%- \dotp{L^\Trp  r_B}{[p_B^\Trp ,  q_B^\Trp ]^\Trp }
 	 \\%[-3mm]
 	 & + \chi_{[0,1]^{N}}(u) +  \chi_{\cdot \ge \O}(r_A) + \chi_{\cdot \ge \O}(r_B).
\end{split}
}
Using the canonic formulation~\eqref{pb:primaldual}, we consider now%
\eql{\label{eq:biconjugate_K}\small
	K = 
%	\begin{bmatrix}
%	H & -L^\Trp & \O
%	\\
%	A & - & \O
%	\\
%	-H & \O & -L^\Trp
%	\\
%	-B & \O & -
%	\\
%	D & \O & \O
%	\\
%	\end{bmatrix}
	%
	\left[\begin{array}{r ccc cc}
		 H && \multicolumn{2}{c}{\multirow{2}{*}{$-L^\Trp$}} & \multicolumn{2}{c}{\multirow{2}{*}{$\O$}}\\
		 A &&  & \\
		-H &&  \multicolumn{2}{c}{\multirow{2}{*}{$\O$}} & \multicolumn{2}{c}{\multirow{2}{*}{$-L^\Trp$}} \\
		-B &&  \\
		 %D & & \O & & \O \\
		 D && \multicolumn{2}{c}{\O} & \multicolumn{2}{c}{\O} \\
	\end{array}\right]
	.
}
In addition, observe that there is now an additional linear term $T(u,r_A,r_B) = \dotp{r_A+r_B}{c}$ whose gradient 
%\cmt{$(\nabla T)^\Trp = [\O_N^\Trp,c^\Trp,c^\Trp]$} 
{$\nabla T = \left(\O_N;\, c;\, c\right)$} 
has a Lipschitz constant $L_T = 0$. 
As in problem~\eqref{pb:dual_type}, we still have $R = \chi_{\cal C}$
which writes here %has also changed
$$
\chi_{\cal C}(u,r_A,r_B) =  \chi_{[0,1]^{N}}(u) +  \chi_{\cdot \ge \O}(r_A) + \chi_{\cdot \ge \O}(r_B)  \;.
$$
The proximity operator of the characteristic function $\chi_{\cdot \ge \O}$ boils down to 
the projection onto the nonnegative orthant $\R^{M^2}_+$: 
\eql{\label{eq:proj_positive_orthant}
	\Proj_{{\cdot \ge \O}}(r) = \max\{\O,r\}.
}

%We have also gained extra non smooth characteristic functions $\chi_{\cdot \ge \O}$, whose proximity operators are simply given by the projection onto the positive quadrant $\R^{M^2}_+$: 
%$\Prox_{\chi_{. \ge \O}}(x) = \max\{\O,x\}$.
%$\Proj_{{\cdot \ge \O}}(x) = \max\{\O,x\}$.

The preconditioners for the problem~\eqref{pb:dual_type_MK} are computed using the definition~\eqref{eq:diag_precond} for the operator $K$ defined in Formula~\eqref{eq:biconjugate_K}.

% \cmt{Détailler les preconditionneurs qui changent ?}

\subsubsection{Advantages and drawback} The main advantage of this segmentation framework is that it makes use of optimal transport to compare histograms of features, 
without sub-iterative routines such as solving optimal transport problems to compute sub-gradients or proximity operators (see for instance~\cite{Cuturi13,Swoboda13}), %  [Jalal]
or without making use of approximation (such as the Sliced-Wasserstein distance~\cite{PeyreWassersteinSeg12}, generalized cumulative histograms~\cite{Papadakis_ip11} or entropy-based regularization~\cite{CuturiDoucet14}).
Last, the proposed framework is not restricted to Wasserstein distances, since it enables the use of any cost matrix, and does not depend on features dimensionality.

However, a major drawback of this method is that it requires two additional primal variables $r_A$ and $r_B$ whose dimension is $M^2$ in our simplified setting, $M$ being the dimension of histograms involved in the model. As soon as $M^2 \gg N$, the number of pixels, the proposed method could be significantly slower than when using $\ell_1$ as in \cite{papa_aujol} due to time complexity and memory limitation. This is more likely to happen when considering high dimensional features, such as patches or computer vision descriptors, as $M$ increases with feature dimension~$n$.

%%%%%%%%%%%%%%%%%%%%%%%%%%%%%%%%%%%%%%%%%%%%%%%%%%%
\section{Regularized $\MK$ distance for image segmentation}\label{sec:sinkorn_distance}

As mentioned in the last section, the previous approach based on optimal transport may be very slow for large histograms.
In such a case, we propose to use instead the entropy smoothing of optimal transport recently proposed and investigated in~\cite{Cuturi13,CuturiDoucet14,CuturiPeyre_smoothdualwasserstein}.
This strategy is also used by the \emph{Soft-Assign} Algorithm~\cite{SoftAssign} to solve linear and quadratic assignment problems.

% that may offer increased robustness to outliers
While offering good properties for optimization, it is also reported~\cite{Cuturi13} to give a good approximation of optimal transportation and increased robustness to outliers.
While it has been initially studied for a pair of vectors on the probability simplex $\PS_{1,M}$, we follow our preliminary work \cite{Rabin_ssvm15} and investigate in details its use for our framework with unnormalized histograms on $\S$.

\subsection{Sinkhorn distances $\MK_{\lbd}$}

The entropy-regularized optimal transport problem~\eqref{eq:OT_indicators} on the set $\S$ (see Eq.~\eqref{eq:admissible_histograms})~is
\eql{\label{eq:sinkhorn_distance_nonnormalized}
\hspace*{-1mm}
\begin{split}
	\MK_\lbd (a,b) 
	& %:=  \dotp{P_\lbd^{\star}}{C} 
	= \hspace*{-2mm} \min_{P \in {\Pp}(a,b)} \dotp{P}{C}  - \tfrac{1}{\lbd} h(P) % \left\{\right\} 
	+ {\chi_{\S}(a,b)}
%	\\
%	\MK_\lbd (\alpha) 
%	& :=  \min_{\substack{p\in \RR^{M_a\cdot M_b}\\\text{s.t. } p \ge \O, \; L^\Trp  p = \alpha}}  \dotp{p}{c + \tfrac{1}{\lbd} \log p} + {\chi_{\S}(\alpha)}.
\end{split}
} where the entropy of the matrix $P$ is defined as $h(P) := -\dotp{P}{\log P}$ 
({with the convention that $h(\O) = 0$}).
% = - \sum_{i,j} P_{i,j} \log P_{i,j}.$
%
Thanks to the  strictly convex negative entropy term, the regularized optimal transport problem has a unique minimizer, denoted $P_\lbd^{\star}$. It can be recovered using a fixed point algorithm as demonstrated by Sinkhorn (see \emph{e.g.} \cite{Sinkhorn67,SoftAssign}). %\cite{sinkhorn}.
The regularized transport cost $\MK_\lbd (a,b)$ is thus referred to as the \emph{Sinkhorn distance}.

\subsubsection{Interpretation}
Another way to express the negative entropic term is:
\eq{
 - h(p) : p \in \R^{M^2}_+ \mapsto \textbf{KL}(p \Vert \U_{M^2}) \in \R, 
}
that is the Kullback-Leibler divergence between transport map $p$ and the uniform mapping. 
This shows that, as $\lbd$ decreases, the model encourages smooth, uniform transport so that the mass is spread everywhere.
This also explains why this distance shows better robustness to outliers, as reported in \cite{Cuturi13}.

Hence, one  would like to consider large values of $\lbd$ to be close to the original Monge-Kantorovich distance, but low enough to deal with feature {intrinsic variability and noise}. % perturbation.
As detailed after, the estimation of this regularized distance involves terms of the form $exp(-\lbd C)$. For numerical reasons, the process is limited to low values of $\lbd$ in practice, so that the Sinkhorn distances are rough approximations of the Monge-Kantorovich distances.

% \cmt{Ajouter reference à Bernard schmitzer et son implémentation log ?} je n'en n'ai pas trouvé

\subsubsection{Structure of the solution} 
First, {using the same vectorial notation as in Eq.~\eqref{eq:primal_dual_MK}}, 
the Sinkhorn distance \eqref{eq:sinkhorn_distance_nonnormalized} reads as
\eql{\label{eq:sinkhorn_distance_nonnormalized_bis}
\begin{split}
	\MK_\lbd (\alpha) 
	& %=\dotp{p^\star_\lbd}{c}
	:= \hspace*{-3mm} \min_{\substack{p\in \RR^{M^2}\\\text{s.t. } p \ge \O, \; L^\Trp  p = \alpha}} 
	\hspace*{-3mm}  \dotp{p}{c + \tfrac{1}{\lbd} \log p} + {\chi_{\S}(\alpha)}.
\end{split}
}
As demonstrated in~\cite{Cuturi13}, when writing the Lagrangian of this problem %\eqref{eq:sinkhorn_distance_nonnormalized}
with a multiplier $\beta$ to take into account the constraint $L^\Trp p=\alpha$, we can show that the respective solutions $P_\lbd^{\star}$ and $p_\lbd^{\star}$ of problem \eqref{eq:sinkhorn_distance_nonnormalized} and \eqref{eq:sinkhorn_distance_nonnormalized_bis} write
\eq{
\begin{split}
&\log p_\lbd^{\star} = \lbd (L\beta - c) - \U \text{ with }
\beta = 
\begin{bmatrix}
x\\y
\end{bmatrix}\\
\Leftrightarrow&
(\log P_\lbd^{\star})_{i,j} = \lbd (x_i + y_j - C_{i,j}) - 1.
%\quad \text{that is } P^\lbd = diag(u) e^{-\lbd C} diag(v)
\end{split}
}

%\begin{remark}
The constant $-1$ is due to the fact that we use the unnormalized KL divergence $\textbf{KL}(p \Vert \U_k)$, instead of $\textbf{KL}(p \Vert \frac{1}{k}\U_k)$ for instance.
%\end{remark}

\subsubsection{Sinkhorn algorithm} 
Sinkhorn showed~\cite{Sinkhorn67} that the alternate normalization of rows and columns of any positive matrix $M$ converges to a unique bistochastic matrix \eq{P= \diag(x) M \diag(y)} with the desired marginals.  
The corresponding fixed-point iteration algorithm can be used to find the solution $P_\lbd^{\star}$: setting $M_\lbd = e^{-\lbd C}$, one has
\eq{%\hspace*{-5mm}
P_\lbd^{\star} = \diag(x^{(\infty)}) M_\lbd \diag(y^{(\infty)})
}
\eq{\text{with }
x^{(t+1)} = \frac {a} {M_\lbd\, y^{(t)}} 
\; \text{ and } \;
y^{(t+1)} = \frac {b} {M_\lbd^\Trp \, x^{(t)}} 
,
}
where $a$ and $b$ are the desired marginals of the matrix.
With this result, one can   design a fast algorithm to compute  the regularized optimal transport plan,  the Sinkhorn distance or its derivative, as shown in~\cite{Cuturi13,CuturiDoucet14}.

%\subsection{Legendre–Fenchel transformation of Sinkhorn distance $\MK_\lbd$}
\subsection{Conjugate Sinkhorn distance $\MK_\lbd^*$}

%\cmt{j'ai changé 'Legendre–Fenchel transformation' par 'Conjugate'}

Now, in order to use the Sinkhorn distance in algorithm \eqref{algo:primaldual}, we need to compute its Legendre-Fenchel transform, which expression has been studied in~\cite{CuturiDoucet14}.
\begin{proposition}[Cuturi-Doucet]
The convex conjugate of $\MK_\lbd (\alpha)$ defined in \eqref{eq:sinkhorn_distance_nonnormalized_bis} reads
\eql{\label{eq:dual_sinkhorn_nonnormalized}
\MK^*_\lbd (\beta) 
=\tfrac{1}\lbd \left \langle  q_\lbd(\beta) ,  \U  \right\rangle
}
\eq{\text{with } \; 
q_\lbd(\beta) := e^{\lbd (L\beta - c)-\U}.
}
\end{proposition}
With matrix notations, writing 
%$\beta^\top = [\beta_1^\top,\beta_2^\top]$
$\beta = (\beta_1;\beta_2)$, we have equivalently
$
	\MK^*_\lbd (\beta) 
	=\tfrac{1}\lbd \dotp{Q_\lbd(\beta)}{\U}
$
\eql{\label{eq:dual_sinkhorn_nonnormalized_matrix}
	\text{with } \;  
	Q_\lbd(\beta_1,\beta_2) := e^{\lbd (\beta_1 \U^\top +  \U \beta_2^\top - C)-\U}.
}

This simple expression of the Legendre-Fenchel transform is $C^\infty$, but unfortunately, its gradient is not Lipschitz continuous. 
We propose two solutions in order to recover the setting of the general primal dual problem~\eqref{pb:primaldual} and be able to minimize the segmentation energy involving Sinkhorn distances.
We either define a new normalized Sinkhorn distance $\MK_{\lbd,\le N}$ (§~\ref{sec:normalized_sinkhorn_distance}), whose gradient is Lipschitz continuous (§~\ref{sec:normalized_sinkhorn_gradient}), or we rely on the use of proximity operator of $\MK_\lbd$ (§~\ref{sec:prox_sinkhorn_distance}). 
A discussion follows to compare the two approaches.
\subsection{Normalized Sinkhorn distance $\MK_{\lbd,\le N}$} %  on $\S_{\le N}$
\label{sec:normalized_sinkhorn_distance}

As the set $\S$ of admissible histograms does not prescribe the sum of histograms, we
consider here a different setting in which the histograms' total mass are bounded above by $N$, the number of pixels of the image domain $\Om$
\eql{\label{eq:set_normalized_histogram}
	\S_{\le N}: = \left\{a, b \in \R^{M}_+ \,\Big\vert\,  \dotp{a}{\U_{M}} = \dotp{b}{\U_{M}} \le N \right\}. %  \in \R^{M}_+
}
As an admissible transport matrices $P_\lbd^\star$ from $a$ to $b$ is not normalized anymore (\emph{i.e.} $\dotp{P_\lbd^\star}{\U}\le N$),  we use a slight variant of the entropic regularization: % normalization of the entropy term.
\eql{\label{eq:normalized_entropy}
\begin{split}
 \tilde h(p) &:= N h\left(\tfrac{p}{N}\right) = - N \,\textbf{KL}\left(\tfrac{p}{N} \Vert \U\right)\\
& = - \dotp{p}{\log p} + \dotp{p}{\U} \log N \;{\color{black}\ge 0}.
\end{split}
}
%This enables to have the following
\begin{corollary}
\label{corollary:dual_simplexNmax}
The convex conjugate of the normalized Sinkhorn distance
\eql{\label{eq:sinkhorn_distanceN}
\begin{split}
	& \hspace{-5mm} \MK_{\lbd,\le N} (\alpha)  :=
	\\
	& \hspace*{-3mm}
	\min_{\substack{p\in \RR^{M^2}\\\text{s.t. } p \ge \O, \; L^\Trp  p = \alpha}}  \hspace*{-2mm}
	 \dotp{p}{c + \tfrac{1}{\lbd}\log p - \tfrac{1}{\lbd}\log N \;\U} % \left\{ \right\} 
	 +{\chi_{\S_{\le N}}(\alpha)}
\end{split}
} 
reads
\eql{\label{eq:dual_sinkhorn_inegality}\small
\MK^*_{\lbd,\le N} (\beta)= 
\left\{\begin{array}{ll}
\frac{N}{\lbd} \langle q_\lbd(\beta), \U\rangle 
&\, \text{ if } \langle q_\lbd(\beta), \U\rangle \le 1,  % \text{otherwise.}
\\[3mm]
\frac{ N}{\lbd} \log \langle q_\lbd(\beta), \U\rangle  + \frac{ N}{\lbd}
& \, \text{ if }\langle q_\lbd(\beta), \U\rangle \ge  1,
 \\
\end{array}\right.
}
%\vspace{-0.3cm}
%\noindent
using the {vector}-valued function $q_\lbd(.) \mapsto e^{\lbd (L. - c)-\U}$ defined in~\eqref{eq:dual_sinkhorn_nonnormalized}.
%\cmt{The set $\{\beta \in \R^{2M} \textrm{ s.t. } \langle q_\lbd(\beta), \U\rangle \le 1\}$ is convex.}
\end{corollary}
\begin{proof}
See appendix \ref{proof_coro}. %A.1.
\end{proof}
Observe that the dual function $\MK^*_{\lbd,\le N}(\beta)$ is continuous at values $\dotp{q_\lbd(\beta^\star)}{\U}=1$.
Note also that the optimal transport matrix now is written
$P^\star_\lbd = N Q_\lbd(\beta)$ 
if $\dotp{Q_\lbd(\beta)}{\U} \le 1$,
and $P^\star_\lbd = N \tfrac{Q_\lbd(\beta)}{\dotp{Q_\lbd(\beta)}{\U}}$ otherwise.

\subsection{Gradient of $\MK_{\lbd,\le N}^*$}
\label{sec:normalized_sinkhorn_gradient}

From Corollary \ref{corollary:dual_simplexNmax}, we can express the gradient of $\MK_{\lbd,\le N}^*$ which is continuous
(writing $Q$ from Eq.~\eqref{eq:dual_sinkhorn_nonnormalized_matrix} in place of $Q_\lbd(\beta)$ to simplify expression)

\eql{\label{eq:dual_derivative_inegality}
\nabla \MK^*_{\lbd,\le N} (\beta) = 
\left\{\begin{array}{cll}
	N &\left(  Q \, \U; Q^\top \U \right)  % \left(  Q \U_{M}, \U_{M}^\Trp  Q  \right) 
	&\,   \text{ if } \langle Q , \U\rangle \le  1,  % \text{otherwise.}
\\[2mm]
	\dfrac	{N}{\langle Q , \U\rangle} & \left( Q \, \U; Q^\top \U\right)   
	&\,  \text{ if } \langle Q, \U\rangle \ge   1. \\
\end{array}\right.
}
In vectorial notation, we have a simpler expression using matrix $L$:
\eql{\label{eq:dual_derivative_inegality_vector}
\nabla \MK^*_{\lbd,\le N} (\beta) = 
\begin{cases}
	N \hspace*{+1mm} L^\top q_\lbd (\beta) \hspace*{-1mm}
	&\text{if } \langle q_\lbd (\beta)   , \U\rangle  \!\le\!  1,  % \text{otherwise.}
\\[3mm]
	N \dfrac{ L^\top q_\lbd (\beta)  }{\langle  q_\lbd (\beta)  , \U\rangle} \hspace*{-1mm}
	&\text{if } \langle q_\lbd (\beta)  , \U\rangle \!\ge \!  1. \\
\end{cases}
}

\noindent
We emphasis here that,
when restricting the Sinkhorn distance to  histograms on the probability simplex $\PS_{1,M}$ (\emph{i.e.} the special case where $N=1$ and $\langle Q, \U\rangle =  1$), or more generally on $\S_{\le 1}$,
we retrieve a similar expression than the one originatively demonstrated in~\cite{CuturiPeyre_smoothdualwasserstein}.

{Finally, the normalized Sinkhorn transport cost can be used in the generic optimization scheme due to the following property.}
\begin{proposition}\label{prop:lipschitz_dual}
The gradient $\nabla \MK^{*}_{\lbd,\le N}$ is a Lipschitz continuous function with constant $L_{\MK^*}$ bounded by  $2\lbd N$.
\end{proposition}
\begin{proof}
See appendix \ref{sec:proof_lipschitz_dual}.
\end{proof}

\subsection{Optimization using $\nabla \MK^*_{\lbd, \le N}$}\label{sec:opt:grad}

The binary-segmentation problem \eqref{eq:convex_segmentation_problem} with normalized Sinkhorn transport cost can be expressed as:
\eql{\label{eq:OT_segmentation_energy_reg}
\begin{split}
\min_u\hspace{0.5cm}  &\chi_{[0,1]^N}(u)+\rho \TV(u) 
+ \MK_{\lbd, \le N}{\left( H u, Au \right)}\\
&+ \MK_{\lbd, \le N}{\left( H (\U -u), B(\U-u) \right)}.
%\vspace{-0.2cm}
\end{split}
}

\noindent
Using the Fenchel transform, the problem \eqref{eq:OT_segmentation_energy_reg} can be reformulated as:
\begin{equation*}\label{pb:dual}
%\hspace*{-6mm}
\begin{split}
	\min_u \max_{\substack{v\\p_A^{},q_A^{}\\p_B^{},q_B^{}}}
	\hspace{0pt}
	& \quad \langle \Diff u,v\rangle\hspace{-1pt} + \chi_{[0,1]^{N}}(u) -\chi_{\lVert . \rVert_{\infty,2}\le\rho}(v)\hspace{-1pt}
	\\[-5mm]
	& \;+ \langle H u,p_A^{}\rangle + \langle H (\U_{} -u),p_B^{}\rangle
	\\ 
	& \;+ \langle A^{}u,q_A^{}\rangle\hspace{-1pt}+\langle B^{}(\U_{} -u),q_B^{}\rangle 
	\\ 
	& \;- \MK^*_{\lbd,\le N}(p_A^{},q_A^{})-\MK^*_{\lbd,\le N}(p_B^{},q_B^{}),%\vspace{-0.1cm}
\end{split}
\end{equation*}

\noindent
and can be optimized with the algorithm \eqref{algo:primaldual},
{setting $S_1^* = 0$ and $S_2^* = \MK^*_{\lbd,\le N}$}.
Using proposition \ref{prop:lipschitz_dual}, $\nabla G^*$ is a Lipschitz continuous function with constant $L_{G^*} = 2L_{\MK^*}$.
The definition of the diagonal preconditionners in the same as in problem~\eqref{pb:dual_type}, using Formula~\eqref{eq:diag_precond}.
The extension to multiphase segmentation is also analogue to problem~\eqref{pb:dual_type} (see the last paragraph of Section~\ref{sec:optim}).

% PREUVE inutile
%First, for $f = f_1+f_2$, the Lipschitz constant of the gradient is
%\eq{
%\begin{split}
%L_{f} 
%	&= \max_{x,y\not=x} \frac{\norm{\nabla f(x) - \nabla f(y)}}{\norm{x-y}} 
%	\\
%	&\le  \max_{x,y\not=x} \frac{\norm{\nabla f_1(x) - \nabla f_1(y)}}{\norm{x-y}} + \frac{\norm{\nabla f_2(x) - \nabla f_2(y)}}{\norm{x-y}} 
%	\\
%	&\le  L_{f_1} +  L_{f_2}.
%\end{split}
%}
%where equality holds for separable functions.
%Therefore, for 
%\eq{
%\begin{split}
%	G^*(p_A^{},q_A^{},p_B^{},q_B^{}) 
%	& = \MK^*_{\lbd,\le  N}(p_A^{},q_A^{}) + \MK^*_{\lbd,\le  N}(p_B^{},q_B^{}) 
%	\\
%	& \quad - \langle H \U_{},p_B^{}\rangle - \langle B^{}\U_{},q_B^{}\rangle
%\end{split}
%}
%so that
%\eq{
%L_{G^*} = 2 L_{\MK^*}.
%}

\paragraph{Advantages and drawback}
It has been shown in~\cite{CuturiPeyre_smoothdualwasserstein} that, aside from an increased robustness to outliers, the smoothing of the optimal transport cost offers significant numerical stability.
However, the optimization scheme may be slow due  to the use of unnormalized simplex $\S_{\le N}$.
In practice, the Lipschitz constant $L_{G^*}$ will be large for high resolution images (\emph{i.e.} large values of $N$) and for tight approximations of the $\MK$ cost (\emph{i.e.} $\lbd \gg1$). It will lead to low values of  time steps parameters in \eqref{eq:diag_precond} and   involve a slow explicit gradient ascent in the dual update \eqref{algo:primaldual}. %{time_step}
In such a case, we can resort to the alternative scheme proposed hereafter.\vspace{-0.4cm}
% in the next subsection.

\subsection{Primal-dual formulation of $\MK_{\lbd}$}
\label{sec:prox_sinkhorn_distance}

%\cmt{Attention !!! C'EST TRES BIZARRE DE CHANGER LA DEFINITION DE $\MK_\lbd$ en changeant la définition de l'entropie !!! J'ai donc  modifié le texte}

% \cmt{Il faut peut être inverser la présentation ? c'est justement ce qui est montré en premier dans la section optim 2.6}

An alternative optimization of \eqref{eq:OT_segmentation_energy_reg} 
consists in using the proximity map of $G^*$.
% , \cmt{as done previously for the $\MK^*$ distance}
Since we cannot compute such an operator for $\MK_\lbd^*$ in a closed form, {or in an effective way}, 
we resort instead to a biconjugaison, as previously done in §~\ref{sec:opt:biconj}.\vspace{-0.2cm} %Section~\ref{sec:biconjugaison}.

%Considering now the normalized function $\MK_\lbd (\alpha)$ using entropy normalization~\eqref{eq:normalized_entropy} on the set $\S$, we thus have
%$\MK_\lbd^* (\beta) 
%	= \frac{N}{\lbd} \langle Q_\lbd(\beta), \U\rangle 
%	=  g_\lbd^*(L \beta)
%.$

\paragraph{Biconjugaison}\label{sec:opt:biconj2}
For consistency with the previous section, we consider again the normalized entropy~\eqref{eq:normalized_entropy} to define the regularized cost function $\MK_{\lbd,N}$ on the set $\S$ in order to exhibit the factor $N$:
 \eql{\label{eq:MK_lbd_N}
	\begin{split}
 	\MK_{\lbd,N} (\alpha)
	& %=\dotp{p^\star_\lbd}{c}
	:= \hspace*{-3mm} 
	\min_{\substack{p\in \RR^{M^2}\\\text{s.t. } p \ge \O, \; L^\Trp  p = \alpha}} \hspace*{-3mm}  
	%\dotp{p}{c + \tfrac{1}{\lbd} (\log p - \U \log N)} + {\chi_{\S}(\alpha)}. 
	\dotp{p}{c + \tfrac{1}{\lbd} \log (p / N)} + {\chi_{\S}(\alpha)}.
\end{split}
}

Simple calculations show that the dual conjugate in Eq.~\eqref{eq:dual_sinkhorn_nonnormalized} becomes
\eq{\MK_{\lbd,N}^* (\beta) 
	= \tfrac{N}{\lbd} \langle q_\lbd(\beta), \U_N \rangle 
	%= \frac{N}{\lbd} L
	%=  g_\lbd^*(L \beta)
.}
Introducing the dual conjugate function 
\eql{\label{eq:g_lbd}
	g_\lbd^*(q) := \tfrac{N}{\lbd} \dotp{e^{\lbd (q-c) - \U}}{\U}
	%g_\lbd^*(\beta) = \frac{N}{\lbd} \langle q_\lbd(\beta), \U_N \rangle 
} 
that is convex and continuous, we have
%such that $\MK_\lbd^* (\beta) = g_\lbd^*(L \beta)$, we have
\eql{\label{eq:primal_dual_MK_lbd_splitting}
\begin{split}
	\MK_{\lbd,N}^*(\beta) 
	& = g_\lbd^*( L\beta ) = \max_r\; \dotp{r}{L\beta} - g_\lbd(r)
	\\
\end{split}
}
\text{and}
\eq{
\begin{split}
	\MK_{\lbd,N}(\alpha) &
	= \max_{\beta}\; \dotp{\alpha}{\beta}  - g_\lbd^*( L\beta )\\
	%&= \max_{\beta}\; \dotp{\alpha}{\beta} + \min_r \dotp{r}{c-L\beta} + \chi_{\cdot \ge \O}(r)
	%\\
	%& 
	%= \max_{\beta} \; \min_{p \in \R^{M_a \cdot M_b}} \; \dotp{\alpha}{\beta} - \dotp{p}{L\beta} + f(p)
	%\\
	& = \min_{r} \; \max_{\beta} \; 
	 \dotp{\alpha - L^\Trp  r}{\beta} + g_\lbd(r).
	\\
%	& =
%	\\
\end{split}
}

This reformulation, combined with the following expression of the proximity function of $g_\lbd$, enables to solve efficiently the segmentation problem with $\MK_{\lbd,N}$. % \cmt{a reformuler ?}.
\begin{proposition}\label{prop:prox_bidual}
	The proximity operator of the function $g_\lbd$, the conjugate of $g_\lbd^*$ defined in Eq.~\eqref{eq:g_lbd},
	% $ g_\lbd^*(q) = \frac{N}{\lbd} \dotp{e^{\lbd (q-c) - \U}}{\U} $
	is
	\begin{equation}\label{eq:prox_dual}
	\begin{split}
		%\text{prox}_{\;\sig g_\lbd^*}(p) 
		%& = p - \frac{1}{\lbd} W\left(  \lbd \sig N e^{\lbd (p-c) - \U}  \right),
		%\\
		\text{prox}_{\;\tau g_\lbd}(r) 
		& =  \frac{\tau}{\lbd} W\left(  \tfrac\lbd\tau N e^{\lbd (\frac{r}{\tau}-c) - \U}  \right)
	\end{split}
	\end{equation}
where $W$ is the \emph{Lambert function}, such that $w=W(z)$ is solution of $we^w = z$. The solution is unique as 
$z\ge  0$.
%$z=\lbd \tau N e^{\lbd (p-c) - \U}\ge  0$.
%\vspace*{-1mm}
\end{proposition}
\begin{proof}
See appendix \ref{proof:prox_bidual}.
%\vspace*{-1mm}
\end{proof}

%\begin{remark}
Note that the Lambert function can be evaluated very fast,
using an efficient parallel algorithm that requires a few iterations~\cite{Lambert}.
%\vspace*{-2mm}
%\end{remark}

\paragraph{Segmentation problem} 
%Plugging the previous expression into Eq.~\eqref{pb:dual_type} enables us to solve it using algorithm~\eqref{algo:primaldual}.
%Indeed, introducing new primal variables $r_A,r_B \in \R^{M^2}$ related to transport mappings for the binary segmentation problem, we recover the following primal dual formulation (extension for multi-phase segmentation is straightforward using Section~\ref{sec:optim})
Using Formula~\eqref{eq:g_lbd} into Eq.~\eqref{pb:dual_type} provides the following primal dual problem
\eql{\label{pb:dual_type_MK_lbd}
%\small
\begin{split}
\hspace*{-1mm}
	\min_{\substack{u\\r_A, r_B}}%\substack{u \in \R^N\\r_A,r_B \in \R^{M^2}}}
	\; \max_{\substack{v\\p_A^{}, q_A^{}\\p_B^{}, q_B^{}}}
	\;
	 %\quad 
	 & \quad \dotp{ \Diff u}{v} - \chi_{\lVert.\rVert_{\infty,2} \le \rho}(v) 
	 \\[-5mm]
	 & + \dotp{ A^{}u}{q_A^{}} + \dotp{ B^{}(\U-u)}{q_B^{}}
	 \\
	 & + \dotp{ Hu}{p_A^{}} + \dotp{ H(\U -u)}{p_B^{}} 
	 \\
 	 & - \dotp{r_A}{L{\small \begin{bmatrix}p_A\\q_A\end{bmatrix}}} 
 	    - \dotp{r_B}{L{\small \begin{bmatrix}p_B\\q_B\end{bmatrix}}}   
 	%- \dotp{L^\Trp  r_B}{[p_B^\Trp ,  q_B^\Trp ]^\Trp }
 	 \\%[-3mm]
 	 & + \chi_{[0,1]^{N}}(u) +  g_\lbd(r_A) + g_\lbd(r_B)
\end{split}
}
Again, we can use a variant of the algorithm described in \eqref{algo:primaldual}, 
augmented by primal variables $r_A$ and $r_B$.
The operator $K$ is the same than in Formula~\eqref{eq:biconjugate_K}.
The proximity function $\Prox_{\tau R}$ corresponds to
\eq{\begin{split}
	& \Prox_{\,\chi_{[0,1]^{N}}(u) + \tau g_\lbd(r_A) + \tau g_\lbd(r_B)} (u,r_A,r_B) \\
	& \quad=  \left( \Proj_{[0,1]^N}(u) ; \Prox_{\tau g_\lbd}(r_A) ; \Prox_{\tau g_\lbd}(r_B) \right)
	.
\end{split}}

%{\color{magenta}
%A FINIR EN ETOFFANT :
%detailler l'optimisation ?
%Multi phase ?
%}
% \cmt{The definition of the diagonal preconditionners in this case is very similar to the Formula~\eqref{eq:diag_precond}.}

\subsection{Comparison of the two approches}
\label{sec:comp_prox_grad}

% dire que c'est pareil même s'il y a une difference de definition !

In the previous sections, two variants of the entropic regularized transportation problem have been introduced: $\MK_{\lbd, \le  N}$ in~\eqref{eq:sinkhorn_distanceN}  and $\MK_{\lbd, N}$ in~\eqref{eq:MK_lbd_N}.
We underline the fact that, while having different definitions, these two metrics provide the same numerical result for any of the segmentation problems investigated in this paper, as the corresponding primal-dual optimal solutions verify the same property (\emph{i.e.} the mass of histograms in $\Delta$ cannot exceed the total number of pixels $N$) for which the metrics behave identically.

%%%%%%%%%%%%%%%%%%%%%%%%%%%%%%%%%%%%%%%%%%%%%%%%%%%

\section{Segmentation Experiments}\label{sec:seg_exp}

\paragraph{Experimental setting}

\begin{figure*}[!ht]
\centering{
	\begin{tabular}{cccc}
	\includegraphics[width=0.21\textwidth]{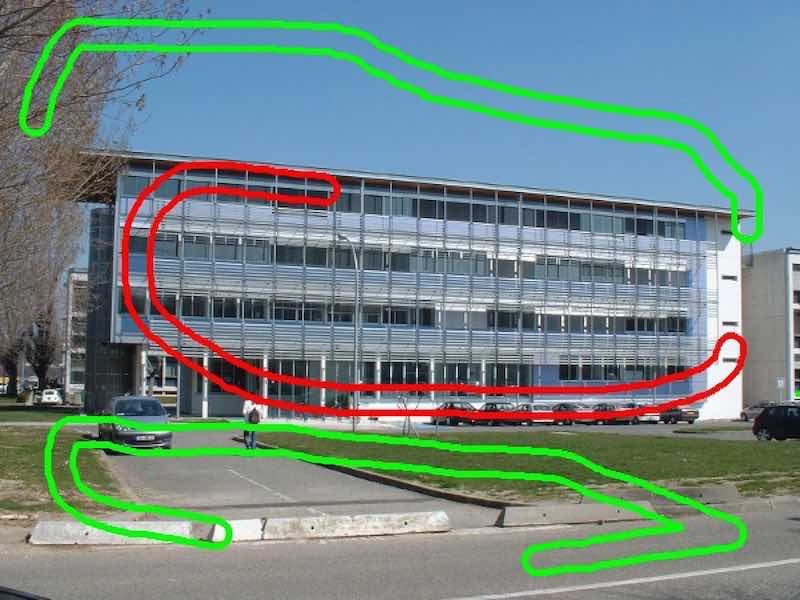} &
	\includegraphics[width=0.21\textwidth]{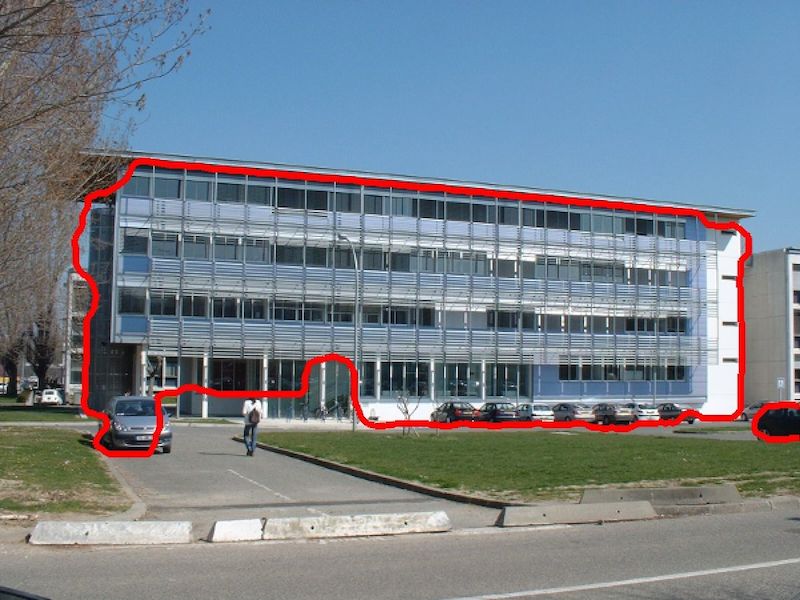} &
	\includegraphics[width=0.21\textwidth]{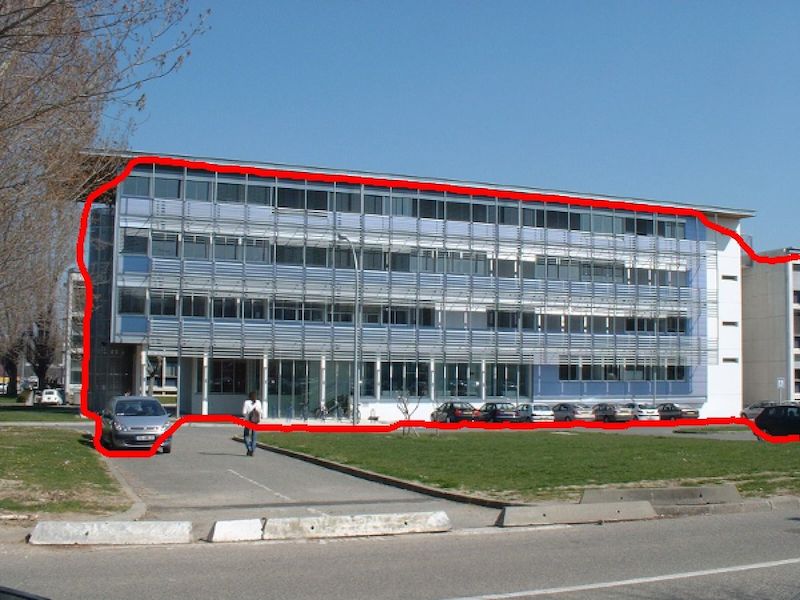} &
	\includegraphics[width=0.21\textwidth]{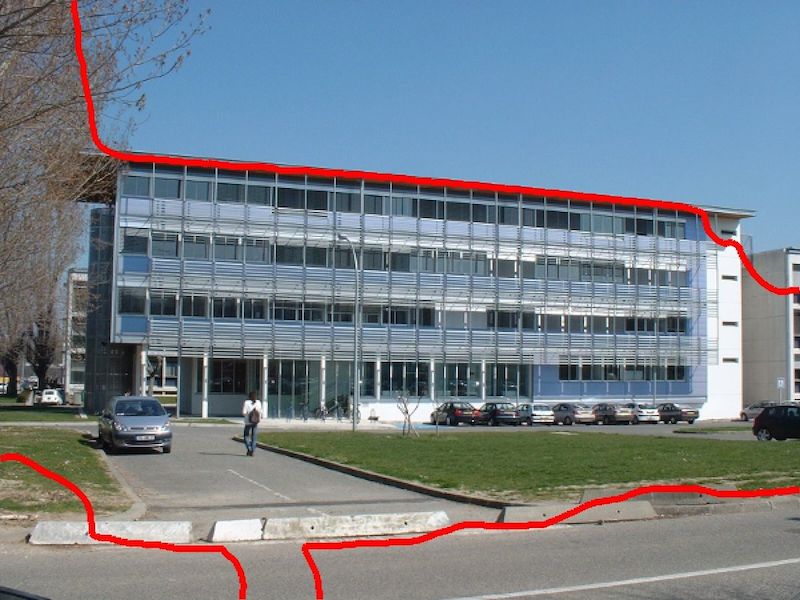} 
	\\
%	Input & $\lbd = \infty$ & $\lbd = 1000$ & $\lbd = 100$\\
	\includegraphics[width=0.21\textwidth]{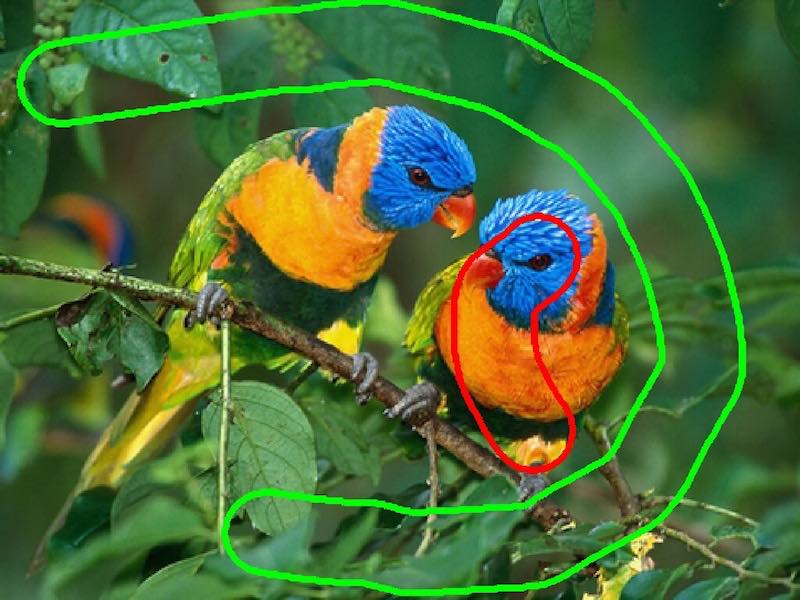} &
	\includegraphics[width=0.21\textwidth]{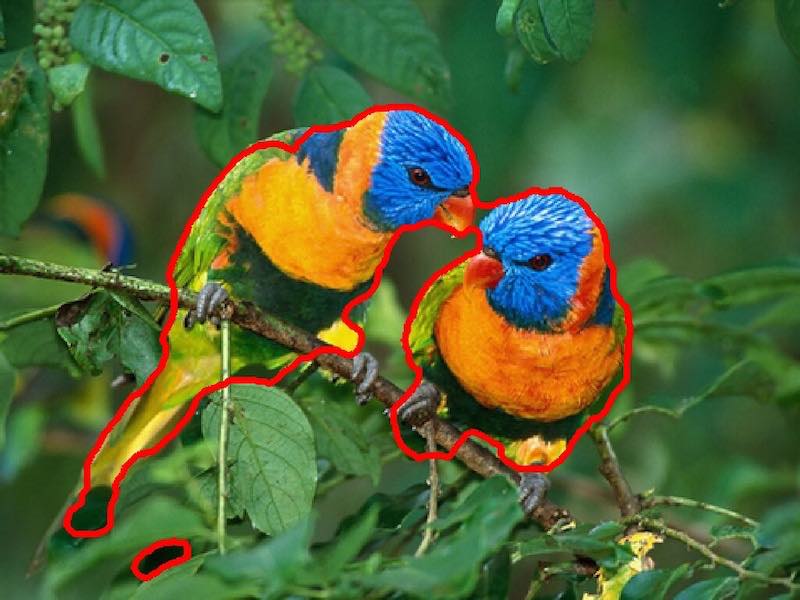} &
	\includegraphics[width=0.21\textwidth]{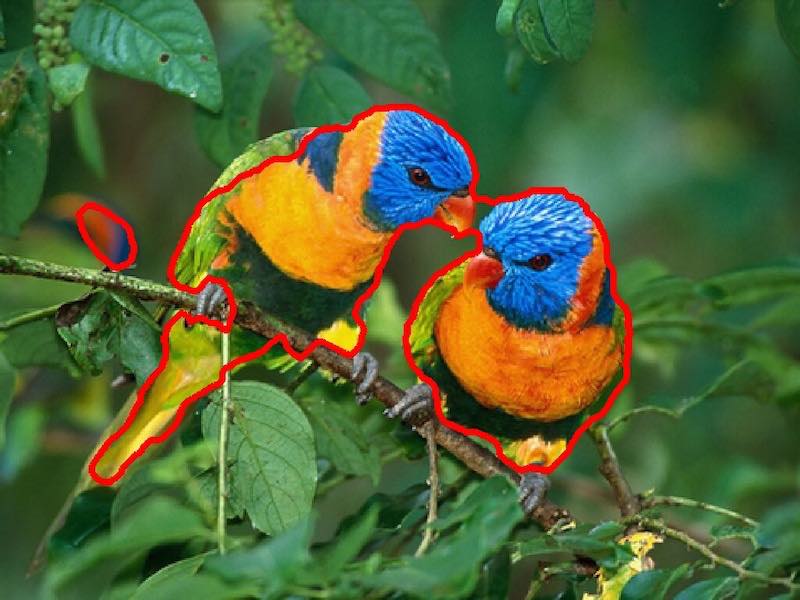} &
	\includegraphics[width=0.21\textwidth]{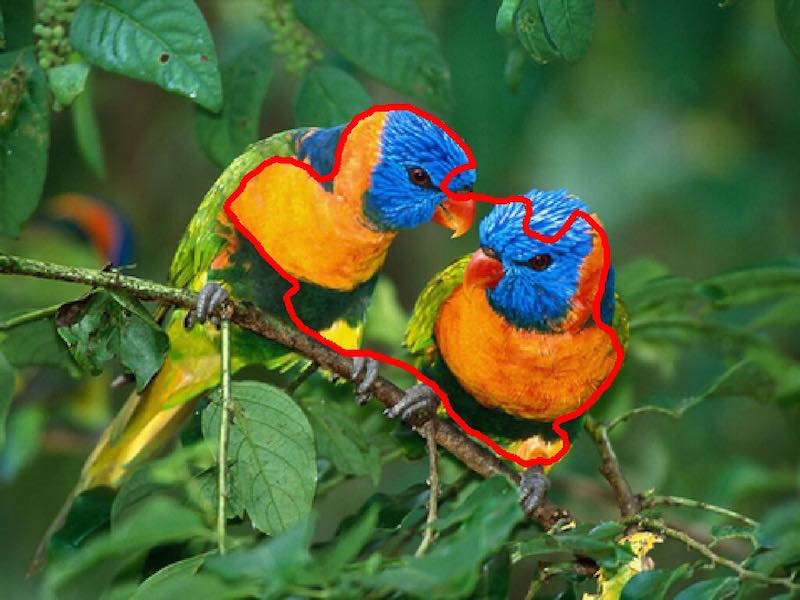} 
	\\
	\includegraphics[width=0.21\textwidth]{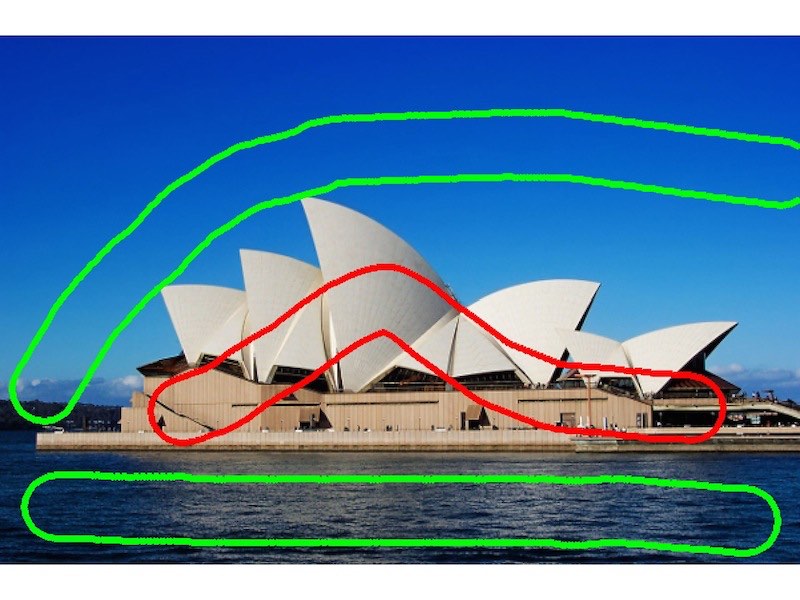} &
	\includegraphics[width=0.21\textwidth]{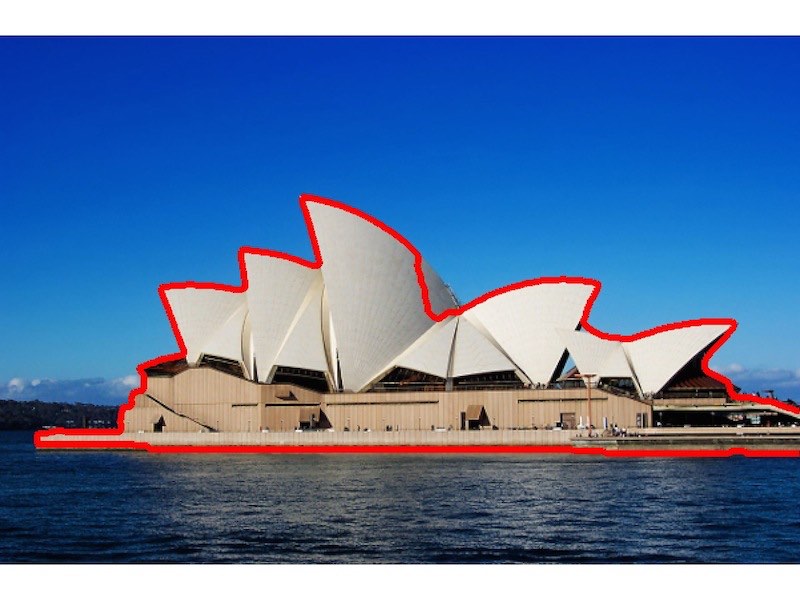} &
	\includegraphics[width=0.21\textwidth]{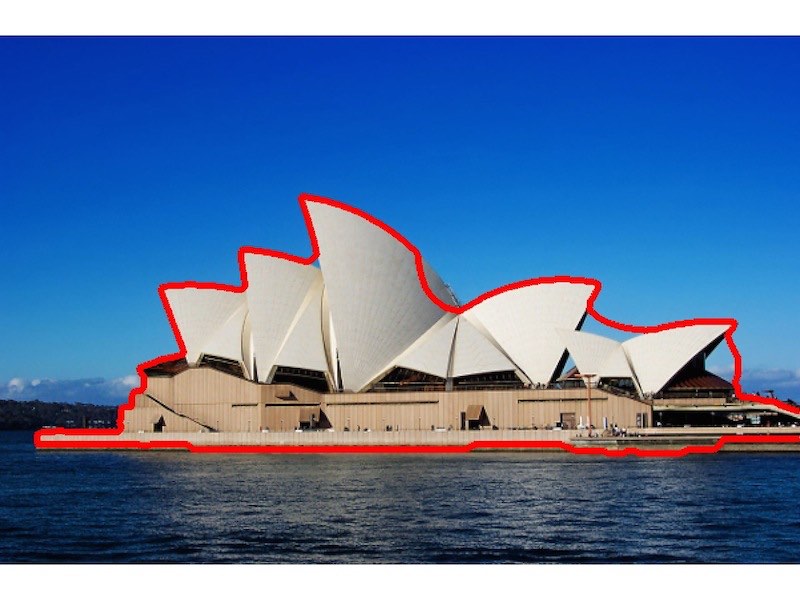} &
	\includegraphics[width=0.21\textwidth]{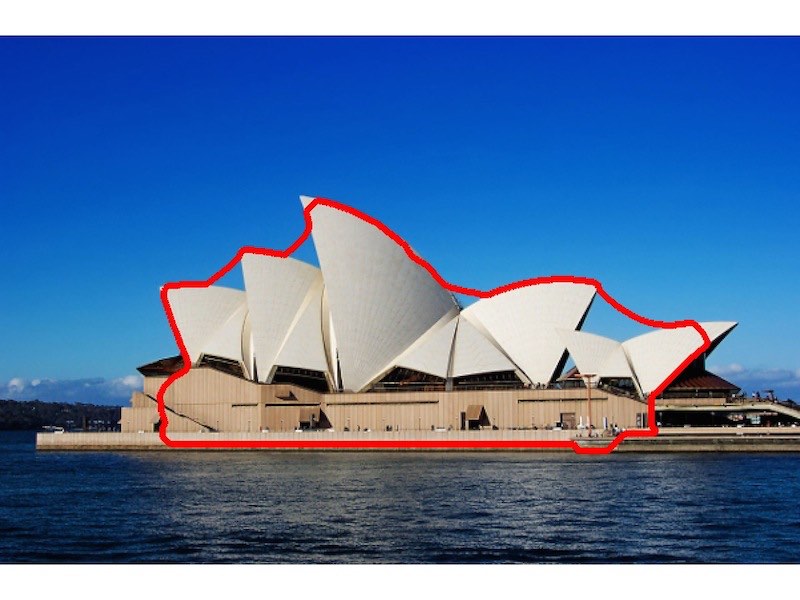} 
	\\
	Input & $\lbd = \infty$ & $\lbd = 100$ & $\lbd = 10$\\	
	\end{tabular}
}

\caption{\textbf{Comparison of segmentations obtained from the proposed models with $\MK_\lambda$ cost function.} 
The input images are partially labelled by the user, and the corresponding areas are used to compute the reference color distributions $a$ and $b$. 
The segmented regions, obtained from the thresholded solution, are contoured in red.
Different values of the regularization parameter $\lambda$ of the transport cost are used, $\lbd=\infty$ corresponding to the non-regularized model.
%The non-regularized model corresponds to $\lbd=\infty$, increasing regularization effects are shown.
\vspace*{5mm}~
}
\label{fig:comp_lambda}
\end{figure*}

In this experimental section, exemplar regions are defined by the user with scribbles (see for instance Fig.~\ref{fig:thres}) or bounding boxes (Fig.~\ref{fig:zebras}). 
These regions are only used to built prior histograms, so erroneous labelling is tolerated.
The histograms $a$ and $b$ are built using hard-assignment on $M=8^n$ clusters, which are obtained with the K-means algorithm. 
%\cmt{il faut détailler ?}

We use either RGB color features ($\F = \Id$ and $n=d=3$) or the gradient norm of color features ($\F = \lVert \Diff . \rVert$ computed on each color channel, so that $n=3$).
The cost matrix is defined from the Euclidean metric $\lVert \cdot \rVert$ in $\R^n$ space, combined with the concave function $1-e^{-\gamma{\lVert \cdot \rVert}{}}$, which is known to be more robust to outliers~\cite{Rubner2000}.
Approximately $1$ minute is required to run $500$ iterations and segment 
a 1 Megapixel color image.

To account for the case where a region boundary coincide with the image border $\partial \Om$, we enlarge the size of the domain $\Om$ by 1 pixel and we force variable $u$ to be null on the border. 
That way, the model does not favor regions that lie across the boundary.

Throughout the experiments, the diagonal preconditioning is defined using Formula~\eqref{eq:diag_precond} with $r=\delta =1$.
We have observed an impressive convergence acceleration (approximately 3 orders of magnitude) due to preconditioning.

\paragraph{Projection onto the simplex}
The projector onto the discrete probability set $\PS_{1,K}$ %and other related simplex sets 
can be computed in linear time complexity, see for instance~\cite{Condat_Simplex}.
% The O(K) algorithm mentioned in footnote 1 relies on partitioning and median-finding, but in practice other algorithms with observed linear time are often a better choice. A full description of these algorithms is out of scope, but see e.g. Table 1 of the following ref, where an empirical comparison is also provided:

\paragraph{Thresholding}
As previously stated, the segmentation map $u^*$ obtained by minimizing the functional \eqref{modele_seg} is not binary in general.
The final segmentation is obtained by thresholding the global minima $u^\star$ with $t=\tfrac{1}{2}$ (see Section~\ref{sec:threshold}). This leads to the best compromise in our experiments, as illustrated in Figure~\ref{fig:thres} that shows the influence of the threshold $t$ used to get a binary segmentation.
% In our experiments, we chose $t=0.5$. 

\begin{figure}[!h]
    \centering{
    \begin{tabular}{ccc}
    \includegraphics[width=0.14\textwidth]{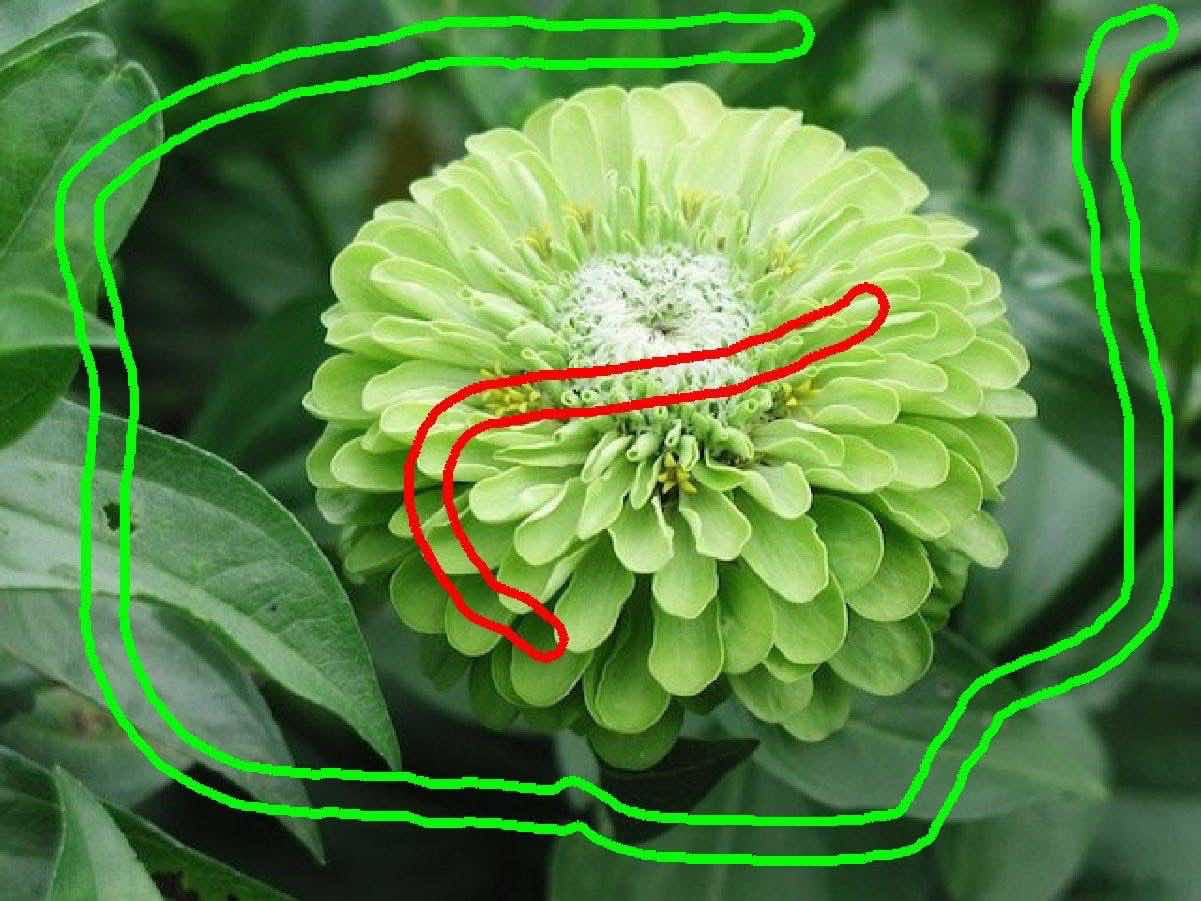} &\hspace{-0.2cm}
    \includegraphics[width=0.14\textwidth]{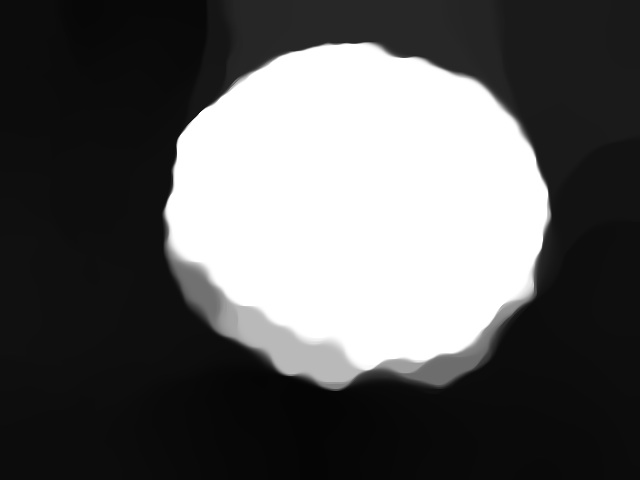} &\hspace{-0.2cm}
    \includegraphics[width=0.14\textwidth]{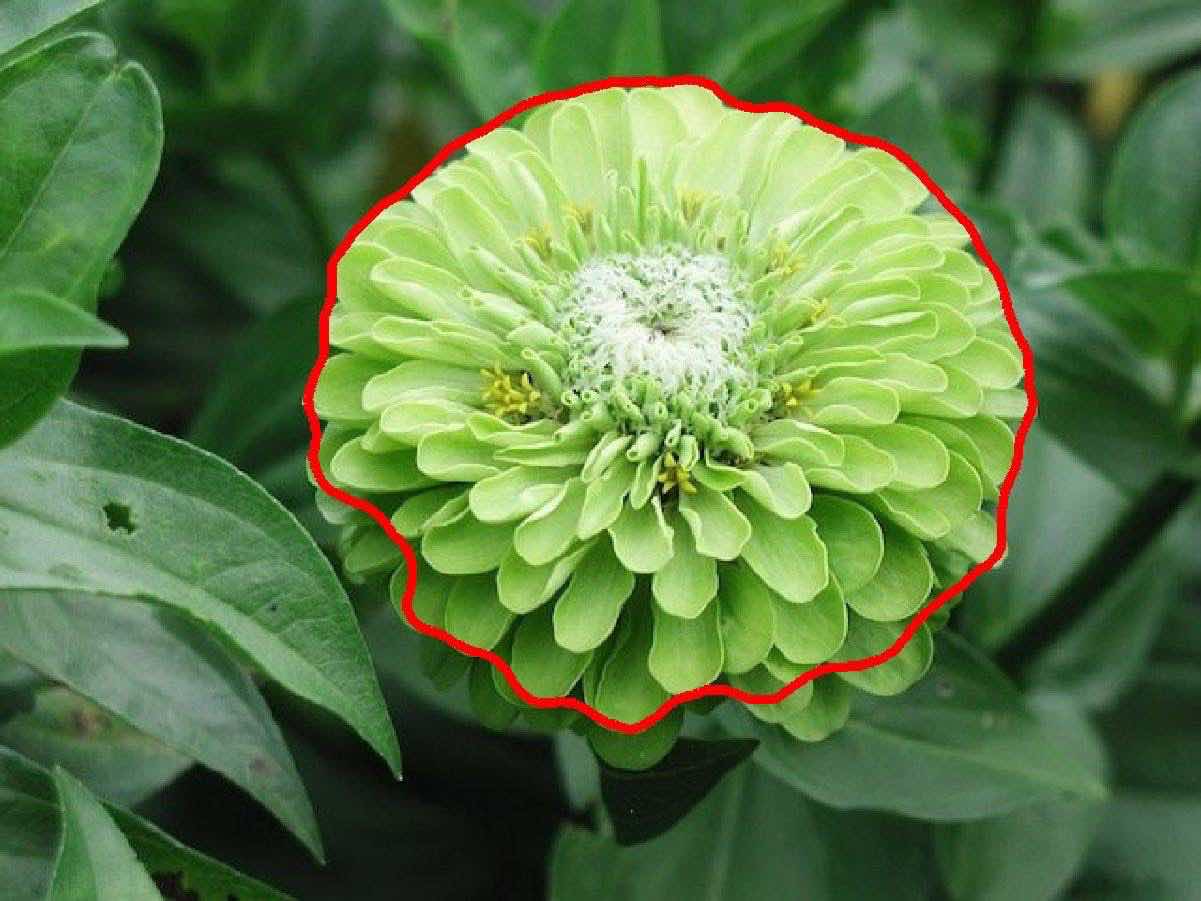}
    \\
    Input & $u$  & $t=0.5$ 
    \\
    \includegraphics[width=0.14\textwidth]{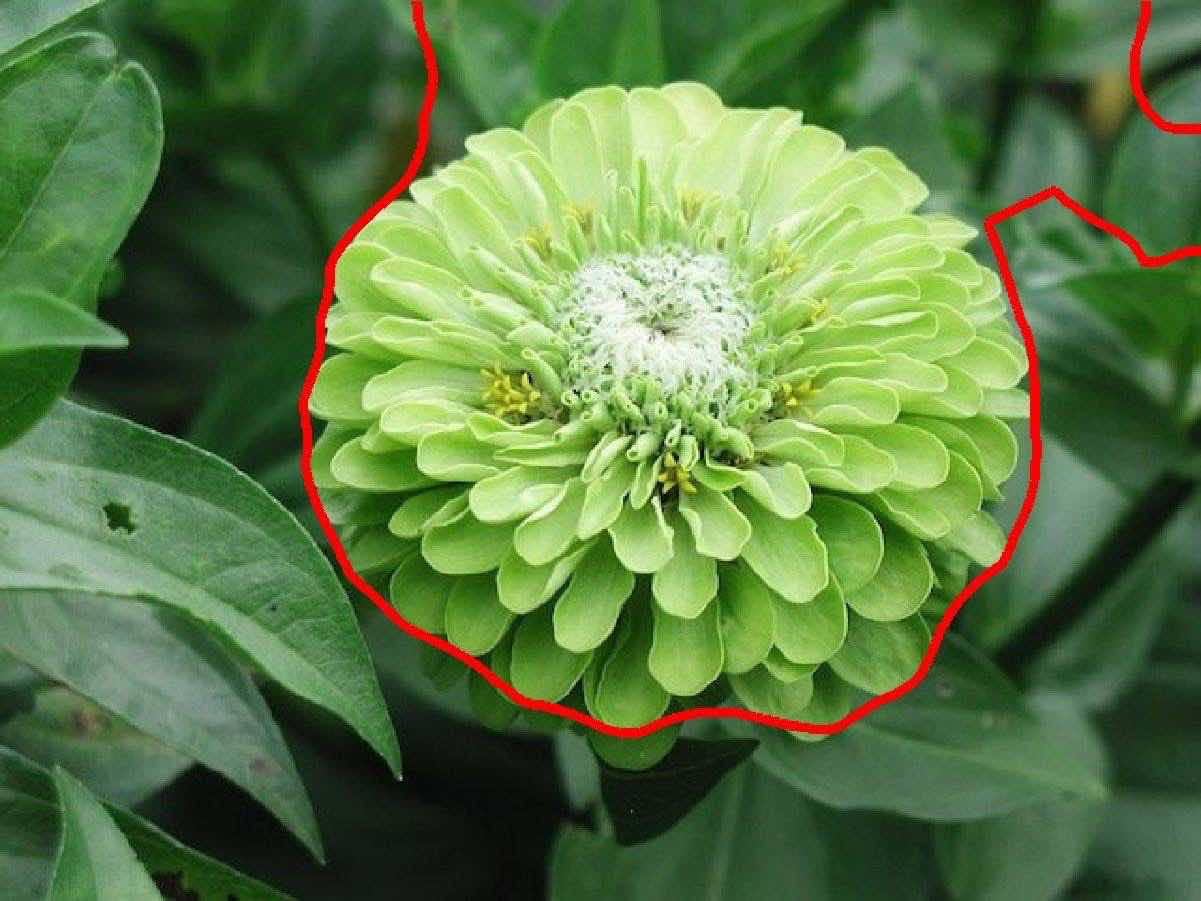} &\hspace{-0.2cm}
    \includegraphics[width=0.14\textwidth]{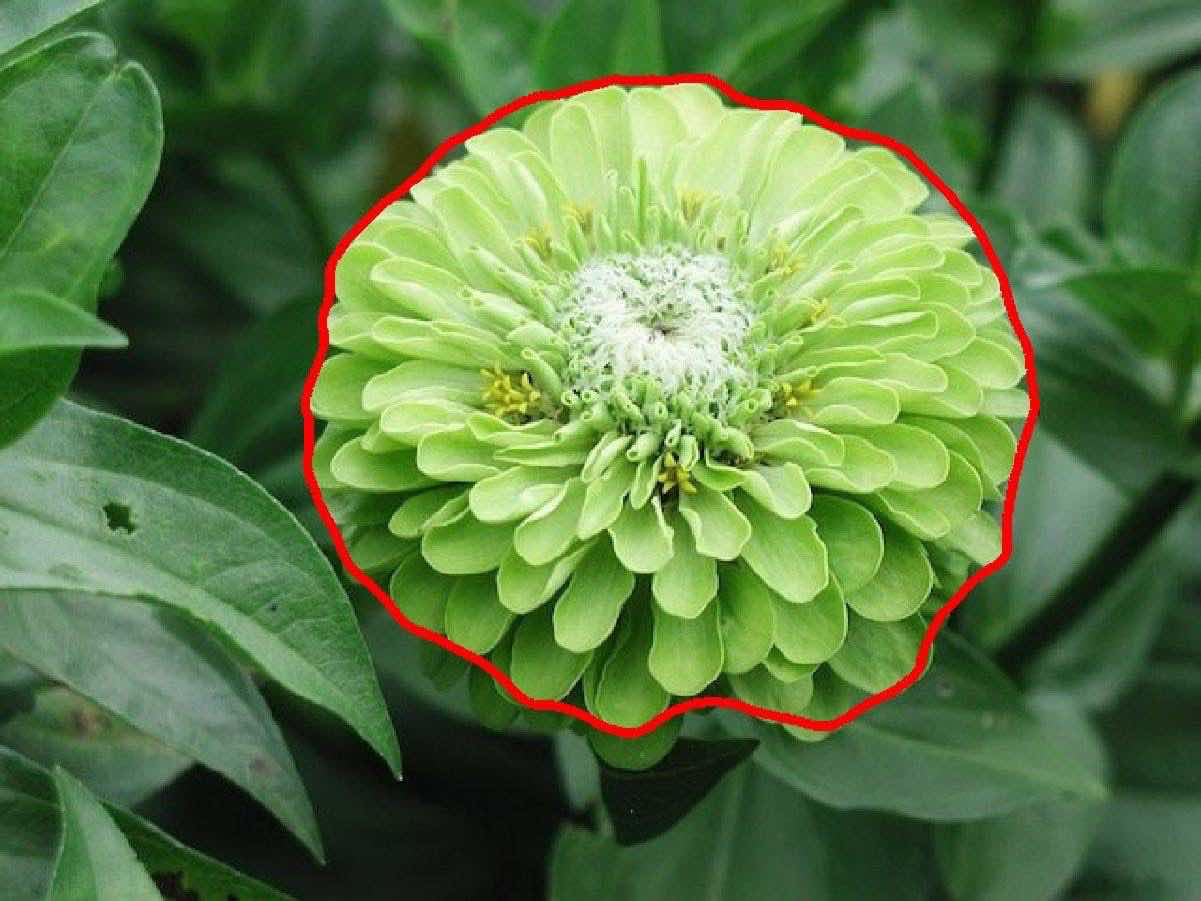} & \hspace{-0.2cm}
    \includegraphics[width=0.14\textwidth]{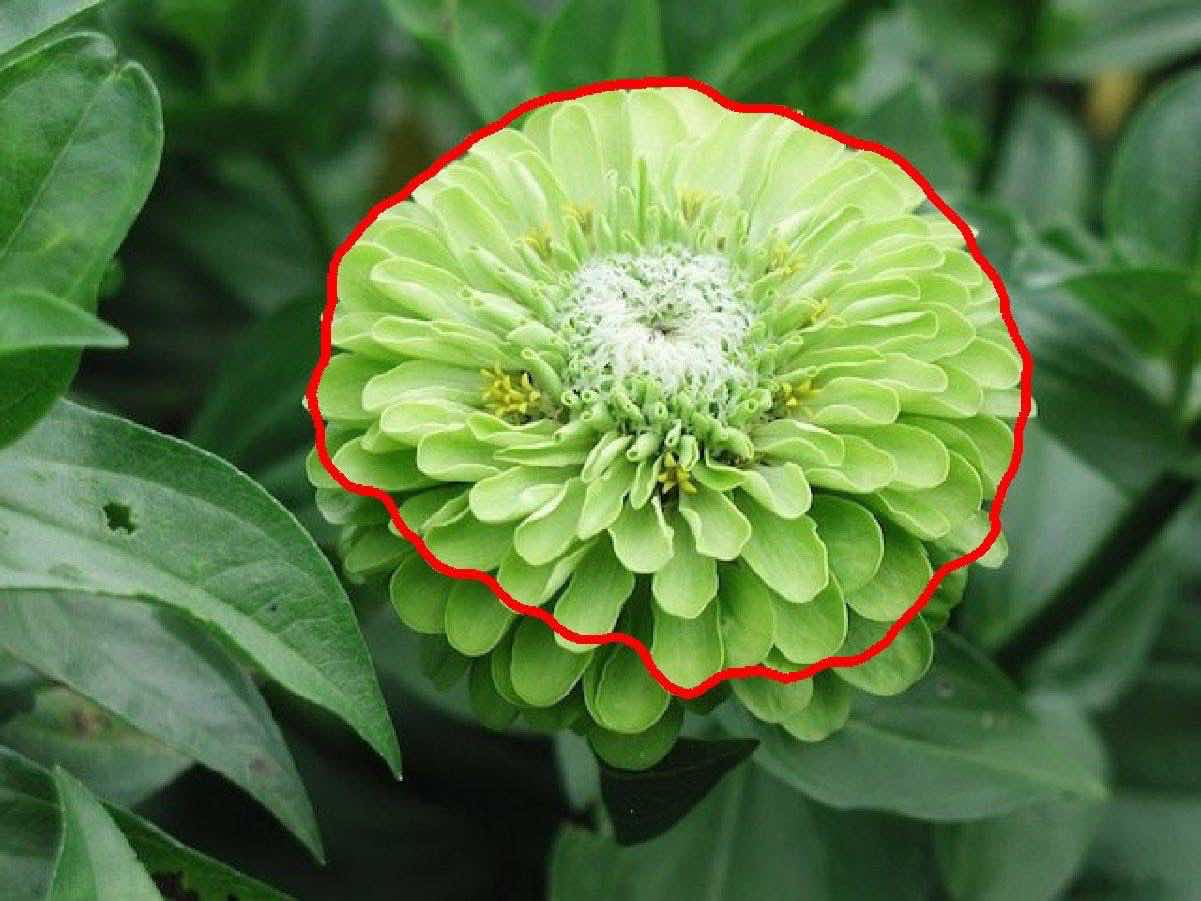} 
    \\
    $t=0.1$ & $t=0.2$ & $t=0.9$
    \end{tabular}
    }
\caption{\textbf{Illustration of the image segmentation method with $\MK$ transport cost and the influence of the final threshold parameter $t$ on the segmentation result.} %(here using non regularized optimal transport, \emph{i.e.} $\lbd = \infty$).
The user defines scribbles which indicates the object to be segmented (here in red) and the background to be discarded (in green). 
The output image $u$ is a regularized weight map that gives the probability of a pixel to belong to the object.
This probability map $u$ is finally thresholded with a parameter $t$ to segment the input image into a region $R_t(u)$, which contour is displayed in red.
%
%In the rest of the paper, $t=0.5$ is always used, but other strategies may be defined, such as selecting the threshold value that minimizes the non-relaxed energy \eqref{eq:segmentation_energy}.
\vspace*{5mm}~
}
\label{fig:thres}
\end{figure}

\subsection{Regularized vs non-regularized $\MK$ distances}

%\cmt{AJOUT : dire que l'on note $\MK_\lbd$ pour l'optimisation en utilisant $\MK_{\lbd,\le N}$ et $\MK_{\lbd, N}$ puisque ces deux méthodes donnent le même résultat (il s'agit juste de vitesse d'execution)}
As previously discussed in Section~\ref{sec:comp_prox_grad}, the solutions when using the gradient of $\MK_{\lbd,\le N}$ or the proximity operator based on $\MK_{\lbd,N}$ are the same when $N=\abs{\Om}$, even if the respective optimization schemes are different.
As a consequence, we simply denote by $\MK_\lbd$ when referring to these methods. We also indicates $\MK$ or $\lbd = \infty$ when not using any regularization. %, that is $\MK$ in place of $\MK_\lbd$.

We first illustrate the influence of the $\lambda$ parameter in the regularized distance $\MK_\lbd$. 
Figure~\ref{fig:comp_lambda} gives a comparison between the non-regularized model, quite fast but using high dimensional representation~\eqref{pb:dual_type_MK}, with the regularized model, using either a smooth low dimensional formulation~\eqref{eq:OT_segmentation_energy_reg} or a smooth high dimensional representation~\eqref{pb:dual_type_MK_lbd}. One can see that setting a large value of $\lbd$ gives interesting results. 
%Namely, on the parrot example, the third parrot in the left part of the background is recovered. Moreover, the non-regularized model has wrongly segmented a leaf and this error is not present with the entropy  regularization. 
%When taking too low value of the regularization parameter,  ($\lbd=10$), the segmentation results get worse.
On the other hand, using a very small value of $\lbd$ always yields poor segmentation results. In practice, if not specified otherwise, we consider $\lbd=100$ in our experiments, as higher values may lead to numerical issues (floating point precision).

\subsection{Comparisons with other segmentation models including Wasserstein distance}

We first exhibit the advantage of considering global data terms over histograms, {such as in Eq.~\eqref{eq:segmentation_energy}}. We present a comparison with the convex model proposed in \cite{bresson_wasserstein09} that includes a local data term over color distributions:
\begin{equation*}
	\begin{split}
	\tilde E(u)
	= \rho \TV(u) + \sum_{x\in\Omega} & \;\; \MK(a,h_{V(x)})u(x) 
	\\
	+ &   \;\; \MK(b,h_{V(x)})(1-u(x)) % +\chi_{\{0,1\}^N}(u) j'ai enlevé ce terme car on ne l'a pas mis dans J(u)
	\end{split}
\end{equation*}
where $h_{V(x)}$ is the color distribution over the neighborhood $V(x)$ of pixel $x$. This model, that can be optimized globally~\cite{Yildi12}, measures the local color distribution of the image with respect to the reference foreground and background distributions $a$ and $b$.
As illustrated in Figure~\ref{fig:comp2}, such local model is not able to perform a global segmentation.
Here the orange colors are more probable in the region related to the butterfly, so in small neighborhoods the flowers are classified as the butterfly, and the darker regions are segmented as being in the background. This example illustrates  the importance of considering global histogram comparisons to get a global segmentation of an image. Indeed, the global distance between histograms (c) is able to recover the butterfly, whereas the local approach (b) completely fails. Local approaches are therefore only relevant when the local histograms correctly approximate the global ones. %\vspace{-0.2cm}

\begin{figure}[!htb]
    \begin{center}
    \begin{tabular}{ccc}
    \hspace{-0.12cm}\includegraphics[width=2.55cm]{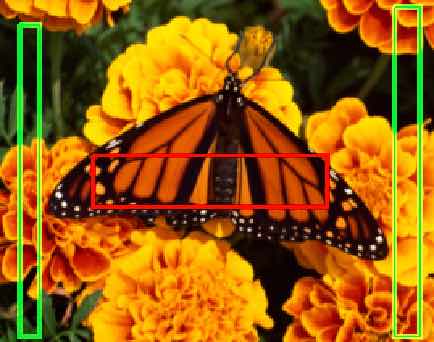}&\hspace{-0.3cm}
    \includegraphics[width=2.55cm]{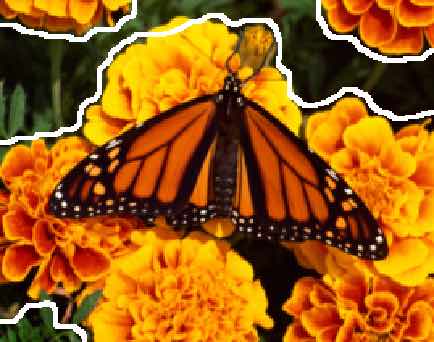}&\hspace{-0.3cm}
    \includegraphics[width=2.55cm]{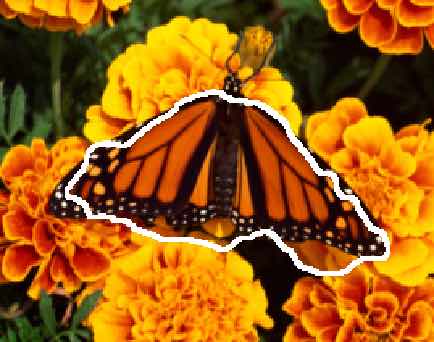}\\
    \hspace{-0.12cm}(a) $I$, $\color{red} a$, $\color{green} b$ &\hspace{-0.3cm}(b) Local \cite{bresson_wasserstein09}&\hspace{-0.3cm}(c) Global  
    \end{tabular}
    \caption{\label{fig:comp2} \textbf{Comparison with a local model.} (a) Input image and regions where reference distributions $a$ and $b$ are estimated. The segmentation fails for the local histogram model (b) as it classifies the orange areas in the first class and the darker ones in the second class. The global histograms on the segmented zones are not close to the given ones, contrary to the global model (c).
    \vspace*{2mm}~
    } 
    \end{center}
\end{figure}

\begin{figure*}[!htb]
\centering
\begin{tabular}{cccc}
\includegraphics[width=0.21\textwidth]{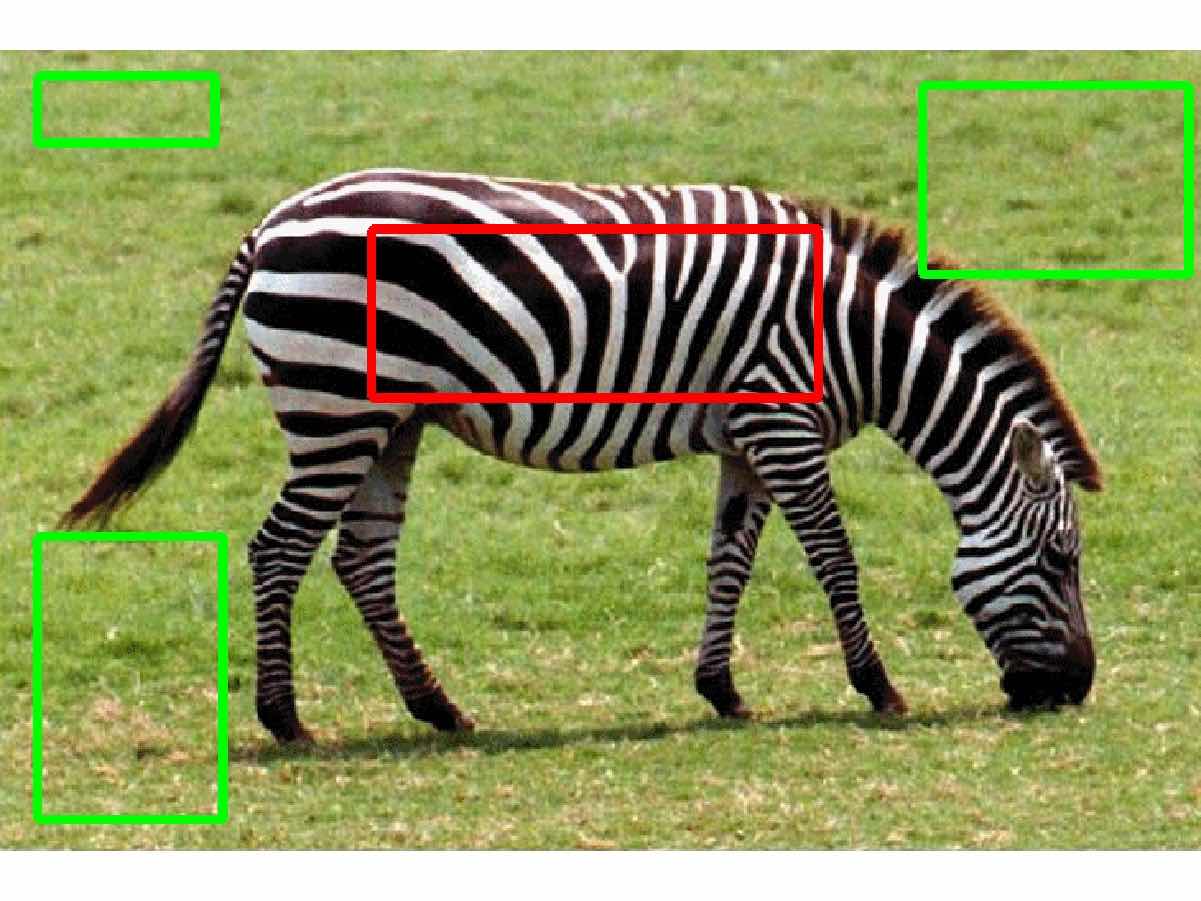} &
\includegraphics[width=0.21\textwidth]{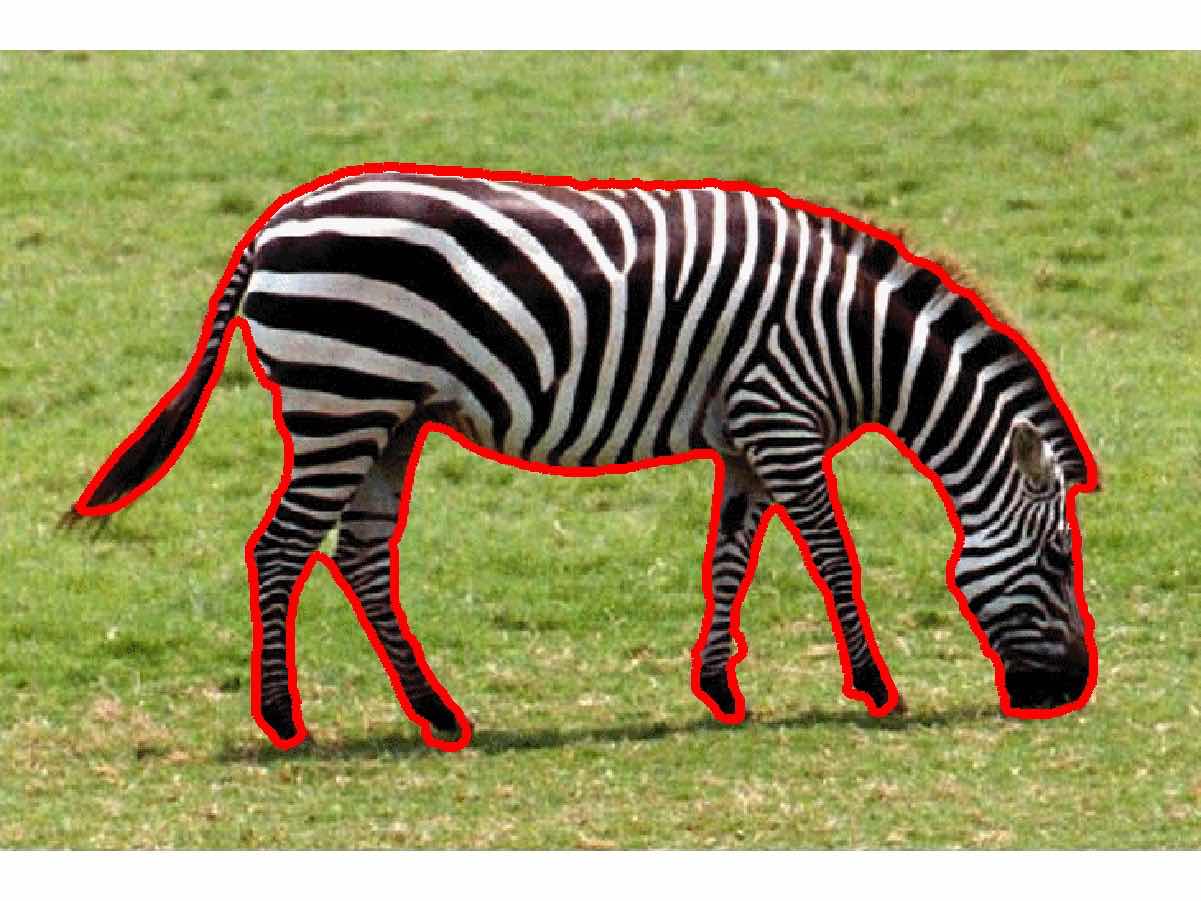} &
\includegraphics[width=0.21\textwidth]{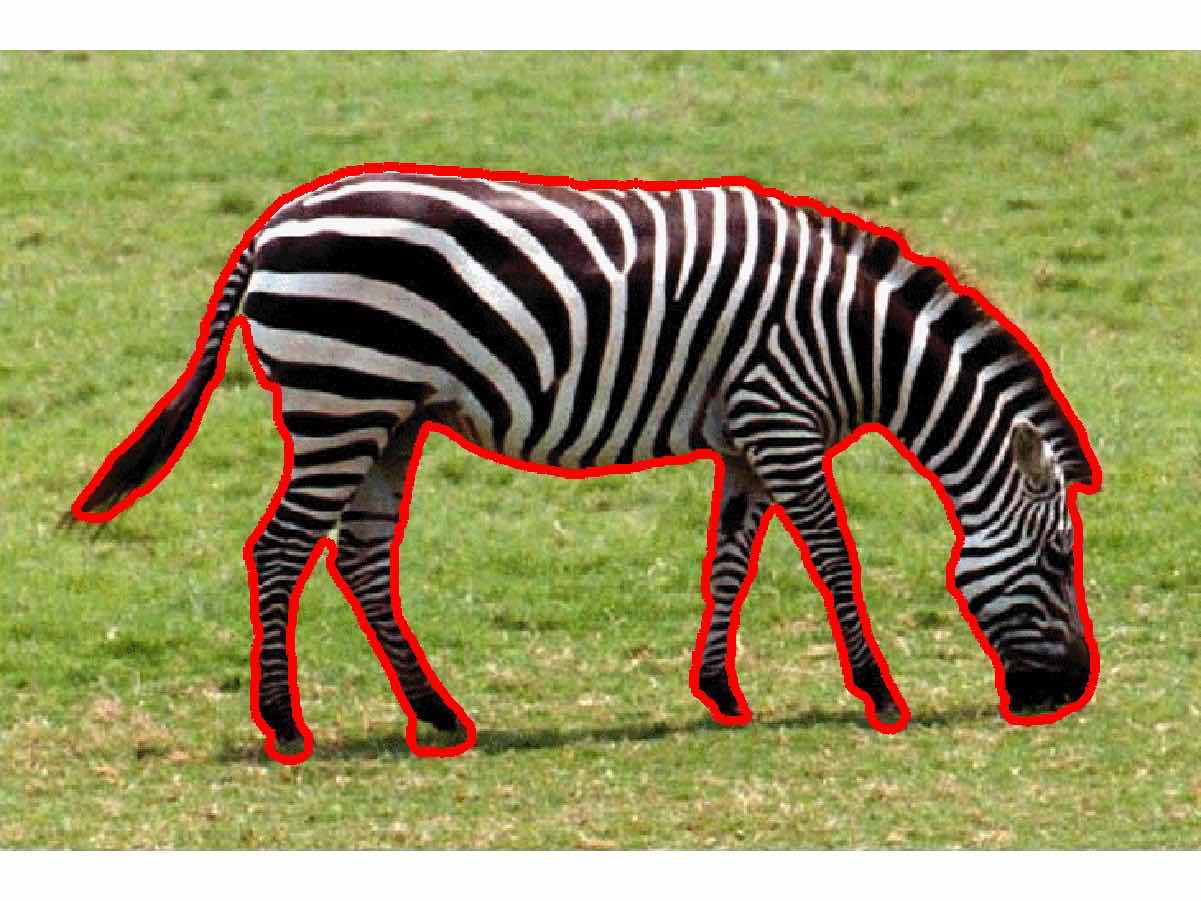} &
\includegraphics[width=0.21\textwidth]{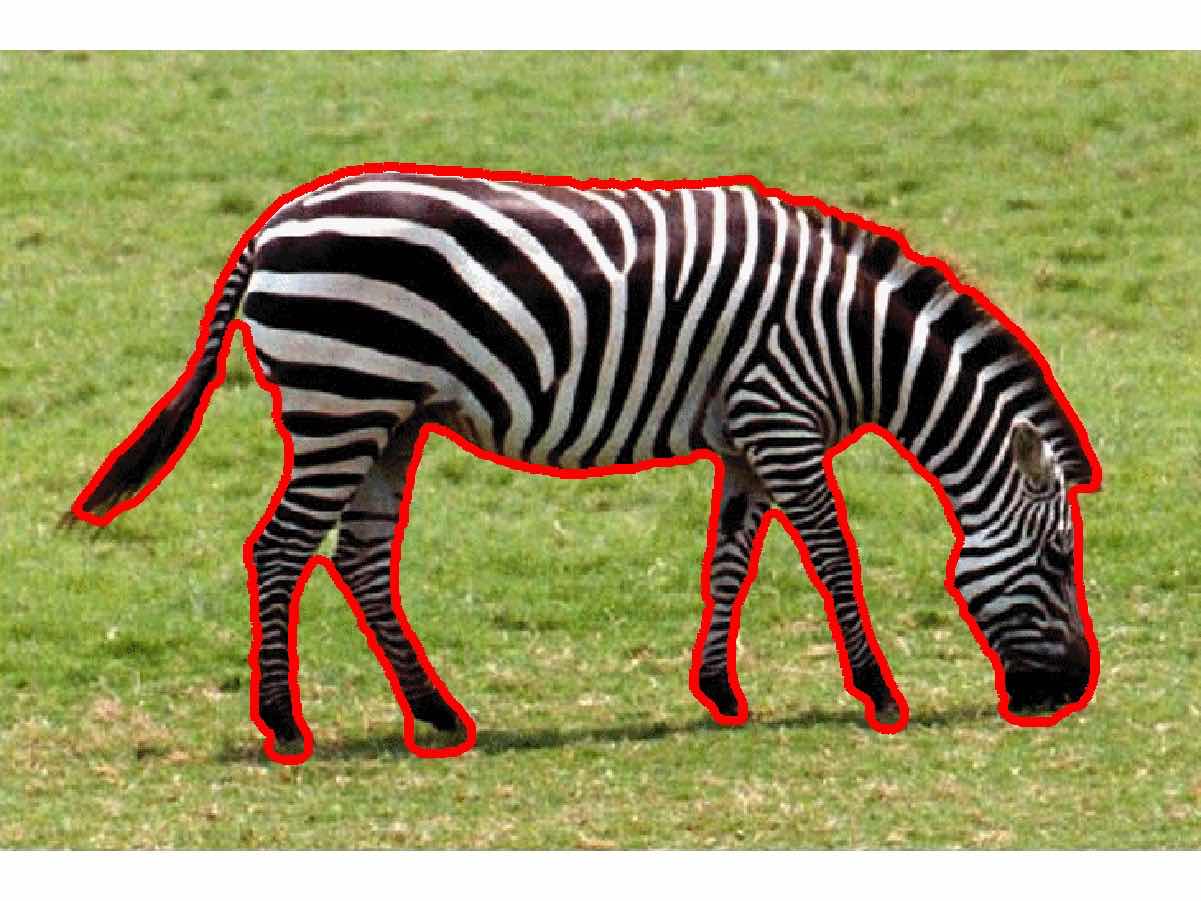} 
\\
\includegraphics[width=0.21\textwidth]{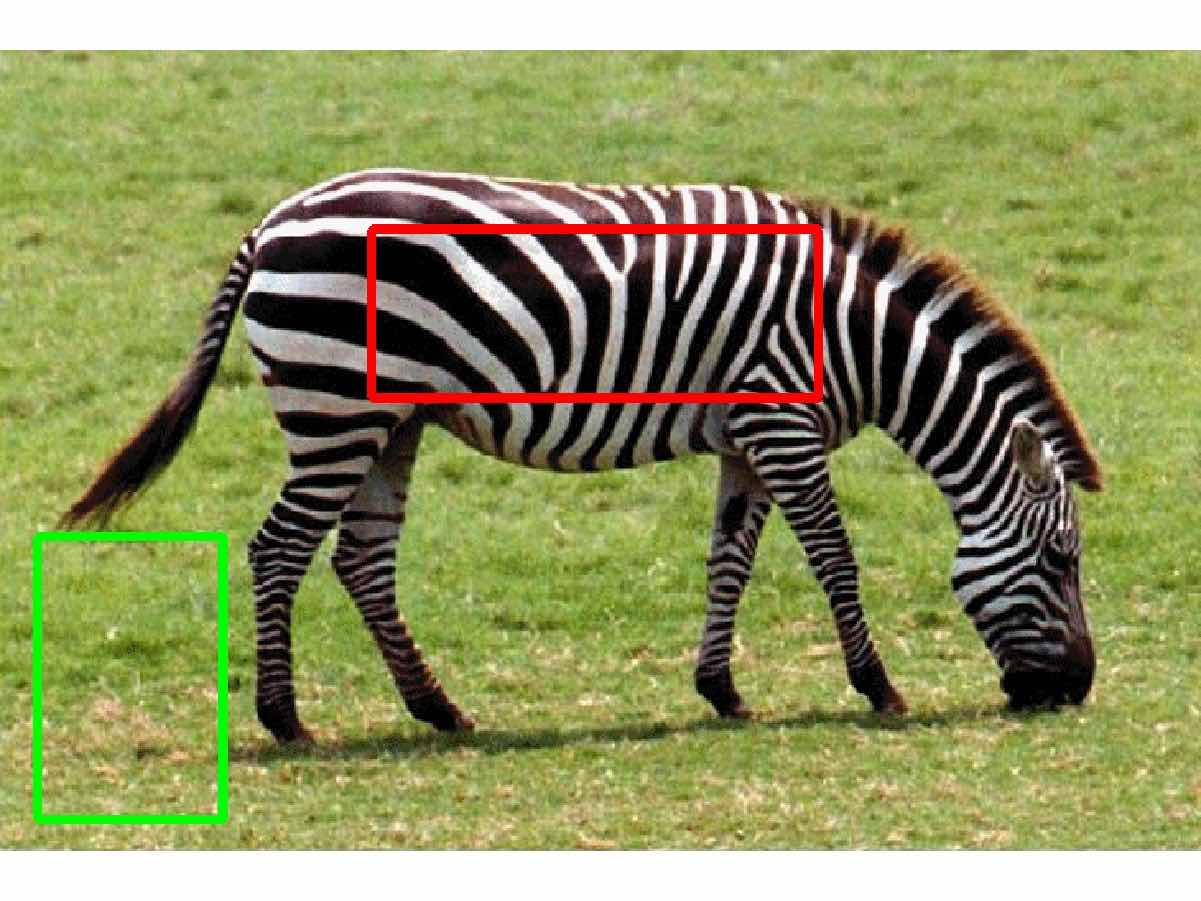} &
\includegraphics[width=0.21\textwidth]{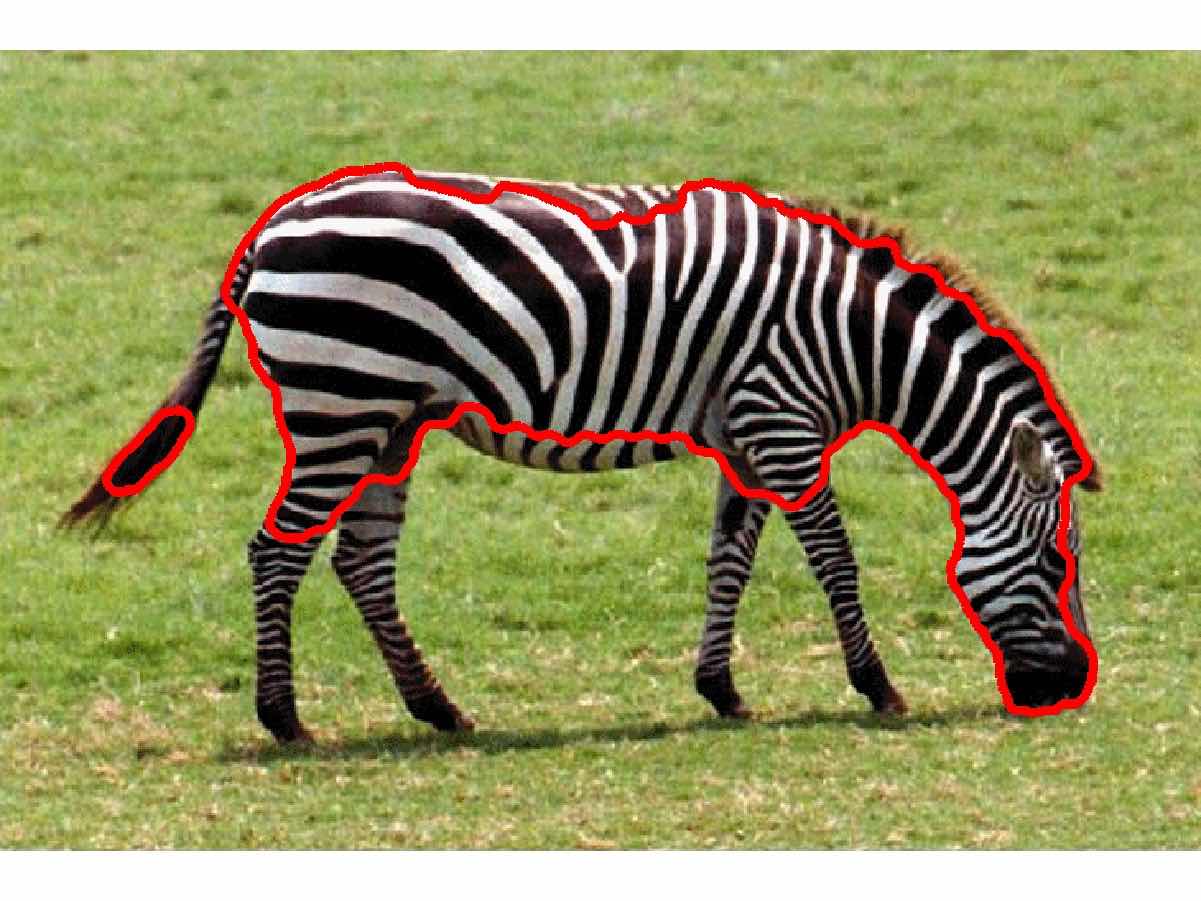} &
\includegraphics[width=0.21\textwidth]{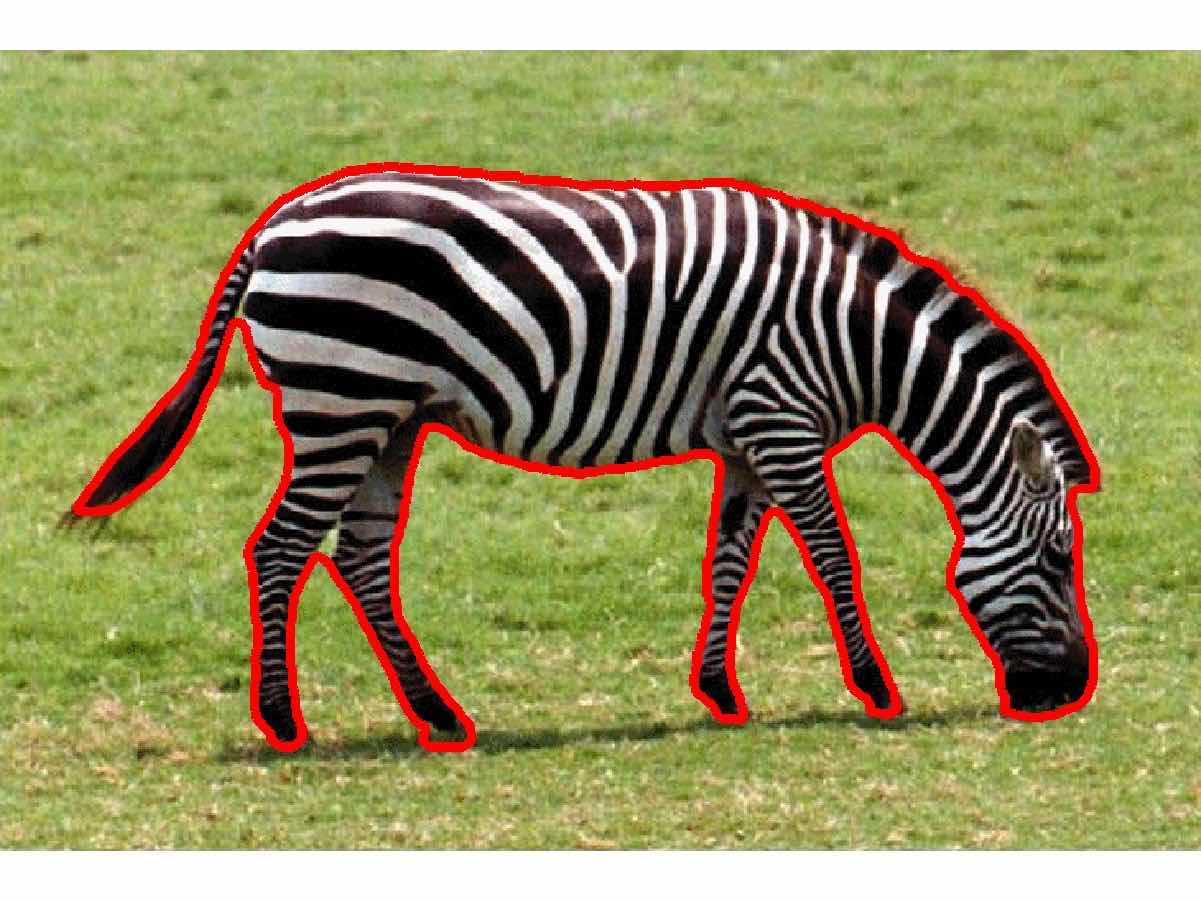} &
\includegraphics[width=0.21\textwidth]{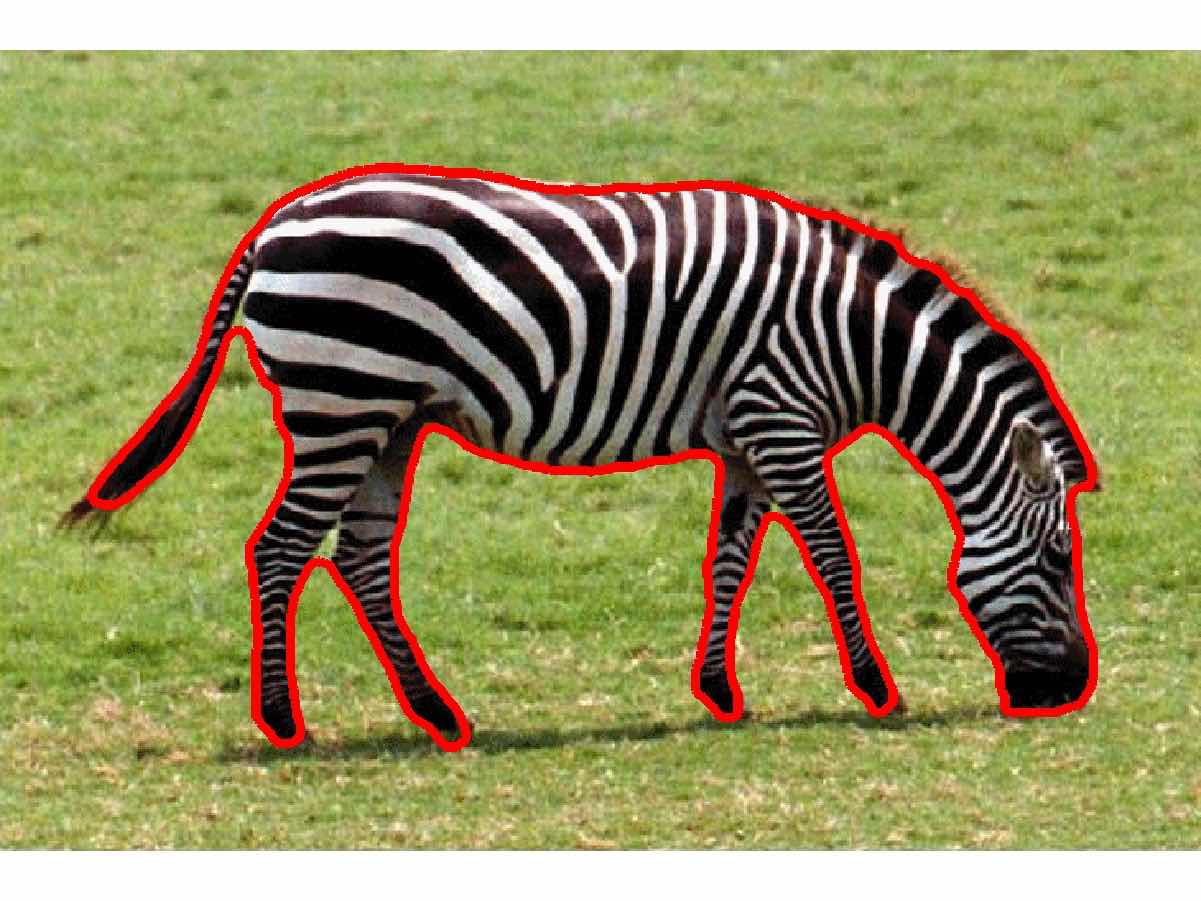} \\
Inputs & $\ell_1$ & $\lbd = \infty$ & $\lbd = 1000$
\end{tabular}

\caption{\textbf{Comparison of the segmentation results obtained from the proposed segmentation models (using $\MK_\lbd$ distances) together with the $\ell_1$ distance used in \cite{papa_aujol}}, for different initialization. The same regularization parameter $\rho$ is used for every segmentations. Note that the optimal transport similarity measure is a more robust statistical metric between histograms than $\ell_1$. %\vspace{0.5cm}
\vspace*{5mm}~
}
\label{fig:zeb}
\end{figure*}

\begin{figure}[!htb]
\begin{center}
\begin{tabular}{cc}
	\hspace{-0.1cm}\includegraphics[width=3.8cm]{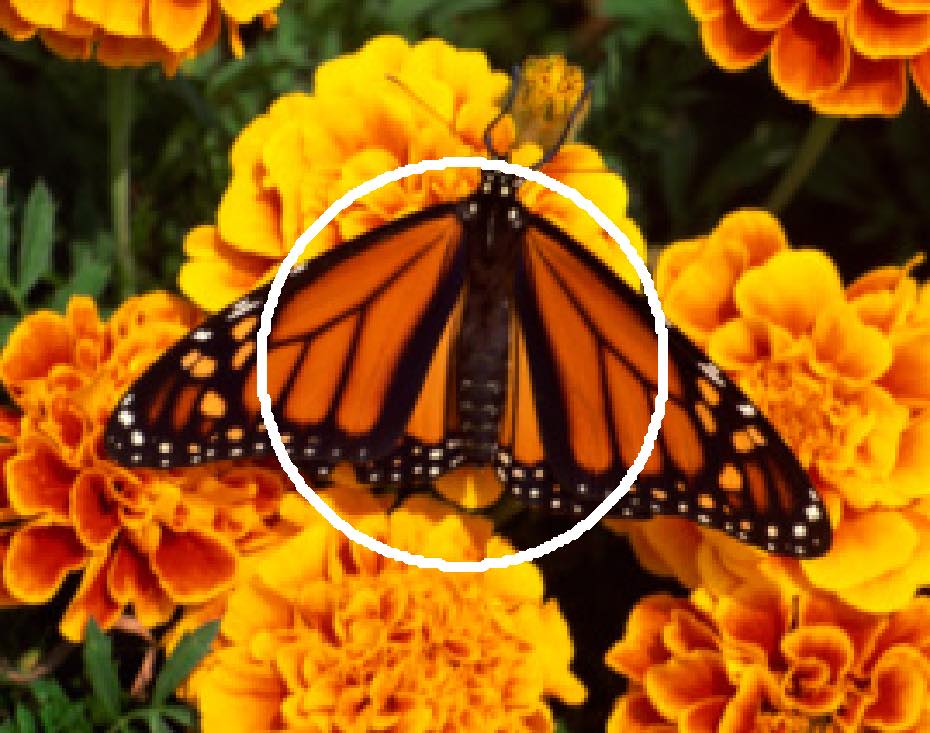}
	&\hspace{-0.2cm}\includegraphics[width=3.8cm]{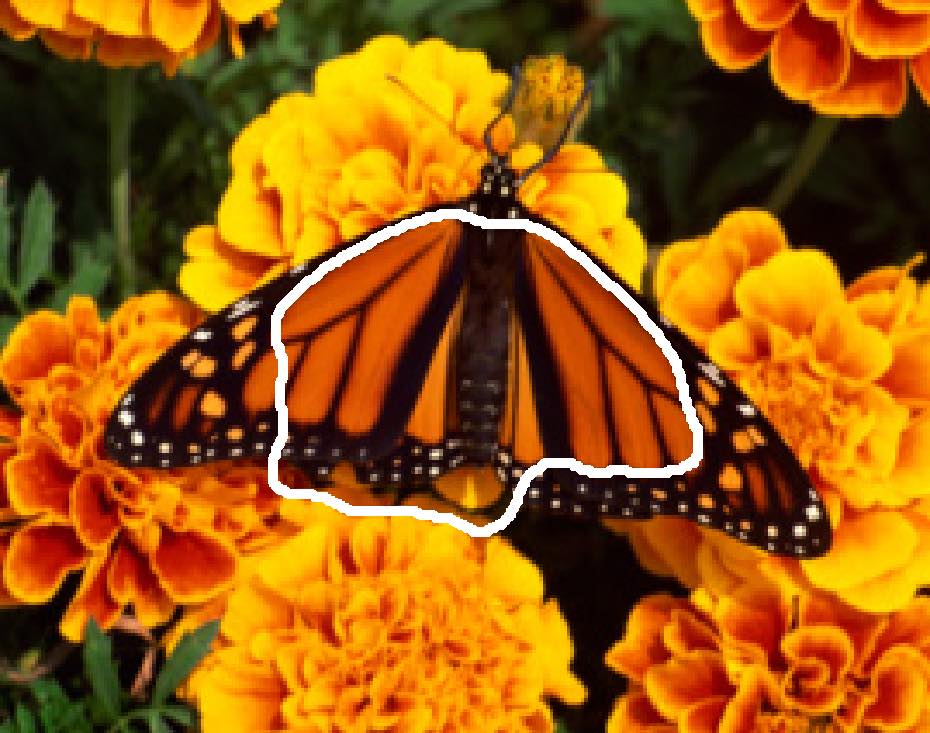}
	\\[2mm]
	\hspace{-0.1cm}\includegraphics[width=3.8cm]{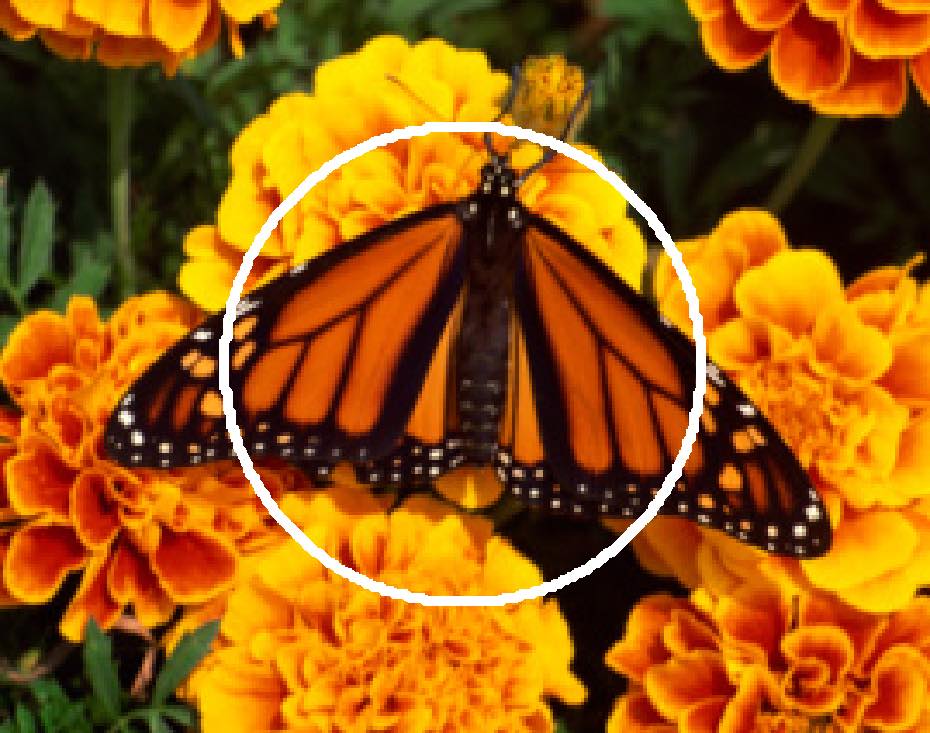}
	&\hspace{-0.2cm}\includegraphics[width=3.8cm]{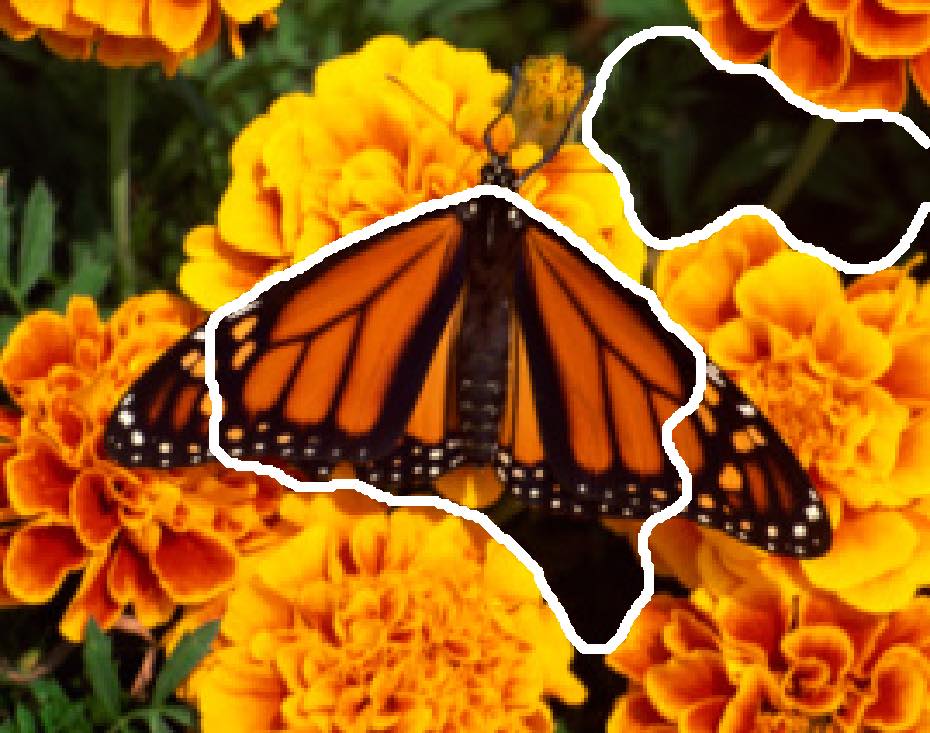}
	\\[2mm]
	\hspace{-0.1cm}(a) Initialization
	&\hspace{-0.2cm}(b) Result
\end{tabular}
\caption{\label{wac} \textbf{Comparison with a non convex model.} The Wasserstein active contours method  \cite{PeyreWassersteinSeg12}, initialized in two different ways in (a), provides the corresponding segmentations presented in (b), illustrating the non-convexity of the model. When carefully parameterized, it leads to a segmentation close to the one obtained with our global approach (see Figure \ref{fig:comp2}(c)). 
\vspace*{5mm}~
}
\end{center}
\end{figure}

%
%
%\begin{figure}[ht!]
%\begin{center}
%\begin{tabular}{cccc}
%	\includegraphics[height=1.9cm]{zeb_data}&\includegraphics[height=1.9cm]{zeb_seg_L1}&
%	\includegraphics[height=1.9cm]{zeb_seg_prox}&\includegraphics[height=1.9cm]{zeb_seg_grad}\\
%	\includegraphics[height=1.9cm]{zeb_data_2}&\includegraphics[height=1.9cm]{zeb_seg_L1_2}&\includegraphics[height=1.9cm]{zeb_seg_prox_2}&\includegraphics[height=1.9cm]{zeb_seg_grad_2}\\
%	$I$, {\color{red} $a_1$}, {\color{green} $a_2$}&$\ell_1$&$\MK$&
%	$\MK_\lbd$
%	\end{tabular}
%\caption{\label{fig:comp1}  Illustration of the robustness of optimal transport distances. First line: accurate description of the  background prior. Second line: other background prior}
%	\end{center}
%\end{figure}

Next, we illustrate the advantage of having a convex model that does not depend on the initialization. We compare our results with the ones obtained with the Wasserstein Active Contour method proposed in \cite{PeyreWassersteinSeg12}. Such approach consists in deforming a level set function in order to minimize globally the Wasserstein distance between the reference histograms and the one induced by the segmentation. To make the level set evolve, this formulation requires complex shape gradients computations.  As illustrated in Figure~\ref{wac}, even if this model can give good segmentations that are close to the ones we obtained in Figure~\ref{fig:comp2}~(c), its initialization may be a critical step as really different segmentations are obtained with very similar initializations.%\vspace{-0.2cm}

We finally show comparisons with the global model of \cite{papa_aujol} that includes $\ell_1$ distances between histograms.
Figure \ref{fig:zeb} first illustrates the robustness  of optimal transport distance with respect to bin-to-bin $\ell_1$ distance. A small variation of the reference histograms may involve a large change in the final segmentation wit $\ell_1$ distance, whereas segmentations obtained with $\MK$ or regularized $\MK_\lbd$ are stable.
\begin{figure}[!htb]
%\begin{center}
\centering
	\begin{tabular}{cc}
	\multicolumn{2}{c}{\includegraphics[width=0.2\textwidth]{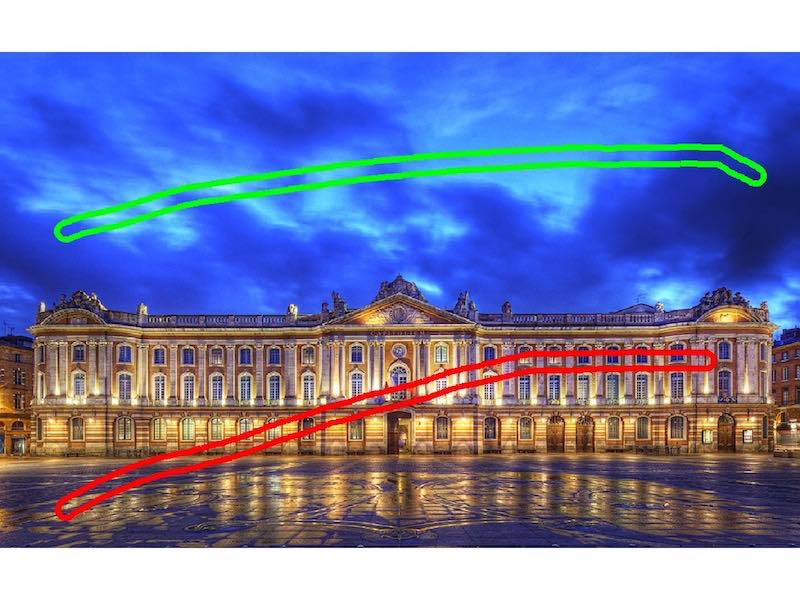}}
	\\
	\multicolumn{2}{c}{Input}
	\\[1mm]
	\includegraphics[width=0.2\textwidth]{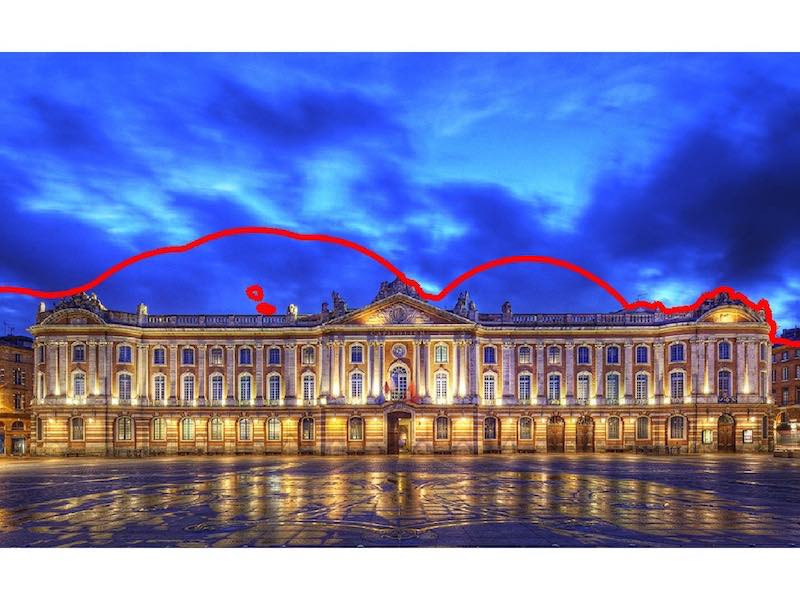}&\includegraphics[width=0.2\textwidth]{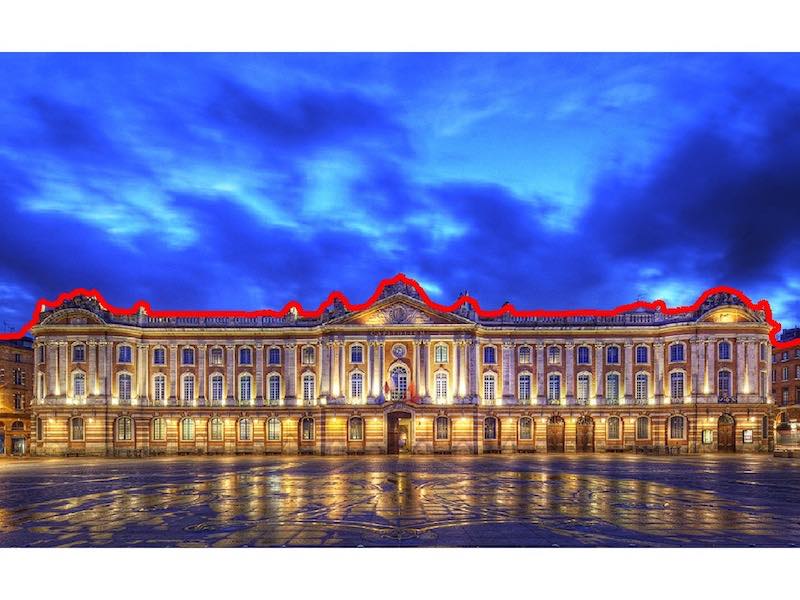}
	\\
	$\ell_1$& $\MK$ %($\lbd = \infty$)
	\end{tabular}
%\end{center}

\caption{\textbf{Robustness of $\MK$ with respect to $\ell_1$.} The blue colors that are not in the reference histograms are considered correctly as background with $\MK$ distance, but as foreground with the $\ell_1$ model where no color comparison is performed. 
\vspace*{5mm}~
}
\label{fig:comp_l1}
\end{figure}

Contrary to optimal transport, when a color is not present in the reference histograms, the $\ell_1$ distance does not take into account the color similarity between different bins, which can lead to incorrect segmentation. This is illustrated with the blue colors in Figure~\ref{fig:comp_l1} where the $\ell_1$ distance leads to an incorrect segmentation by associating some blue tones to the building area. %\vspace{-0.5cm}

\subsection{General results}

The robustness of optimal transport distances is further illustrated in Figure~\ref{fig:zebras}. It is indeed possible to use a prior histogram from a different image, even with a different clustering of the feature space. This is not possible with a bin-to-bin metric, such as $\ell_1$, which requires the same clustering. %\vspace{-0.2cm}
%and can take bad decisions by associating some blue tones to the building area.

%
\begin{figure}[!htb]
\centering
\begin{tabular}{cc}
\includegraphics[width=0.205\textwidth]{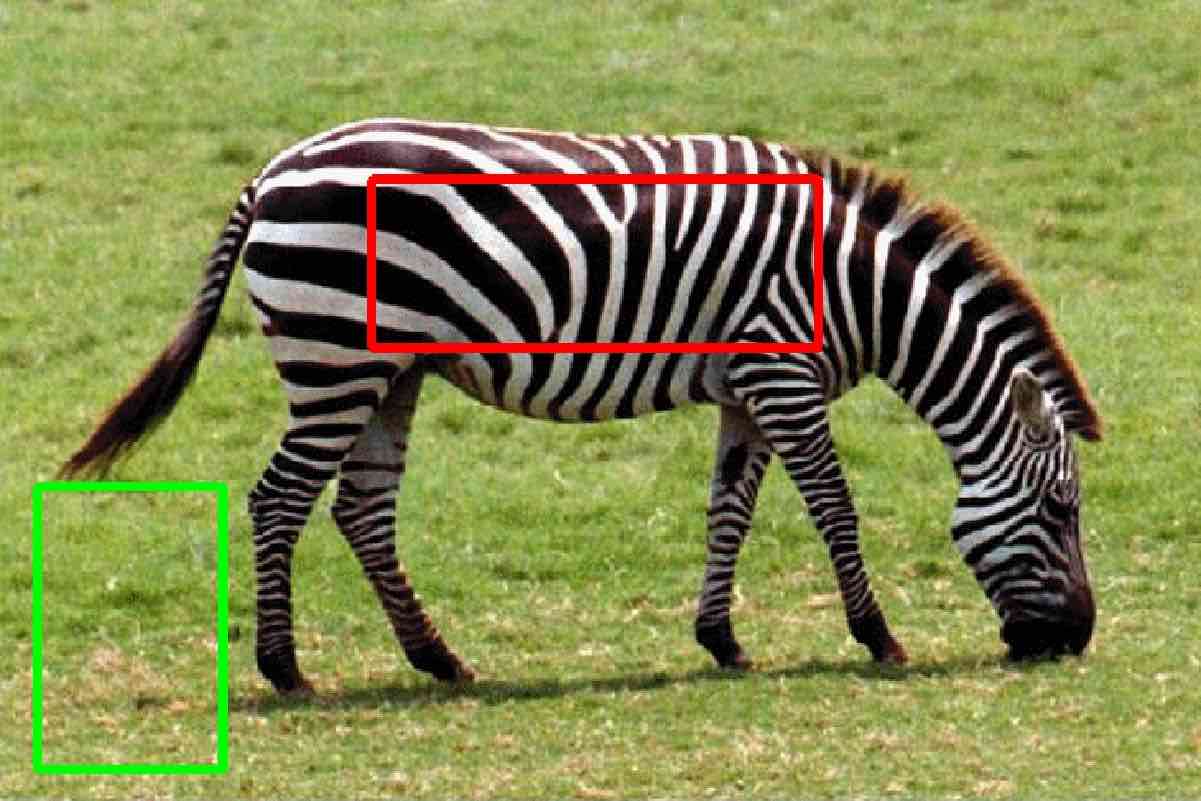}&
\includegraphics[width=0.205\textwidth]{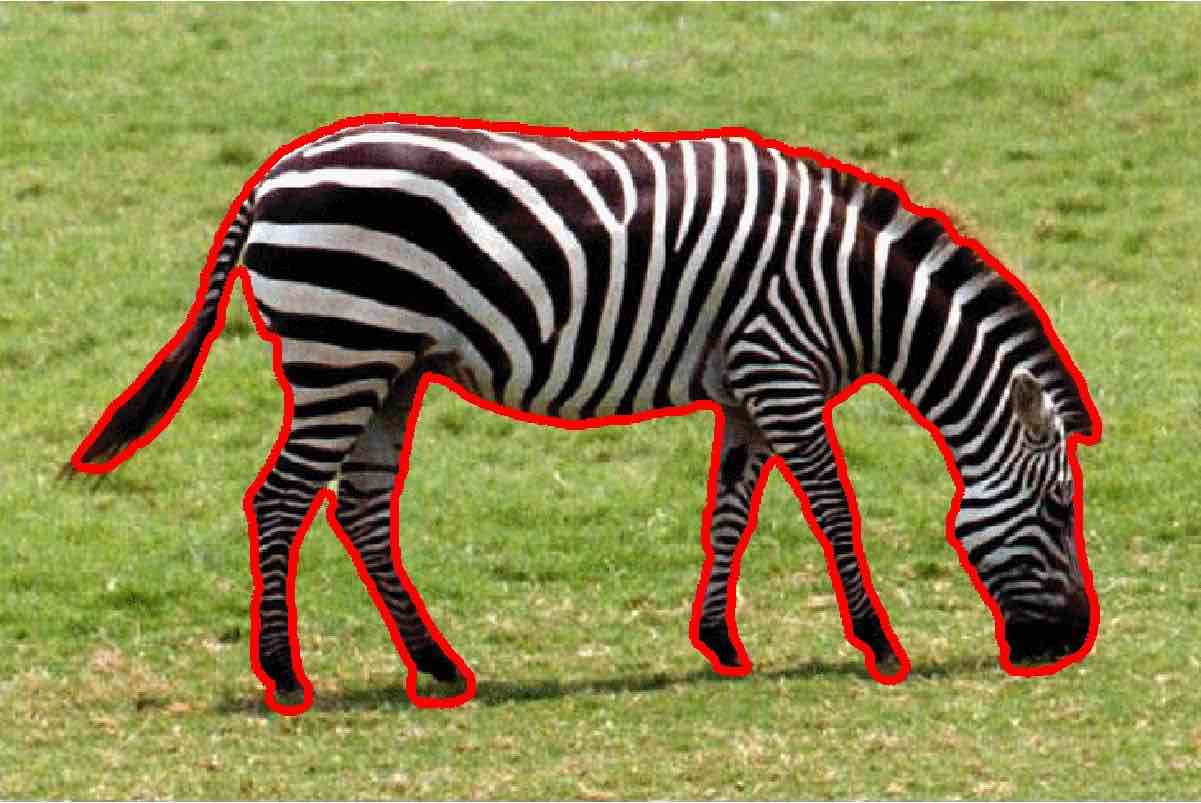}
\\
Input histograms & Segmentation $1$
\\[2mm]
\includegraphics[width=0.205\textwidth]{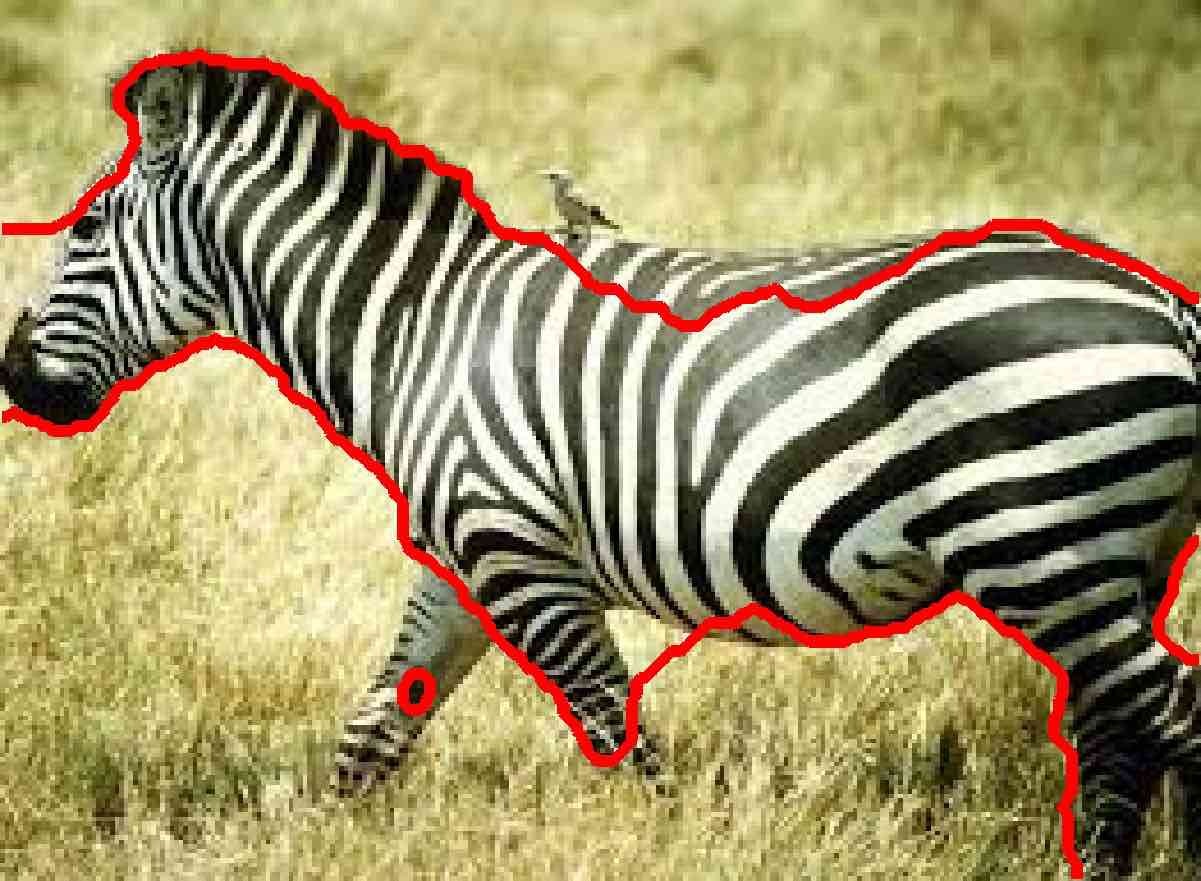}&
\includegraphics[width=0.205\textwidth]{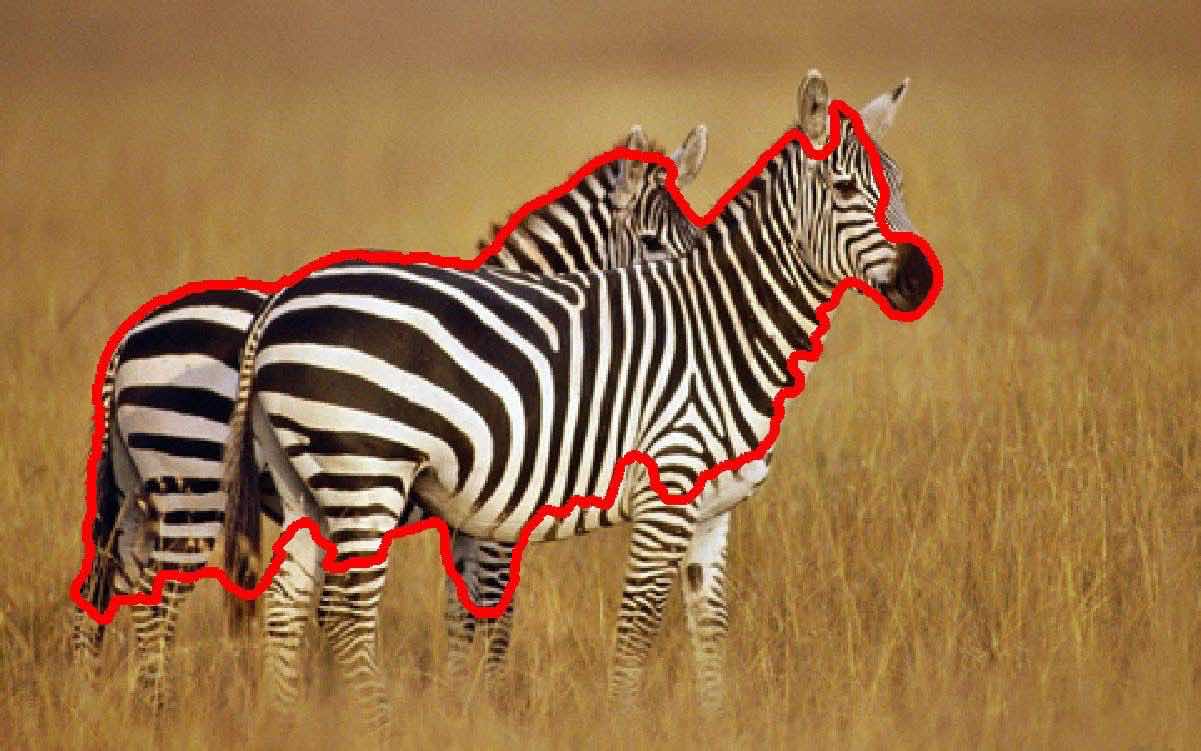}
\\
Segmentation $2$ & Segmentation $3$
\end{tabular}
\caption{\textbf{Illustration of the interest of optimal transport cost for the comparison of histograms}. Its robustness makes it possible to use prior histograms from different images (in this exemple, histograms are estimated from image $1$ and used to segment all images).
%, even with a different clustering of feature space. Note that it is not possible with bin-to-bin metric, which requires the same clustering.
\vspace*{5mm}~
}
\label{fig:zebras}
\end{figure}

Other examples on texture segmentation are presented in Figure~\ref{fig:grad} 
where the proposed method is perfectly able to recover the textured areas.
We considered here the joint histogram of gradient norms on the $3$ color channels. 
The complexity of the algorithm is the same as for color features, as long as we use the same number of clusters to quantize the feature space.

\begin{figure}[htb]
\begin{center}
	\begin{tabular}{cc}
	\hspace{-0.15cm}\includegraphics[width=3.8cm]{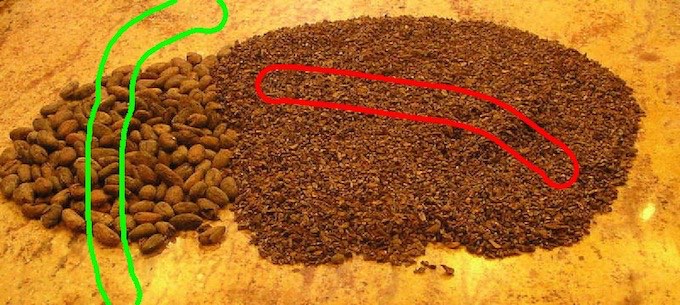}&\hspace{-0.3cm}  
	\includegraphics[width=3.8cm]{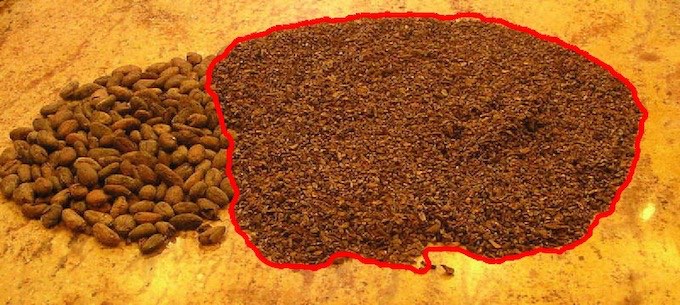}%\vspace{0.12cm}
	\\
	\hspace{-0.15cm}\includegraphics[width=3.8cm]{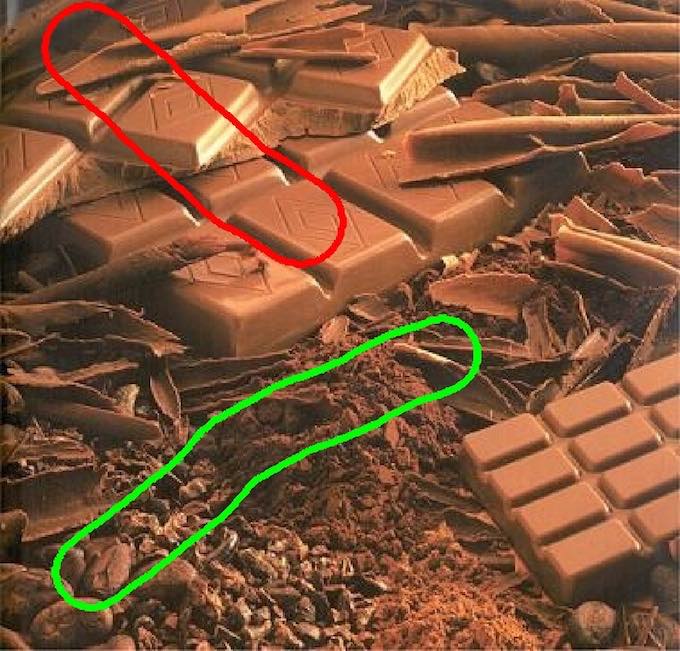}&\hspace{-0.3cm} 
	\includegraphics[width=3.8cm]{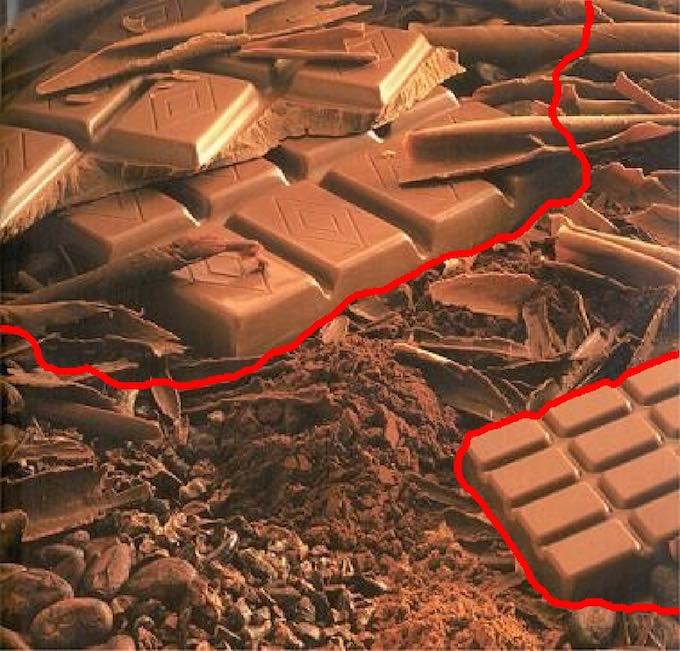} 
	\\
	\hspace{-0.15cm}Input &\hspace{-0.3cm}   $\MK$ % $\lbd = \infty$ 
	\end{tabular}
\end{center}
\caption{\textbf{Texture segmentation using joint histograms of color gradient norms.} In this exemple, only gradient information is taken into account, illustrating the versatility of the optimal transport framework.
\vspace*{5mm}~
}
\label{fig:grad}
\end{figure}

We finally present experiments involving more than two partitions in Figure \ref{fig:multiphase}.
In the first line, three regions are considered for the background and the two parrots. Even if the two parrots share similar colors, the model is able to find a global segmentation of the image in three regions.
In the second line of Figure \ref{fig:multiphase}, we considered $4$ regions for the sky, the grass, the forest and the plane.
The approach is able to deal with the color variations inside each class in order to perform a correct segmentation.

\begin{figure}[!h]
\begin{center}
	\begin{tabular}{cc}
	\hspace{-0.15cm}\includegraphics[width=3.8cm]{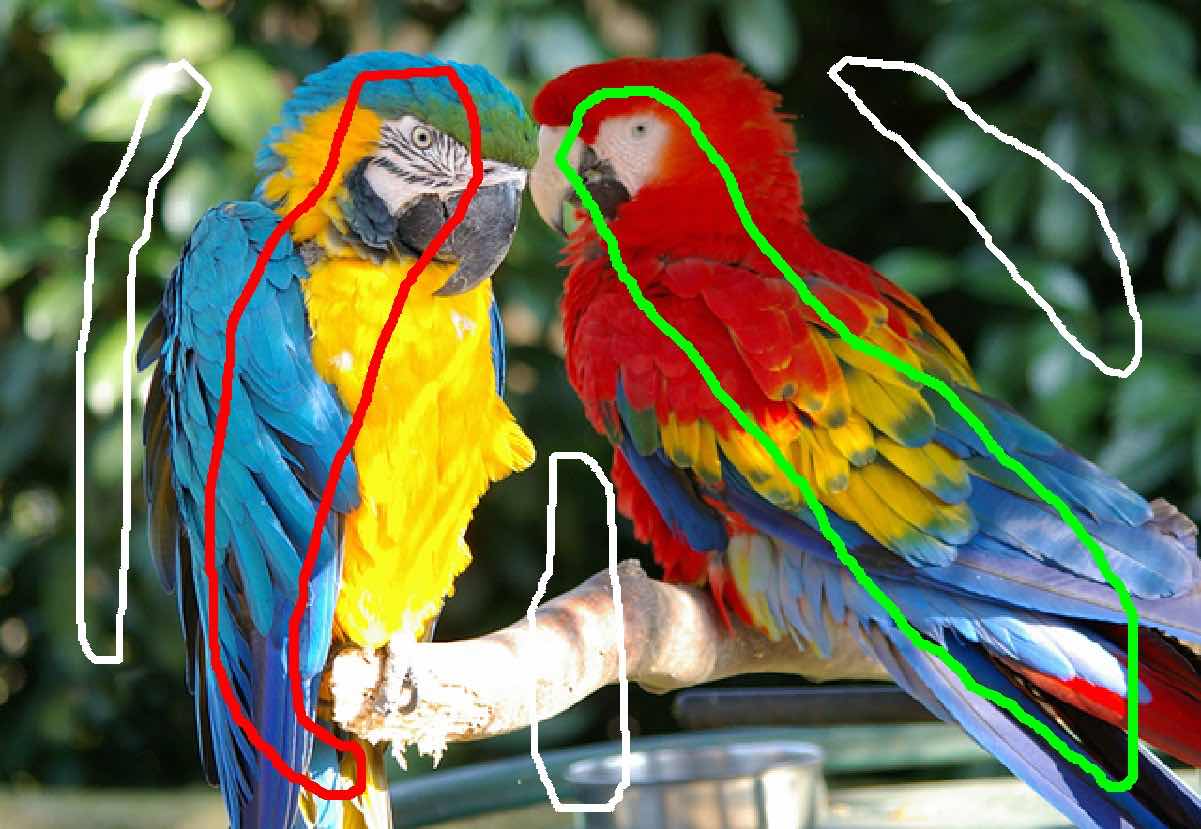}&\hspace{-0.3cm} 
	\includegraphics[width=3.8cm]{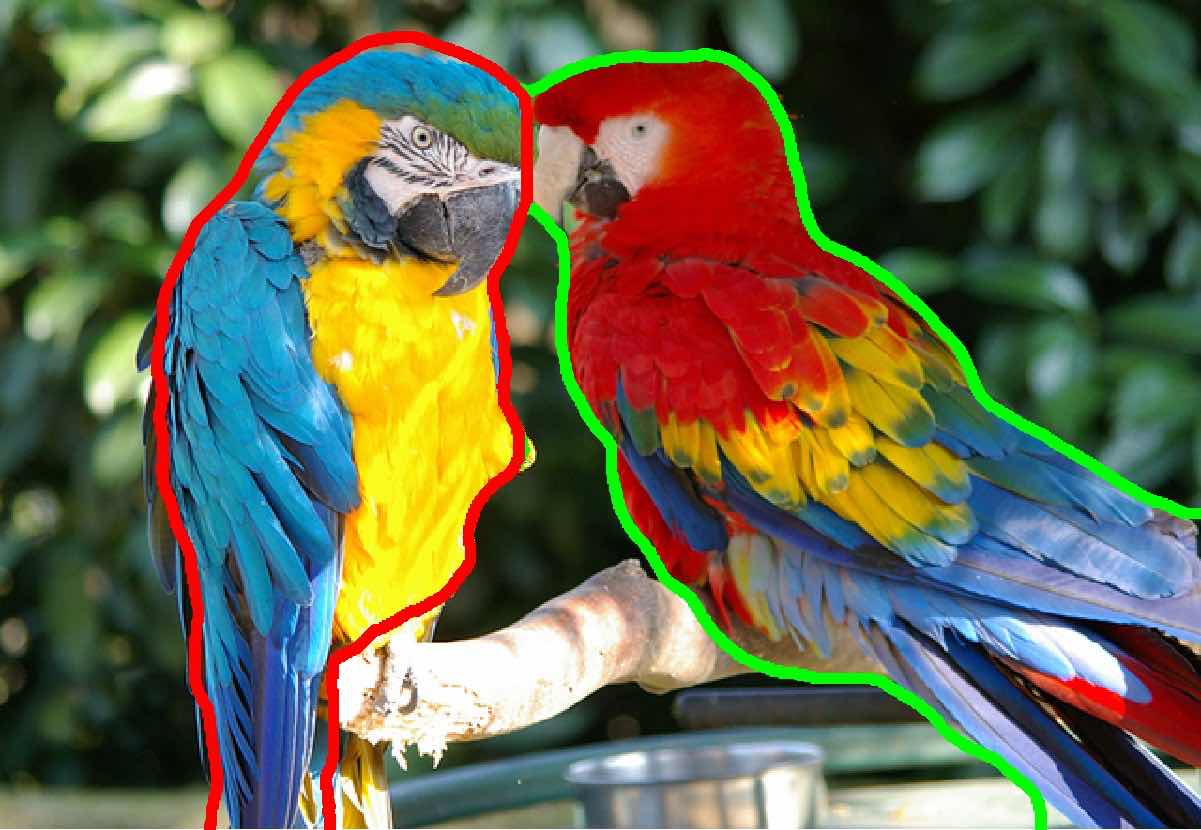} %\vspace{0.12cm}
	\\[1mm]
	\hspace{-0.15cm}\includegraphics[width=3.8cm]{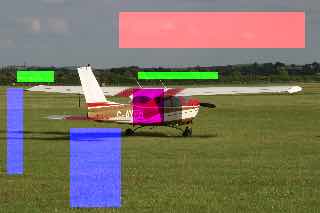}&\hspace{-0.3cm}
	\includegraphics[width=3.8cm]{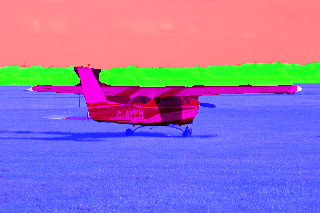} 
	\\
	\hspace{-0.15cm}Input &\hspace{-0.3cm} Segmentation
	\end{tabular}
\end{center}

\caption{\textbf{Multi-phase segmentation} with $3$ regions (first line) and $4$ regions (second line). 
\vspace*{5mm}~
}
\label{fig:multiphase}
\end{figure}

% - o - o - o - o - o - o - o - o - o - o - o - o - o - o - o - o - o - o - o - o - o - o - o - o - o - o - o - o - o - o - o - o - o - o - o - o 
% o - o - o - o - o - o - o - o - o - o - o - o - o - o - o - o - o - o - o - o - o - o - o - o - o - o - o - o - o - o - o - o - o - o - o - o - o
%\bigskip
%\newpage

\section{Unsupervised Co-segmentation}\label{sec:coseg}

In this section, we extend our framework to the unsupervised co-segmentation of multiple images.
We invite the reader to see the following reference~\cite{Vicente_co-seg_eccv10} for a complete review.

\subsection{Co-segmentation of $2$ images}

We first consider two images $I^1$ and $I^2$ the domain of which is respectively $\Om_1$ and $\Om_2$ composed of $N_1$ and $N_2$ pixels.
Assuming that the images contain a common object, the goal is now to jointly segment them without any additional prior. 

\paragraph{Model for two images}

To that end, following the model used in~\cite{Vicente_co-seg_eccv10,Swoboda13}, we aim at finding the \emph{largest} regions that have similar feature distributions. To define the segmentation maps $u^1$ and $u^2$ related to each image, we consider the following model first investigated in~\cite{coseg_hist_matching,coseg_hist_matching_L2}, %,Vicente_co-seg_eccv10
denoting $u = (u^1;u^2)$:
\eql{\label{eq:coseg_model0}
	\hspace*{-1mm}
	J'(u) := \Dist \left( H_1 u^1, H_2u^2 \right)
		    +  \sum_{k=1}^2  \rho \TV(u^k) - \delta {\norm{u^k}}_1 %  \left(\right) % \dotp{\U_{N_k}-u^k}{\U_{N_k}}
}
where, for a non-negative variable $u^k$ we have a total mass ${\norm{u^k}}_1= \dotp{u^k}{\U_{N_k}}$. When $u^k \in \{0,1\}^{N_k}$, this term corresponds to the area of the region segmented in image $I^k$. Such a ballooning term encourages the segmentation of large regions.
Without this term, a global optimum would be given by $u^k=\O$.
%\cmt{j'ai changé $\MK$ par $\Dist$ pour ensuite évoquer toutes les possibilités (j'ai seulement codé $\MK_\lbd$) en disant que c'est seuleùment pour S=MK que l'on a besoin de la definition de simplexe non normalisé !}

Following definition \eqref{op:H}, the operator $H_k(i,x)$ is  $1$ if pixel $I^k(x)$ belongs to the cluster $\C_{\X_k}(i)$ and $0$ otherwise.
As before, the value of the segmentation variables $u^k$ are relaxed into the convex intervals $[0,1]^{N_k}$.

In \cite{Vicente_co-seg_eccv10}, several cost functions $\Dist$ are benchmarked for the model defined in Eq.~\eqref{eq:coseg_model0}, such as $\ell_1$ and $\ell_2$. It is demonstrated that $\ell_1$ performs the best.
In~\cite{Swoboda13}, Wasserstein distance %(Section~\ref{sec:wasserstein}) transformé en paragraphe
is used again to measure the similarity of the two histograms.
In the following, we investigate the use of these two metrics in our setting.

\paragraph{Property of the segmented regions}
To begin with, note that when considering optimal transport cost to define $\Dist$,
one has to constraint the histograms to have the same mass, \emph{i.e.} $(H_1 u^1,H_2 u^2) \in \S$.
%Finally, in order to have a well defined distance between histograms, 
%
When using assignment operators such as in~\eqref{op:H}, this boils down to constraint the segmentation variables to have the same mass, \emph{i.e.} $\dotp{u^1}{\U_{N_1}} = \dotp{u^2}{\U_{N_2}}$.

%	
%\emph{i.e.} $(u^1,u^2) \in \HS$ where
%\eql{\label{eq:def_domaine_coseg}
%	%\begin{cases}
%	\HS	\hspace{-0.05cm} := \hspace{-0.05cm} 
%	\left\{ 
%		%(x_1,x_2)\hspace{-0.05cm} \in\hspace{-0.05cm} \R^{N_1}_+\hspace{-0.05cm} \times \hspace{-0.05cm} \R^{N_2}_+, 
%		x_1\hspace{-0.05cm} \in\hspace{-0.05cm} \R^{N_1}_+, x_2\hspace{-0.05cm} \in\hspace{-0.05cm} \R^{N_2}_+, 
%		\dotp{x_1}{\U_{N_1}}\hspace{-0.05cm}= \hspace{-0.05cm}\dotp{x_2}{\U_{N_2}} 
%	\right\}%
%	%\\
%	%T_2:= \left\{ {(x,y) \in [0,1]^{N_1} \times  [0,1]^{N_2}}\right\}
%	%\end{cases}.
%}
%\cmt{using a slightly different definition than the one previously introduced for $\S$ in \eqref{eq:admissible_histograms} to take into account the different of size in segmentation variables}.

When looking for a binary solution, this condition implies that the two regions corresponding to the segmentation of each image have the exact same number of pixels.
This means that the model is not robust to small scale change in appearance with optimal cost transport, while is it the case when using the $\ell_1$ metric, as demonstrated in~\cite{Vicente_co-seg_eccv10}.
In practice, while this property does not necessarily hold for solutions of the relaxed problem that are binarized by thresholding (see Eq.~\eqref{def:region}), this limitation has been also observed.

One simple way to remove such restriction from the model is to use the same formulation introduced in Section~\ref{sec:convexification_histogram} for segmentation.
Unfortunately this boils down to define the similarity measure with 
$$S(H_1 u_1 \dotp{u_2}{\U}, H_2 u_2 \dotp{u_1}{\U})$$ 
which is obviously non-convex and does not fit the optimization framework used in this paper.
% its optimization is beyond the scope of this paper.

It is not the first time that the conservation of mass in the optimal transport framework is reported to limit
its practical interest in imaging problem, and several variations have been proposed to circumvent it.
Without entering into details, a common idea is to discard the conservation of mass when the two histograms are \emph{unbalanced} and to define alternative transport maps that may create or annihilate mass.
As an exemple, a solution might be to transport the minimum amount of mass between the unnormalized histograms and penalize the remaining, as done by the distance introduced in \cite{Pele-ICCV} and similarly in \cite{gramfort_2015}.
Other models has been recently investigated, such as in 
\cite{lombardi_maitre_2013}, and
\cite{Vialard15b,Frogner}.
However, the application of such metric for our setting is far from being straightforward and need careful analysis that is left for future work.
%(mainly due to the absence of Lipschitz property) and beyond the scope of the paper}.

\paragraph{Optimization}
%
%To optimize this problem, we can still consider the convex conjugate of the function 
%$\MK \left( H_1 u^1, H_2u^2 \right) +\chi_\HS(u^1,u^2)$ 
%$\MK \left( H_1 u^1, H_2 u^2 \right) + \chi_{ \dotp{u^1}{\U_{N_1}} = \dotp{u^2}{\U_{N_2}} }$ % FAUX vu la definition de MK
%(see relation \eqref{eq:primal_dual_MK} or expression \eqref{eq:dual_sinkhorn_nonnormalized} in case of entropic regularization of the $\MK$ distance).
To solve the relaxed problem
$$\min_{u \in [0,1]^{N_1 + N_2}} J'(u)$$
using either $\ell_1$, $\MK$ or $\MK_{\lambda}$ as a cost function $\Dist$,
we rely again on the primal-dual formulation~\eqref{pb:primaldual} of the problem and the algorithm~\eqref{algo:primaldual}.
Notice that the minor difference with previous segmentation problems is the presence of the linear ballooning term and that there is only one dissimilarity term.

\paragraph{Experiments} 

We now illustrate the behavior of this mo\-del.
Again, we underline that the convex co\-seg\-men\-tation model \eqref{eq:coseg_model0} is not new, as our approach only differ algorithmically from~\cite{Vicente_co-seg_eccv10,Swoboda13} when using $\ell_1$ or $\MK$ as a cost function. Therefore, we only focus on results obtained using optimal transport with entropic regularization (setting $\lbd = 100$).

In the synthetic experiment of Figure~\ref{fig:coseg} containing exactly the same object with different backgrounds, we compare our approach with the one of \cite{Swoboda13}, that does not include entropic regularization\footnote{Another main difference is that \cite{Swoboda13} makes use of super-pixel representation to reduce the complexity, whereas we use a pixel representation.}. 
Both methods gives similar co-segmentations.

\begin{figure}[bth]
    \begin{center}
    \begin{tabular}{cc}
    \includegraphics[width = 0.42\linewidth]{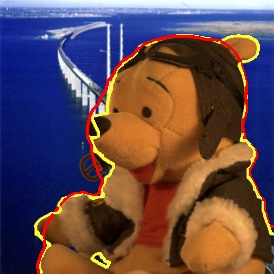}&\hspace{-0.4cm}
    \includegraphics[width = 0.42\linewidth]{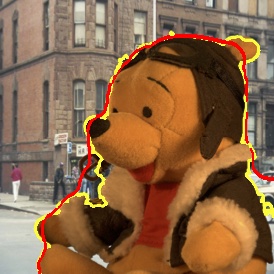}%
\\
    \includegraphics[width = 0.42\linewidth]{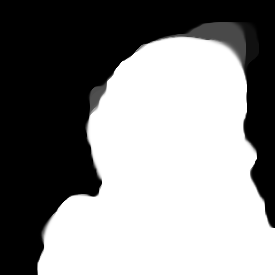}&\hspace{-0.4cm}
    \includegraphics[width = 0.42\linewidth]{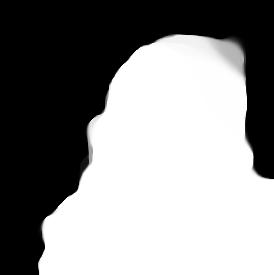}%
    \end{tabular}
    \end{center}
    %\vspace*{-3mm}
    \caption{\label{fig:coseg}  %Exemple de co-segmentation non supervisée
    \textbf{Co-segmentation and optimal transport with or without entropic regularization}. The results obtained with the model \eqref{eq:coseg_model0} with the entropic regularization (in red), that approximate the method of  \cite{Swoboda13} (in yellow, {image courtesy of \cite{Swoboda13}}).
    The estimated segmentation maps  ${u^k}$ are binary almost everywhere. The threshold $t = \frac{1}{2}$ is used to obtain the final co-segmentation regions.
    \vspace*{5mm}~
    }
\end{figure}

When considering images where the common object has a similar scale in both images, Figure~\ref{fig:coseg2} shows that the condition $\norm{u^1}_1=\norm{u^2}_1$ is not restrictive and our method still gives acceptable co-segmentations.

\begin{figure}[!htb]
\begin{center}
\begin{tabular}{cc}
\hspace{-0.cm}\includegraphics[height=2.45cm]{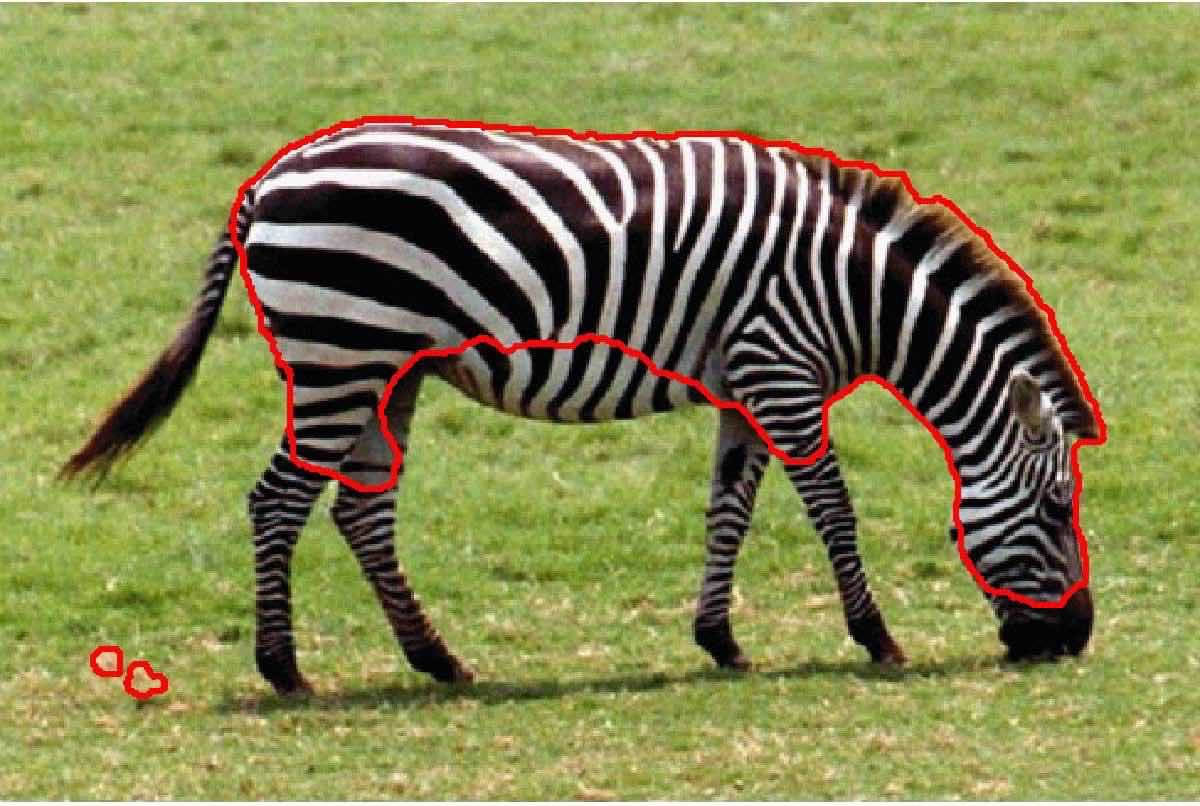}   &\hspace{-0.2cm}
\hspace{-0.cm}\includegraphics[height=2.45cm]{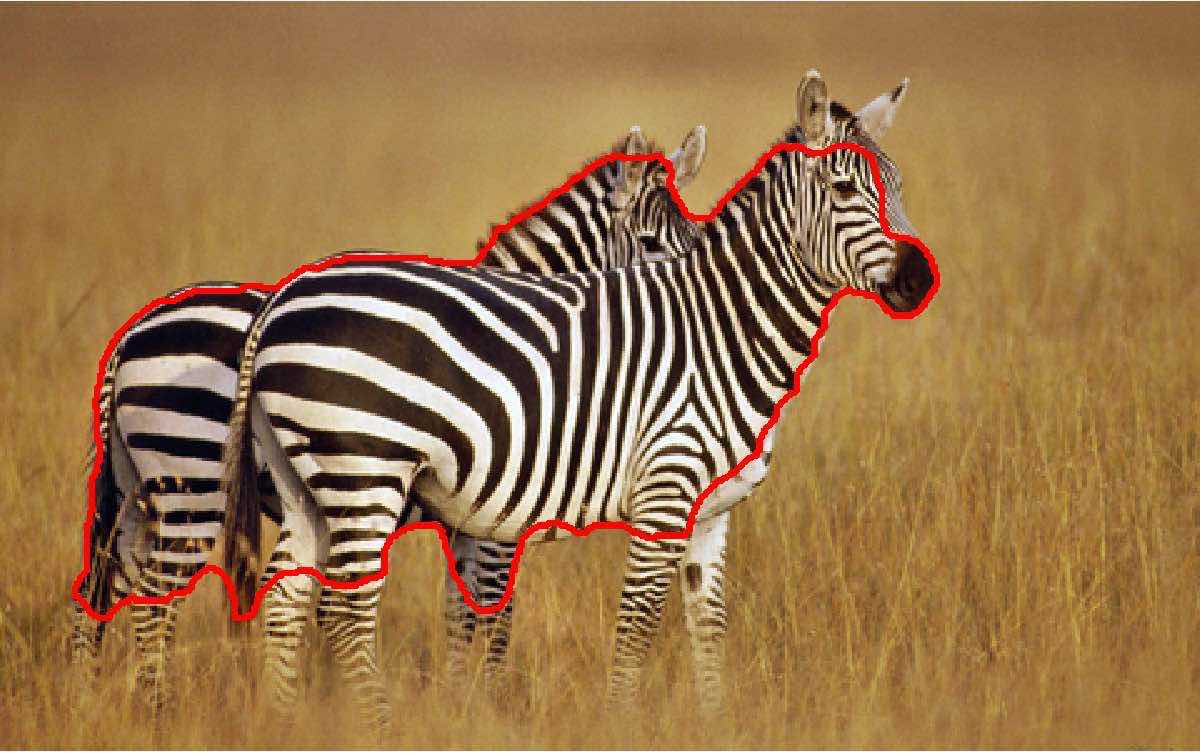}   \\
$u^1$&\hspace{-0.2cm}$u^2$
\end{tabular}
\end{center}
\caption{\label{fig:coseg2}  %Exemple de co-segmentation non supervisée
Co-segmentation of two zebras with the model \eqref{eq:coseg_model0}. The convex constraint $\norm{u^1}_1=\norm{u^2}_1$ enforces   the segmented regions to have the same area. As the obtained result is not binary, the areas may be different after the thresholding.
\vspace*{5mm}~
}
\end{figure}

Nevertheless, in a more general setting, we cannot expect the common objects to have the same scale in all images.
We leave the study of alternative optimal transport based distance such as \cite{Vialard15b,Frogner} for future work.

%\subsection{Non-convex model} : on a supprimé ce model pour faire un papier exclivement convexe

%
%Pour le cas de $P$ images, on peut alors definir:
%\eql{\label{eq:coseg_model1}
%J(u) := \sum_{k=1}^P \left( \sum_{l=k+1}^P\MK \left( H_k u^k, H_l u^l,\U\rangle  \right) + 
% \rho \TV(u^k) - \delta {\norm{u^k}}_0+\chi_{[0;1]}(u^k)\right)+\chi_{\Delta_P}(u) % \dotp{\U_{N_k}-u^k}{\U_{N_k}}
%}

%\subsection{Unconstrained areas for co-segmentation of $P$ images} ancien titre
\subsection{Co-segmentation of $P$ images}

We consider now the generalization of the previous co-segmentation model to an arbitrary number of $P\ge  2$ images.

\paragraph{Complexity}

A natural extension of \eqref{eq:coseg_model0} for more than two images would be to penalize the average dissimilarity between all image pairs, writing for instance
 \eql{\label{eq:coseg_model2}
    \begin{split}
    J''(u) &	= \sum_{k=1}^{P-1}  \left(\sum_{l=k+1}^{P} S(H_k u^k,H_l u^l) \right) % {\norm{H_k u^k-b}}_1     %  \left( \right.    \left.\right)
     			+ \rho \TV(u^k) \\ 
		& \quad 	 - \delta {\norm{u^k}}_1+ \chi_{[0,1]^{N_k}}(u^k) %+ \chi_{\PS_{1,M_b}}(b)
\end{split}}
which would require to compute $\tiny\begin{pmatrix} P\\2 \end{pmatrix}$ similarity terms $S(H_k u^k,H_l u^l)$.
However, the complexity of such a model scales quadratically with the number of images $P$ which is not desirable.

To that end, we consider instead the following barycentric formulation which scales linearly with $P$
\eql{\label{eq:coseg_model1}
\begin{split}
    J'''(u,b) &	= \sum_{k=1}^P  S(H_k u^k,b)  % {\norm{H_k u^k-b}}_1     %  \left( \right.    \left.\right)
     			+ \rho \TV(u^k) \\ 
		& \quad - \delta {\norm{u^k}}_1+ \chi_{[0,1]^{N_k}}(u^k) + \chi_{\ge \O}(b) %+ \chi_{\PS_{1,M_b}}(b)
\end{split}
}
where $b$ is the estimated barycentric distribution between the histograms of the segmented regions in all images.
Note that in the \emph{unsupervised} case studied in this section, the barycenter $b$ has to be estimated jointly with segmentation variables.

\paragraph{Model properties with optimal transport costs} 

Combining the $O(P^2)$ model~\eqref{eq:coseg_model2} or the linear barycentric formulation~\eqref{eq:coseg_model1} with optimal transport based cost functions $\MK$ and $\MK_\lbd$ results in the scaling problem previously reported with model~\eqref{eq:coseg_model0}, as definitions of $\MK$ \eqref{eq:OT_indicators} and $\MK_\lbd$ \eqref{eq:sinkhorn_distance_nonnormalized} constraint each pair of histograms to have the same mass. 
%
%solution with unbalanced histogram :
Non-convex formulation or unbalanced transport costs~\cite{Frogner,Vialard15b} should again be considered as a solution, but do not fit in the proposed optimization framework.
%but they are from now too complex to optimize in the proposed framework and not able to give accurate information between very different histograms.

\paragraph{Model properties with $\ell_1$} 

In order to circumvent this issue, we consider the $\ell_1$ case instead, \emph{i.e.} using 
\eq{S(H_k u^k,b)  = {\norm{H_k u^k-b}}_1.}
As stated before in paragraph~\ref{sec:wasserstein_vs_l1}, %\cmt{reférer au paragraphe ou on dit que $\ell_1$ est un cas particulier}
the $\ell_1$ distance between normalized histograms can be seen as the total variation distance, a specific instance of the $\MK$ distance that naturally extend to unnormalized histograms.
In this setting, recall that histograms must have the same number of bins $M$ and the exact same feature clusters $\C_\X$
(see Section~\ref{sec:def_feature_histogram}).

%\cmt{What happen when $b=\O$ : the ballonning term should prevent that ?!}
%Another advantage is that we don't need to constraint the mass of the barycenter

\paragraph{Optimization}

The minimization of the functional~\eqref{eq:coseg_model1} for fixed histograms $H_k u^k$ boils down to the smooth Wasserstein barycenter problem studied in \cite{CuturiPeyre_smoothdualwasserstein}.
The authors investigate the dual formulation this primal problem and show that it can be solved by a projected gradient descent on the dual problem. 
They resort to a splitting strategy, defining $P$ primal histogram variables $(b_k \in \R^M)_{k=1..P}$ with the linear constraint $b_1=\ldots=b_P.$
%\eql{\label{eq:linear_constraint_barycenter_splitting}
%	%(\B):\quad 
%	b_1=\ldots=b_P.
%}
%The dual constraint reads $\sum_k b_k = \O$ ... See \cite{CuturiPeyre_smoothdualwasserstein} for details.
Using a similar approach, one obtain the following primal-dual formulation 
\begin{equation*}
\begin{split}
\hspace*{0mm}
%\min_{\substack{u=\left(u^k  \in [0,1]^{N_k}\right)_{k=1}^P\\b= \left(b_k  \in [0,1]^{M}\right)_{k=1}^P}} %\; 
%\max_{\substack{h=\left(h_k\right)_{k=1}^P\\v=\left(v_k\right)_{k=1}^P}} \;
%\min_{\substack{u\\[1mm]b}}\;
%\max_{\substack{h\\[0.5mm]v}}\;
%	& \sum_{k=1}^P \dotp{H_k u^k - b_k}{h_k}  - \delta \dotp{u^k}{\U_{N_k}} +  \dotp{D_k u^k}{v_k} 
%	\\
%	&  \qquad -  \chi_{\norm{h_k}_{\infty}\le 1}- \chi_{\norm{v_k}_{\infty,2}\le \rho}
%	\\
%	& \qquad + \chi_{u^k \in [0;1]^{N_k}} + \chi_{b_1 = \ldots = {b_P} }
%\end{split}
%\end{equation*}%
\min_{\substack{u\\[1mm]b}}\;
\max_{\substack{h\\[0.5mm]v}}\; \hspace{-1pt} 
	& \sum_{k=1}^P \dotp{H_k u^k - b}{h_k}  -  \hspace{-1pt} \delta \hspace{-1pt}  \dotp{u^k}{\U_{N_k}}  \hspace{-1pt} +   \hspace{-1pt} \dotp{D_k u^k}{v_k} 
	\\
	&  \qquad -  \chi_{\norm{h_k}_{\infty}\le 1}- \chi_{\norm{v_k}_{\infty,2}\le \rho}
	\\
	& \qquad + \chi_{u^k \in [0;1]^{N_k}} + \chi_{b \ge \O }
\end{split}
\end{equation*}%
which fits the canonic form of Problem~\eqref{pb:primaldual}.
In the above equation, $D_k$ refers to the finite difference operator in the grid $\Om_k$ of image $I^k$.
%which can be written in the canonic form of Problem~\eqref{pb:primaldual}
%\begin{equation*}
%% \min_{p=(u;b)} \; \max_{q=(h;v)} \;
%\min_{p} \; \max_{q} \;
%	\dotp{Kp}{q} + R(p) + T(p) - F^*(q)-G^*(q)
%\end{equation*}%
%where:
%\begin{itemize}
%    \item we denote the primal variable $p=(u;b)$ %where variables $u=(u^1; \ldots;  u^P)$ and $b = (b_1; \ldots;  b_P)$, 
%    and the dual variable $q=(h;v)$; 
%    % $v = (v_1; \ldots; v_P)$;
%    \item Let $K$ a linear operator such that 
%    $$\dotp{Ku}{p} = \sum_{k=1}^P \dotp{H_k u^k - b_k}{h_k} +  \dotp{D_k u^k}{v_k} $$
%    %$A p = H_k u^k - b_k$; $\norm{Ap}_1 = \max_{q} \dotp{Ap}{q} - (\norm{q}_1)^* = \max_{q} \dotp{Ap}{q} - \chi_{\norm{q}_\infty \le 1}$
%    \item The non smooth primal term is
%    $$R(p) =  \chi_{[0,1]^{N_1 \times \ldots N_P}}(u) + \norm{A p}_1 = \sum_{k=1}^P \chi_{[0,1]^{N_k}}(u^k) + \norm{H_k u^k - b_k}_1$$ 
%    is separable, the proximity operator of $R$ is  therefore ... A FINIR \& VERIFIER 
%    \item $T(u,b) = - \sum_{k=1}^P \dotp{\U_{N_k}}{u_k}$
%\end{itemize}

The algorithm~\eqref{algo:primaldual} requires to compute the Euclidean projector onto 
the nonnegative orthant~\eqref{eq:proj_positive_orthant} % positive quadrant
and on 
the $\ell_\infty$ unit ball (similarly to Eq.~\eqref{eq:proj_Linf2}) 
%that write as follows:
%linear constraint \eqref{eq:linear_constraint_barycenter_splitting}:
\eq{
	\Proj_{\norm{.}_\infty\le 1}(h)(i) = \frac{ h(i) }{ \max \{ \abs{h(i)} , 1\} }.
}

\paragraph{Experiments} 
To illustrate the validity of the proposed model, we repeat the toy experiment of Figure~\ref{fig:coseg} in Figure~\ref{fig:coseg3}, where 
%As in the experiment of Figure~\ref{fig:coseg}, we first consider 
the same object is shown in two images with different backgrounds. 
While there is no more constraint on the size of the objects to segment as in Eq.~\eqref{eq:coseg_model0}, the model~\eqref{eq:coseg_model1} is still able to get a good co-segmentation of the data.

\begin{figure}[!htb]
    \begin{center}
    \begin{tabular}{cc}
    \includegraphics[width = 0.45\linewidth]{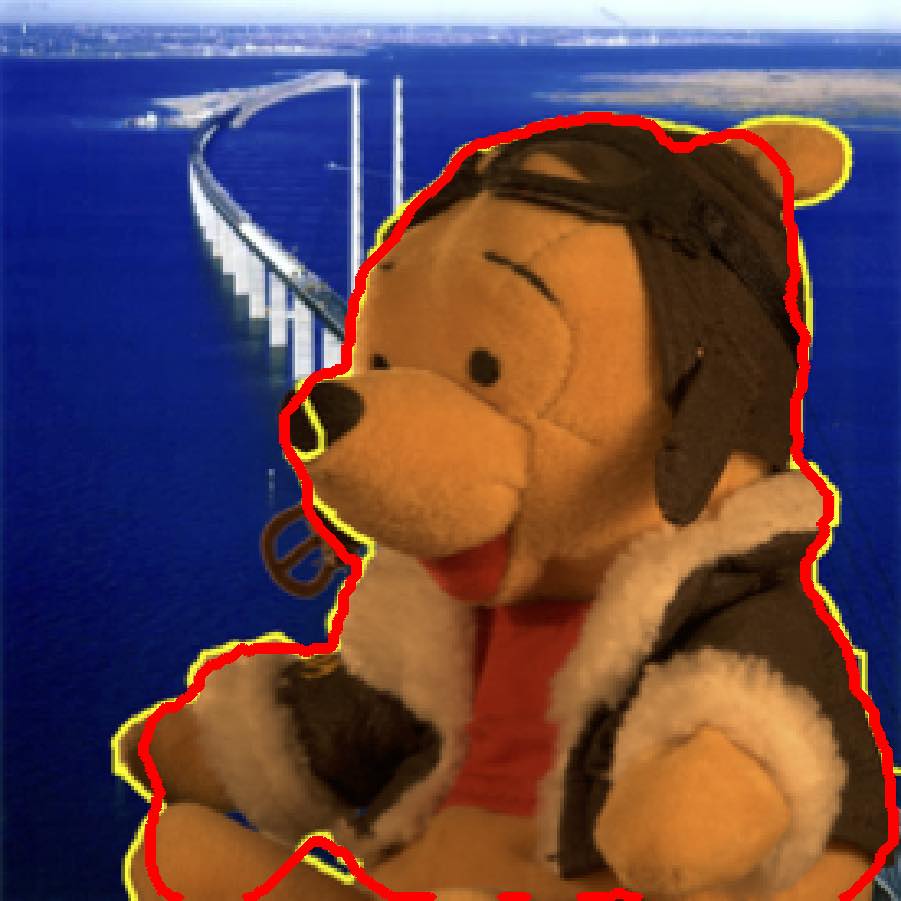}&\hspace{-0.4cm}
      \includegraphics[width = 0.45\linewidth]{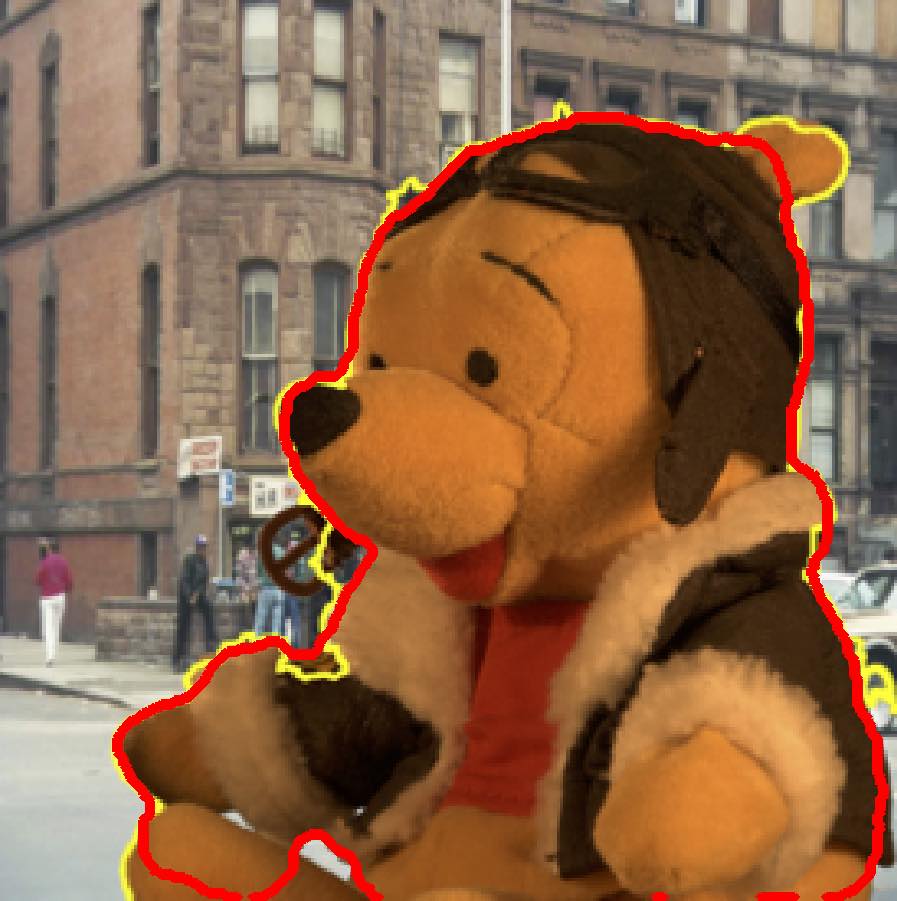}\\
          \includegraphics[width = 0.45\linewidth]{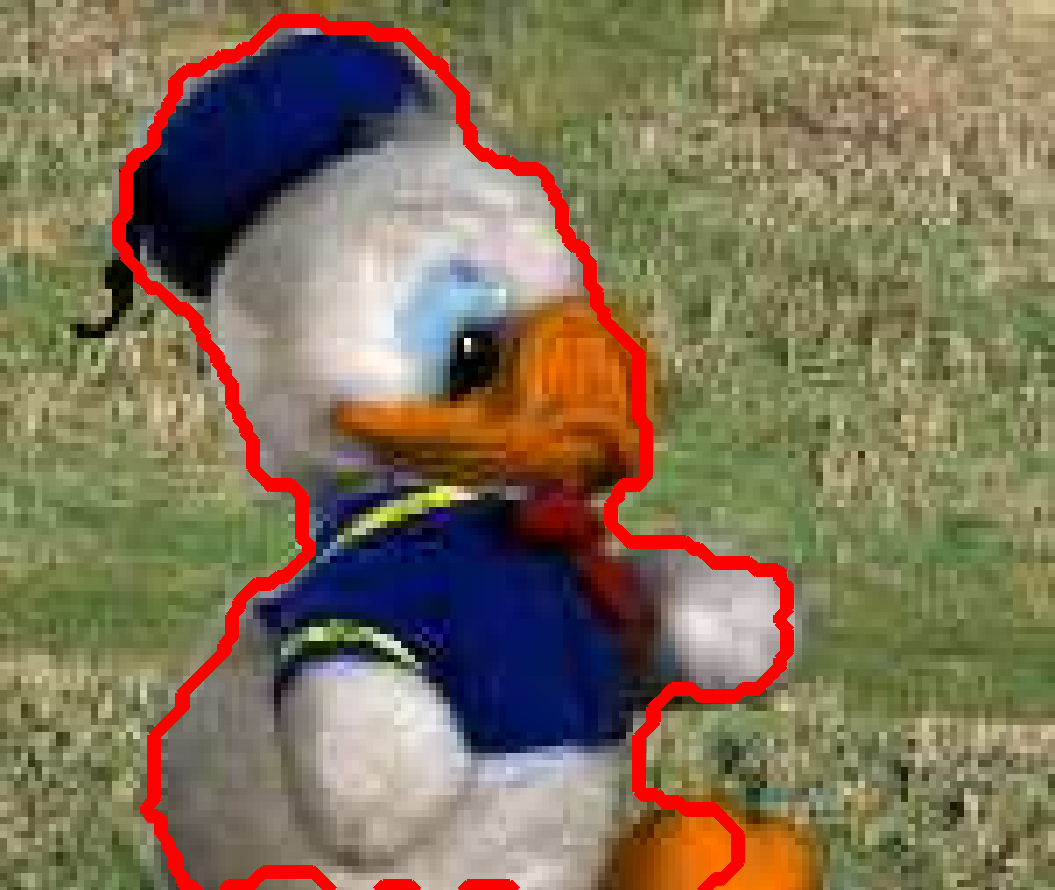}&\hspace{-0.4cm}
      \includegraphics[width = 0.45\linewidth]{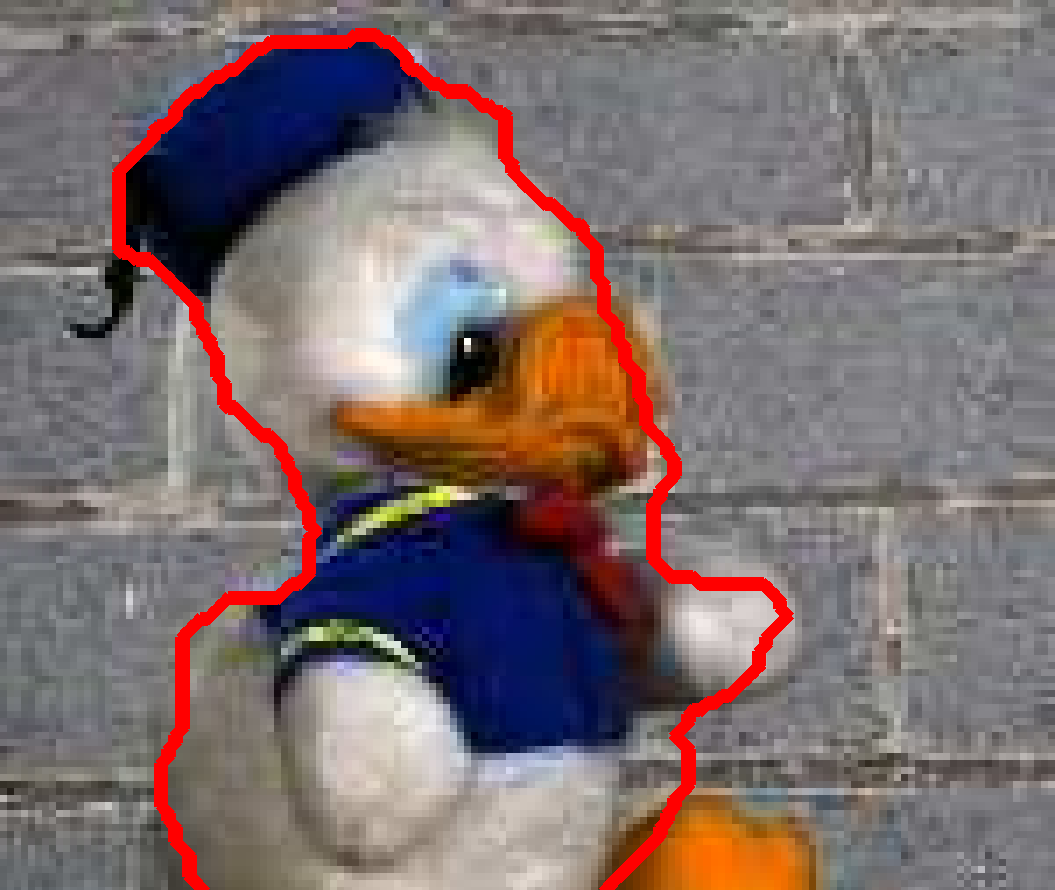}
    \end{tabular}
    \end{center}
    %\vspace*{-3mm}
    \caption{\label{fig:coseg3}  %Exemple de co-segmentation non supervisée
    Examples of co-segmentation of $P=2$ images with model \eqref{eq:coseg_model1}. The objects to segment have the same scale.  Comparing with the first line with~Figure \ref{fig:coseg}, results similar to the model~\eqref{eq:coseg_model0} and the method~\cite{Swoboda13} (in yellow), whereas no constraint is considered here on the size of the regions.%\vspace{-0.8cm}
   % \cmt{Il y a toujours ce problème d'adhérence au bord car les conditions au bord sont favorables : il faut changer la déf du gradient sur les bords
   \vspace*{5mm}~
   }
    %\vspace*{-1mm}
\end{figure}

\begin{figure*}[t]
    \begin{center}
    \begin{tabular}{ccccc}
    \includegraphics[width = 0.18\linewidth]{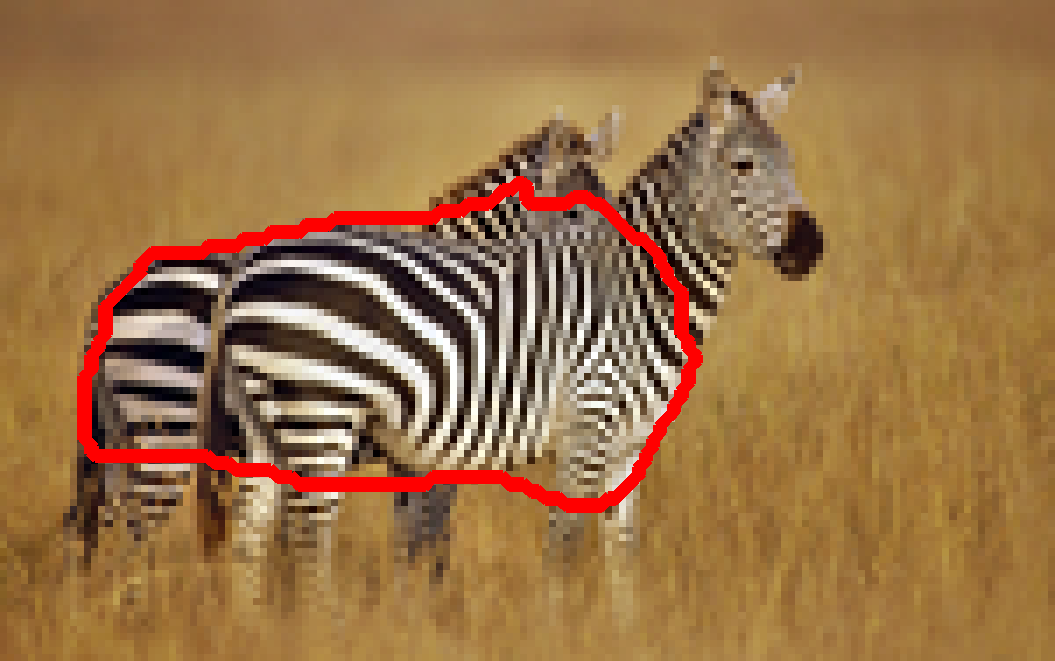}&\hspace{-0.4cm}
      \includegraphics[width = 0.18\linewidth]{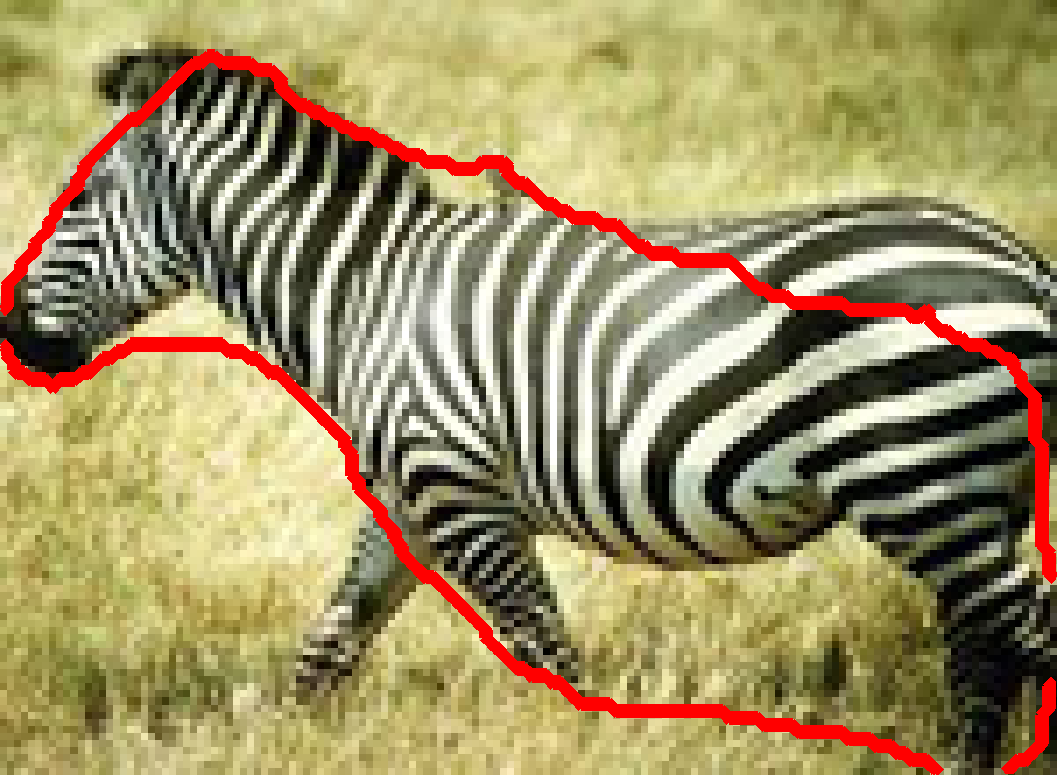}&\hspace{-0.4cm}
          \includegraphics[width = 0.18\linewidth]{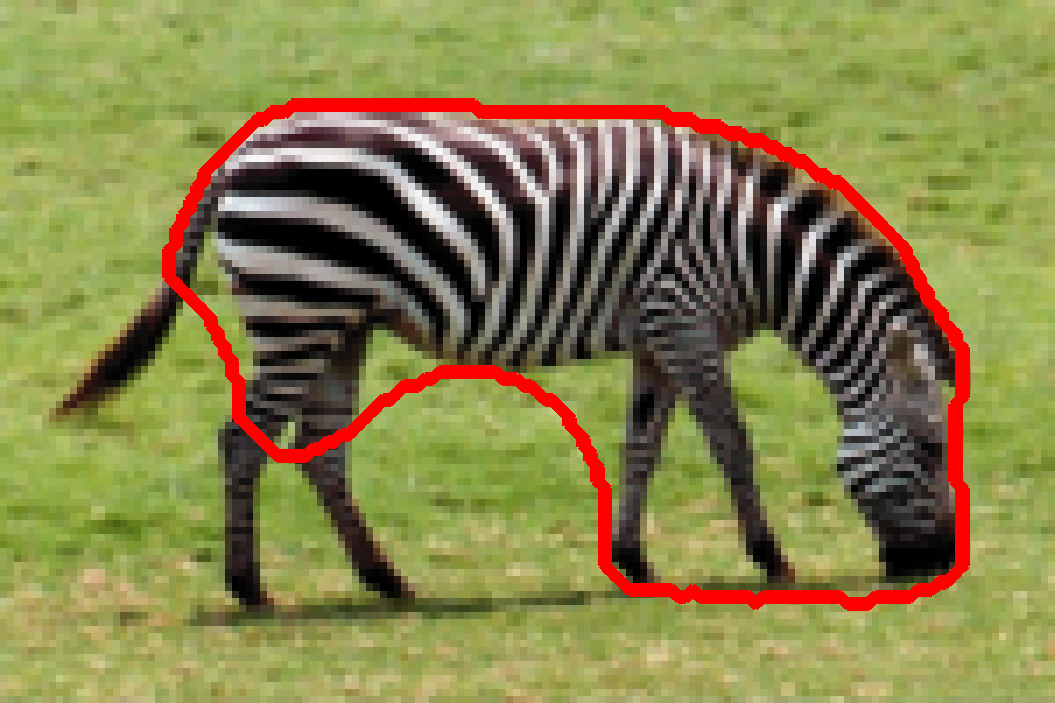}&\hspace{-0.4cm}
      \includegraphics[width = 0.18\linewidth]{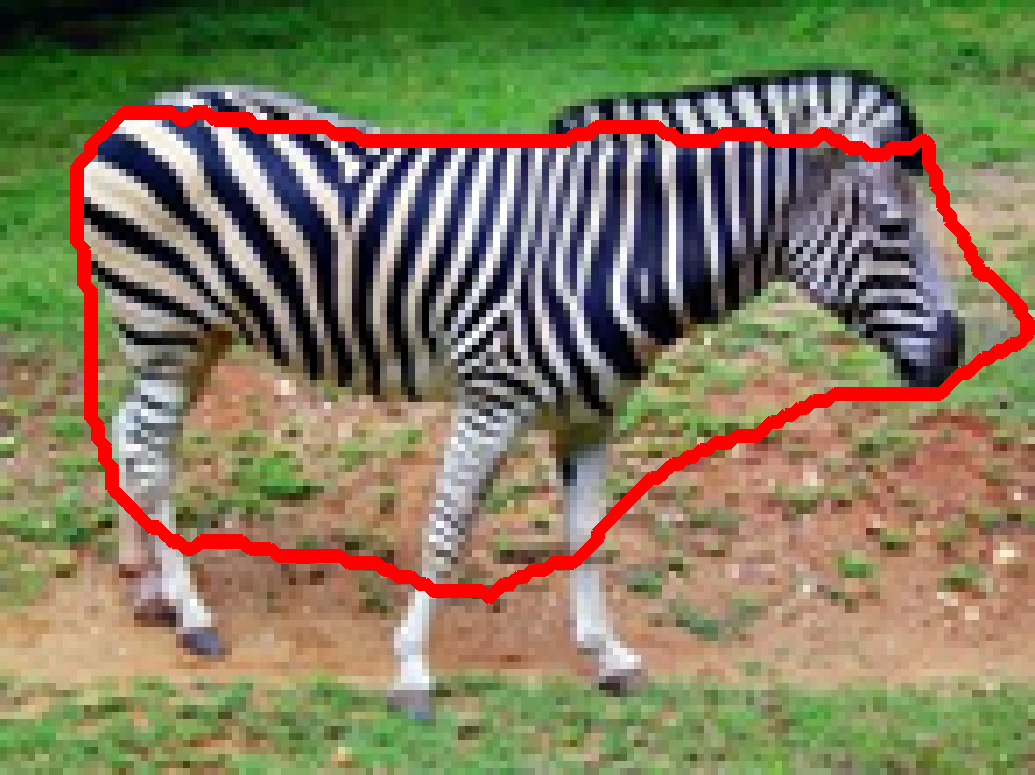}&\hspace{-0.4cm}
      \includegraphics[width = 0.18\linewidth]{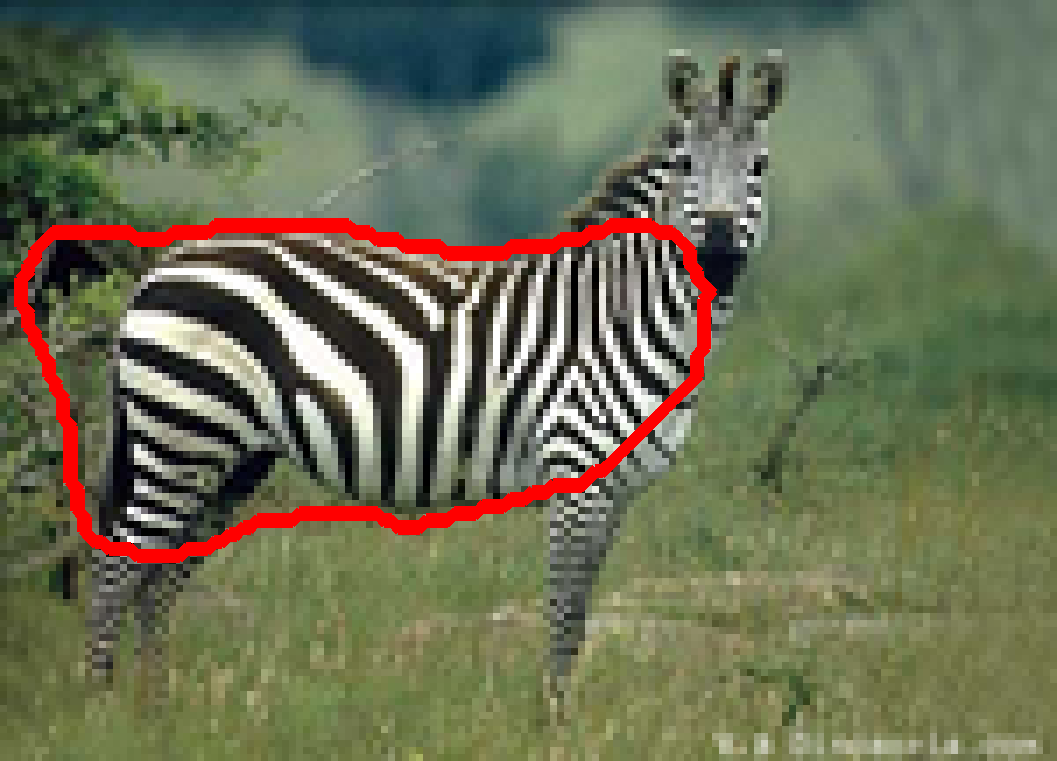}
    \end{tabular}
    \end{center}
    %\vspace*{-3mm}
\caption{\label{fig:coseg5}  %Exemple de co-segmentation non supervisée
   Co-segmentation of  $P=5$ images with model \eqref{eq:coseg_model1}.
   %  \cmt{Il y a toujours ce problème d'adhérence au bord car les conditions au bord sont favorables : il faut changer la déf du gradient sur les bords}
     \vspace*{5mm}~
}
\end{figure*}

%Taking $\beta=30\%$, the model \eqref{eq:coseg_model1}, even if non-convex, leads to a good co-segmentation of the ducks.

In Figure \ref{fig:coseg5}, %taking $\beta=30\%$, 
we illustrate how this new model  is able to segment objects of different scales in $P=5$ images. The area of the zebra can be very different in each image. 
Note that for simplicity the same regularization parameter $\rho$ and ballooning parameter $\delta$ are used for all images in the model~\eqref{eq:coseg_model1}, whereas one should tune separately these parameters according to each image in order to obtain more accurate co-segmentations.
The  histogram $b$ recovered from the model is shown in Figure~\ref{fig:bary}.

\begin{figure}[!htb]
	\centering{\includegraphics[width=7cm]{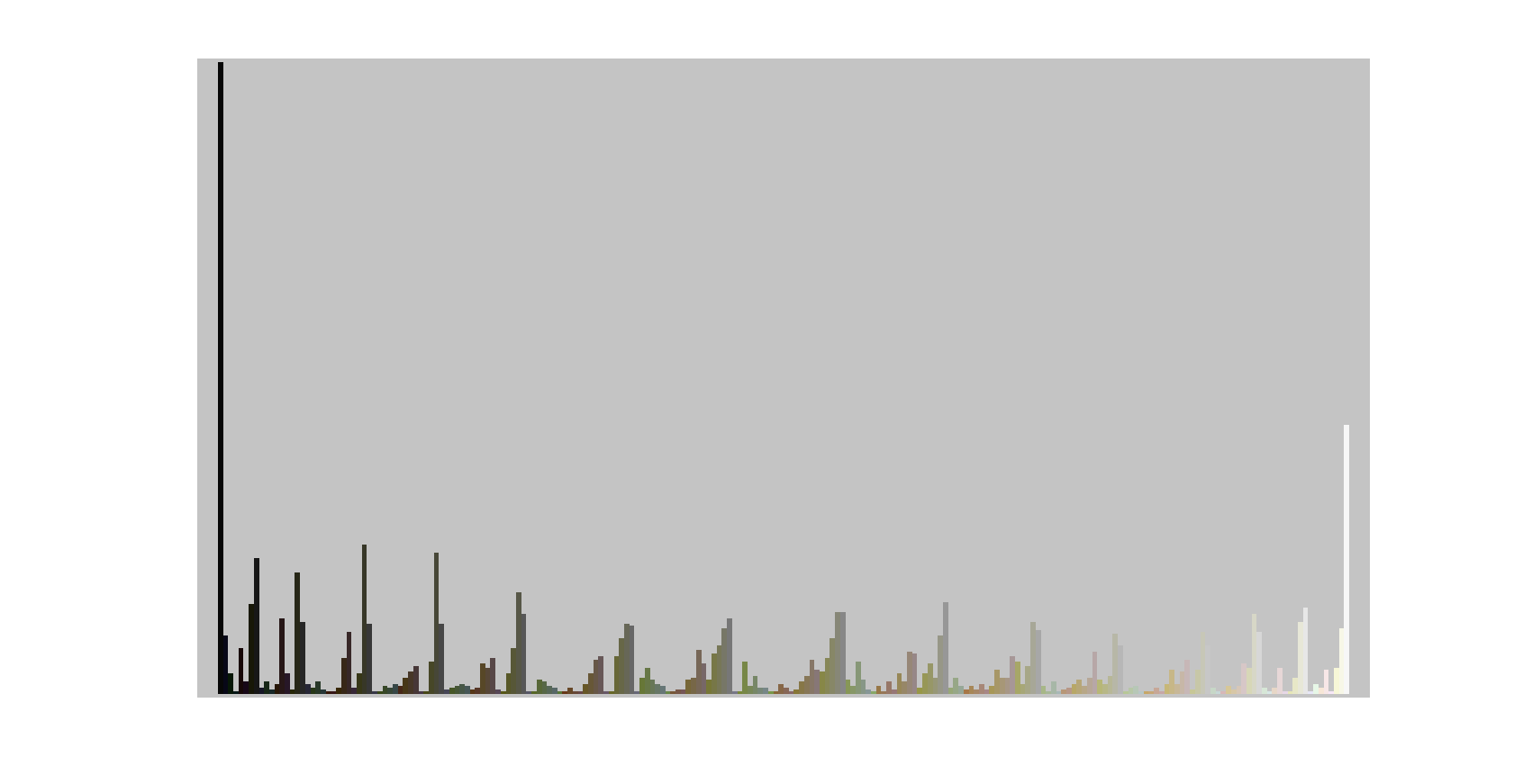}}
	\caption{Learnt barycenter histogram for Figure \ref{fig:coseg5}. It mainly contains black and white colors corresponding to the zebras.
\vspace*{5mm}~}
	\label{fig:bary}
\end{figure}

%We next emphasize that this model is still not sufficient for general purposes.
%When considering more images ($P=5$), as in Figure \ref{fig:coseg5}, 
%Finally, we show in Figure \ref{fig:coseg6} a main limitation of this unsupervised model. 
%When the object scale is very different, taking a too large parameter ($\beta=15\%$ in the first row) leads to over-segmentations, whereas a too small parameter ($\beta=1\%$ in the second row) gives very bad co-segmentations.
%
Moreover, it seems necessary to add information on the background for improving these results. 
In the previous examples of Figures~\ref{fig:coseg3} and \ref{fig:coseg5}, the backgrounds were enough different in the different images to be discarded by the model.
As soon as the backgrounds of the co-segmented images contain very similar informations (for instance grey regions outside the gnome in images of Figure~\ref{fig:coseg6}), the ballooning term in Eq.~\eqref{eq:coseg_model1} forces the model to include these areas in the co-segmentations.
%there is not reason for the model not to include these areas in the co-segmentations.

%\begin{figure}[htb]
%    \begin{center}
%    \begin{tabular}{ccc}
%    \includegraphics[width = 0.3\linewidth]{zeb_coseg1a}&\hspace{-0.4cm}
%      \includegraphics[width = 0.3\linewidth]{zeb_coseg2a}&\hspace{-0.4cm}
%          \includegraphics[width = 0.3\linewidth]{zeb_coseg3a}
%    \end{tabular}
%    \end{center}
%    %\vspace*{-3mm}
%    \caption{\label{fig:coseg4}  %Exemple de co-segmentation non supervisée
%    Non-convex co-segmentation of  $P=3$ images with $\beta=30\%$.%\vspace{-1cm}
%    }
%\end{figure}

%\begin{figure}[htb]
%    \begin{center}
%    \begin{tabular}{cc}
%    \includegraphics[width = 0.45\linewidth]{zeb_coseg1}&\hspace{-0.4cm}
%      \includegraphics[width = 0.45\linewidth]{zeb_coseg2}\\
%          \includegraphics[width = 0.45\linewidth]{zeb_coseg3}&\hspace{-0.4cm}
%      \includegraphics[width = 0.45\linewidth]{zeb_coseg4}
%    \end{tabular}
%    \end{center}
%    \vspace*{-3mm}
%    \caption{\label{fig:coseg}  %Exemple de co-segmentation non supervisée
%    Résultats obtenu par notre modèle non-convexe pour $4$ images}
%\end{figure}

\begin{figure}[htb]
    \begin{center}
    \begin{tabular}{ccc}
        \includegraphics[width = 0.3\linewidth]{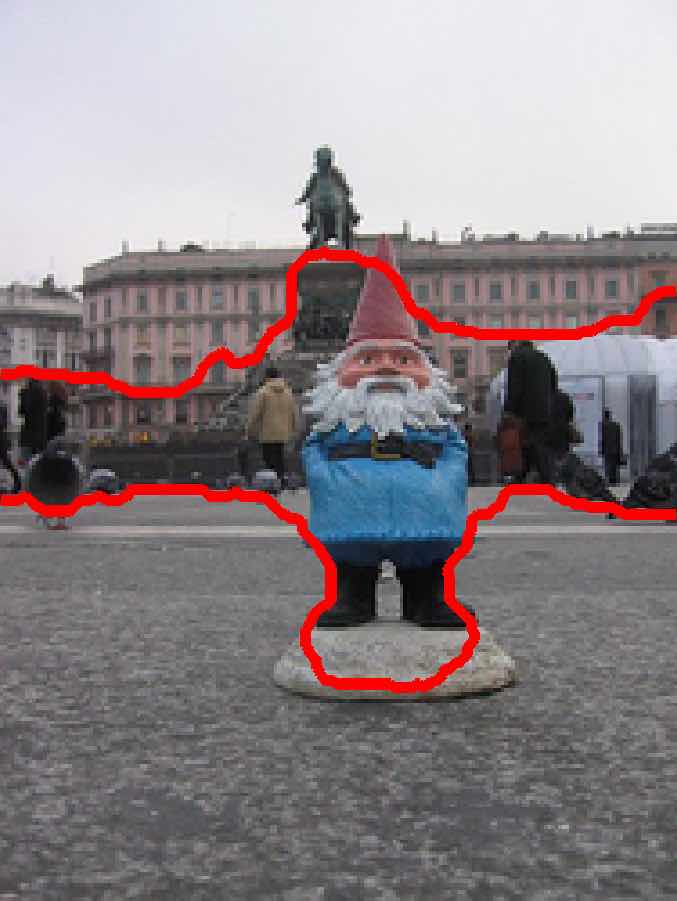}&\hspace{-0.4cm}
      \includegraphics[width = 0.3\linewidth]{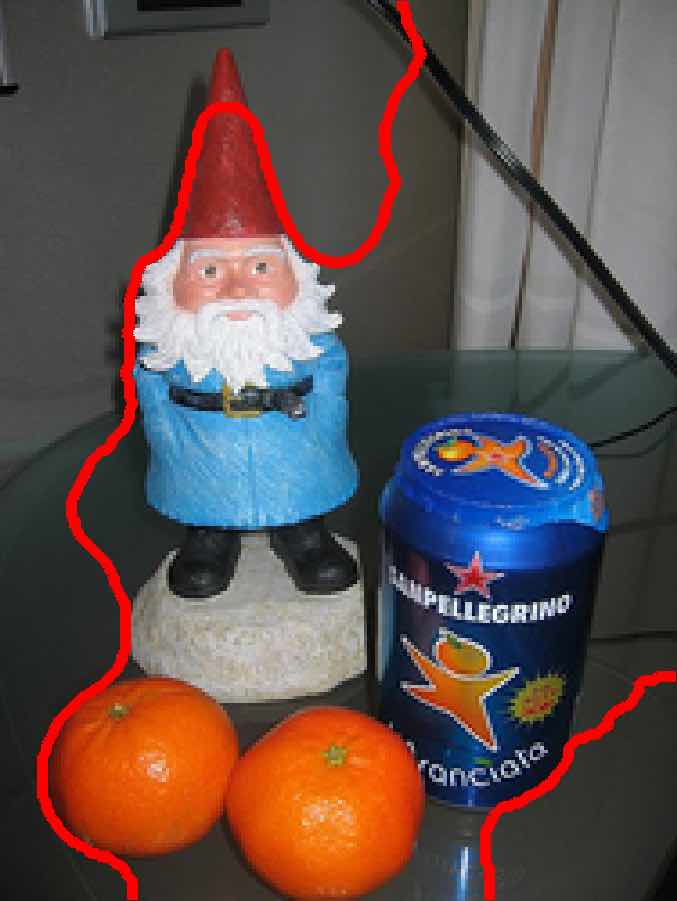}&\hspace{-0.4cm}
          \includegraphics[width = 0.3\linewidth]{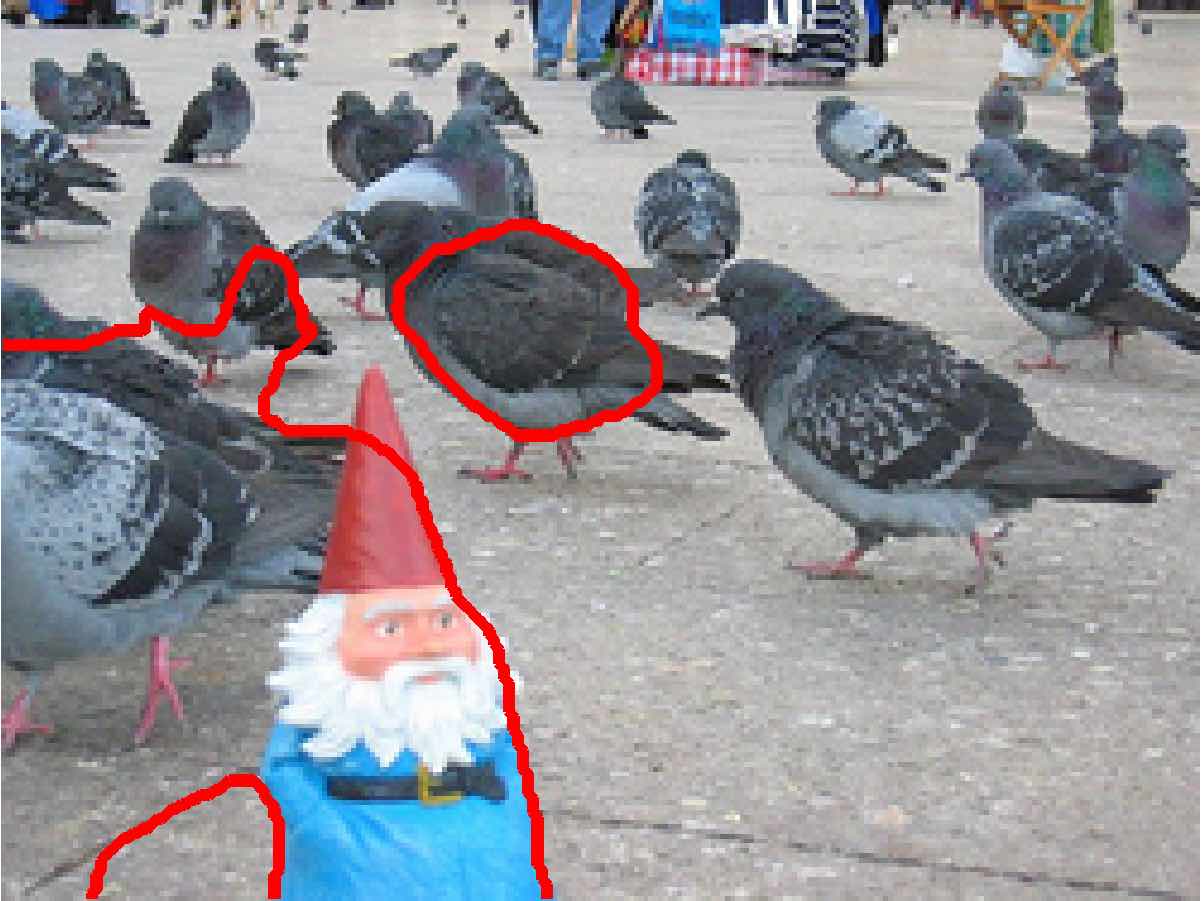}\\
    \end{tabular}
    \end{center}
    %\vspace*{-3mm}
    \caption{\label{fig:coseg6}  %Exemple de co-segmentation non supervisée
    Incorrect co-segmentation result in case of similar backgrounds in the two images. % We first illustrate the influence of the parameter $\beta$, taken as $15\%$ in the first row and $1\%$ in the second row. A large parameter over-segments the gnomes in the different images.  A too small parameter leads to very bad co-segmentations. Hence, when the difference of scale is too important the model \eqref{eq:coseg_model1} is not able to give pertinent co-segmentations.%\vspace{-9mm}
    \vspace*{10mm}~
    }
\end{figure}

\section{Conclusion and future work}

In this work, several formulations have been proposed to incorporate transport-based cost functions in convex variational models for  supervised segmentation and unsupervised co-segmentation. The proposed framework includes entropic regularization of the optimal-transport distance, and deals with multiple objets as well as collection of images.
%With our model, supervised global segmentation of still images can be realized.
%Unsupervised co-segmentation of multiple images has also been studied.

As already demonstrated for the image segmentation problem, optimal transport yields significative advantages when comparing feature distribution from segmented regions (robustness with respect to histogram quantization, the possibility to incorporate prior information about feature into the transport cost, definition of a metric between different types of feature, \emph{etc}).
When considering entropic regularization, the algorithmic scheme is yet very similar to the one obtained for the $\ell_1$ norm, 
at the only expense of requiring more memory.
We observed, as acknowledged in~\cite{CuturiPeyre_smoothdualwasserstein}, that such regularization offers practical acceleration for small histograms but also improves robustness to outliers (such as noise or rare features).
However, we also emphasized that large regularization degrade significatively the performance.

The main limitation highlighted and discussed in this work is the lack of scale invariance property for the unsupervised co-segmentation problem due to the convex formulation. In comparison, non-convex formulations of optimal transport with probability distribution such as~\cite{PeyreWassersteinSeg12} yields such invariance, while usual cost functions such as $\ell_1$ offer some robustness~\cite{Vicente_co-seg_eccv10}.
A promising perspective to overcome this restriction it the use of the unbalanced optimal transport framework recently studied in \cite{Vialard15b,Frogner}.

In the future, other potential improvements of the model will be investigated, such as 
the optimization of the final thresholding operation, 
some variable metrics during optimization, 
the use of capacity transport constraint relaxation~\cite{Ferradans_siims14},
%or unbalanced transport models \cite{Frogner,Vialard15b} for co-segmentation, 
the incorporation of other statistical features 
and the integration of additional priors such as the shape of the objects \cite{Schmitzer_jmiv13,Schmitzer_jmiv15}.

%%%%%%%%%%%%%%%%%%%%%%%%%%%%%%%%%%%%%%%%%%%%%%%%%%%

\bigskip

\subsection*{Acknowledgements}

%\begin{acknowledgements}
The authors acknowledge support from the CNRS in the context of the ``Défi Imag’In'' project CAVALIERI (CAlcul des VAriations pour L'Imagerie, l'Edition et la Recherche d'Images).  
This study has been carried out with financial support from the French State, managed by the French National Research Agency (ANR)
in the frame of the Investments for the future Programme IdEx Bordeaux (ANR-
10-IDEX-03-02). 
The authors would like to thanks Gabriel Peyré and Marco Cuturi for sharing their preliminary work
and Jalal Fadili for fruitful discussions on convex optimization.
%\end{acknowledgements}

% APPENDIX
%\clearpage
\appendix
\section{Appendices}%Proofs

\input{appendix}

% ------------- BIBLIO -------------------

\bibliographystyle{abbrv}
\bibliography{refs}

\end{document}

%% file: appendix.tex
\subsection{Norm of $K$}
\label{proof:norm_K}
\begin{proof}
We recall that $K=[H^\Trp,A^\Trp,-H^\Trp,-B^\Trp,\Diff^\Trp]^\Trp$
so that
\eq{
\norm{K}\le 2\norm{H} + \norm{A}  + \norm{B} + \norm{\Diff} .
}
%\eq{
%\norm{K}\le 2\sqrt{N} + \sqrt{N} + \sqrt{N} + \norm{\Diff} 
%}
For rank one operator $A$ and $B$ we can write 
\eq{
\norm{B} = \max_{\norm{x}=1}  \norm{Bx} =  \max_{\norm{x}=1} \abs{\dotp{x}{\U_N}} \norm{b} =  \norm{\U_N} \norm{b} =  \sqrt{N}\norm{b}
}
%and for histogram $b$ subject to $b\ge 0$ and $\dotp{b}{\U}=1$
%\eq{
%\norm{b} = \sqrt{\sum_{i=1}^M b_i^2} \le \sqrt{\sum_{i=1}^M \abs{b_i}} = 1
%}
and for histogram $b$ subject to $b\ge 0$ and $\dotp{b}{\U}=1$, we have
$
\norm{b}_2 \le \norm{b}_1 = 1
$
and the same for histogram $a$.
For the hard assignment operator
\eq{
	\norm{H} \le \norm{H}_F = \sqrt{\sum_{x\in \Om} \sum_{i=1}^M H_{i,x}^2} = \sqrt{N}
}
where the equality holds when assignment matrix is $H = \U^\Trp_N$.
The finite difference operator $\Diff$ verifies (see for instance \cite{Chambolle2004})
%\cmt{à vérifier à cause des conditions au bord qui sont bien nulle avec la demo de chambolle (condition de neumann)}
\eq{
	\norm{\Diff} \le \sqrt{8}.
}
Finally, we obtain
\eq{
	\norm{K}\le 4 \sqrt{N}+ \sqrt{8}.
}
\end{proof}

\subsection{Proof of the special $\ell_1$ case}
\label{sec:proof_l1}

\begin{proof}
We consider here that two histograms $(a,b) \in \S$ and the cost matrix $C$ such that $C_{i,j} = 2 (1-\delta_{ij})$.
This cost is only null when not moving mass (that is $j=i$) and constant otherwise, so that an optimal matrix $P \in {\Pp}(a,b)$ must verifies $P_{i,i} = \min{(a_i,b_i)}$ to minimize the transport cost $\dotp{P}{C}$. 
Therefore, we have that 
\eq{
\begin{split}
	\MK(a,b) 
	&= 2 \sum_i \sum_{j \not=i} P_{i,j} = 2 \sum_i (a_i - \min{(a_i,b_i)}) 
	\\
	&= 2 \sum_i (a_i - b_i) \Idc_{a_i > b_j}
	\\
	&= 2 \sum_j (b_j - a_j) \Idc_{a_i < b_j}.
\end{split}
}
The last equality is obtained by symmetry.
Then, adding the two last equalities, we obtain the desired result
\eq{
	\MK(a,b) = \norm{a-b}_1
	.
}
\smartqed
\end{proof}

\subsection{Proof of Corollary \ref{corollary:dual_simplexNmax}}%
\label{proof_coro}

\begin{proof}
For sake of simplicity, %concision
the notation $\U$ without subindex refers to either the vector $\U_{M}$ or a matrix $\U_{M \times M}$ depending on the context.
Let us consider the problem~\eqref{eq:sinkhorn_distanceN} using Lagrangian multipliers:
\begin{equation*}
\begin{split}
	& \hspace*{-4mm}\MK_{\lambda,\leq N}(a,b)
	\\ 
	= &\min_{P \in {\Pp}(a,b)} \langle P,C \rangle +\tfrac{{1}}\lambda \langle P, \log (P/N) \rangle % {\color{red}\langle  P, \U \rangle}=N
+ {\chi_{\Delta_{\le N}}(a,b)}
	\\
	%\hspace*{-10mm}
	=&\min_{P\ge \O} \max_{\substack{u, v\\w \leq 0}} 
	\left\{\begin{array}{l} 
		\langle P,\tfrac1\lambda {\log P/N}+ C\rangle 
		+\langle u,a - P \U \rangle %  \U_{M}
		\\
		+ \langle v,b - P^\Trp  \U \rangle + w (N-\U^\Trp  P \U) %  \U_{M}
		%\right.\\
		%&\qquad\qquad\qquad\qquad \left. 
	\end{array}\right\} % {\color{red}\langle  P, \U \rangle}
  	\\%\hspace*{-10mm}
 	= & \max_{\substack{u, v\\w \leq 0}}  
	\left\{ \begin{array}{l}
		\langle u,a \rangle+ \langle v,b \rangle + Nw \\
		+ \min_{P\ge \O}  \langle P,\frac1\lbd { \log P/N}
		+C - u \U^\Trp  - \U v^\Trp - w\, \U \rangle 
	\end{array}\right\} 
	\\ 
\end{split}
\end{equation*}
using the fact that the (normalized) negative-entropy is continuous and convex.
%where $\U_{M \times M} = \U_{M} \U_{M}^\Trp$ is a matrix full of one, and
%where we used the fact that  $\langle u,P\U_{M} \rangle = \sum_{i=1}^{M} u_i \sum_{j=1}^{M} P_{i,j} =  \sum_{i,j} P_{i,j} u_{i} = \langle P, u \U_{M}^\Trp \rangle $ and $\langle v,P^\Trp\U_{M} \rangle = \langle P,  \U_{M} v^\Trp \rangle$.
%
%We then compute the first partial derivative of the Lagrangian 
The corresponding Lagrangian is
\begin{equation*}
\begin{split}
	{\cal L}(P,u,v,w) & = \langle u,a \rangle+ \langle v,b \rangle + Nw 
	%& \quad = \langle P,\log P -\log N+ \lambda(C - u \U_{M}^\Trp - \U_{M} v^\Trp - w \U_{M\times M}) \rangle
	 \\
	 & \frac{1}{\lbd} \langle P,\log P -\log N+ \lambda(C - u \U_{}^\Trp - \U_{} v^\Trp - w \U_{}) \rangle.
\end{split}
\end{equation*}
The first order optimality condition $\partial_P {\cal L}(P^\star,u,v,w) = \mathbf{0}$ gives:
\begin{equation*}
\begin{split}
	%\log (P^\star/N) +\hspace{-1pt} \lambda (C\hspace{-1pt} - \hspace{-1pt}u \U_{M}^\Trp \hspace{-1pt}-\hspace{-1pt} \U_{M} v^\Trp  \hspace{-1pt}- \hspace{-1pt}w \U_{M\times M}) \hspace{-1pt}+\hspace{-1pt} \U_{M\times M}\hspace{-1pt}=\hspace{-1pt}\mathbf{0},
	\log P^\star - \log N +  \lbd  (C  -  u \U_{}^\Trp - \U_{} v^\Trp  -  w \U_{}) + \U_{}=\mathbf{0},
\end{split}
\end{equation*}
that is
\begin{equation}\label{eq:optim_P_MK_lbd_N}
%\begin{split}
%\log P_{i,j}^\star &= -1 +\log N- \lambda(C_{i,j} - u_i - v_j - w)  \\
P_{i,j}^\star =N e^{-1+\lbd w } \; e^{- \lbd(C_{i,j} - u_i - v_j)} \ge 0.
%\end{split}
\end{equation}
%
%Replacing this result back in the equation, we get:
Using this expression in $\L(P^\star,u,v,w)$
\begin{align}\label{calcul_intermediaire0}
%\begin{split}
	&\hspace*{-3mm}\MK_{\lambda,\leq N}(a,b)\nonumber\\ 
	=&\max_{u ,v,w \leq 0}\, \langle u,a \rangle+ \langle v,b \rangle + Nw
	- \frac1\lambda\langle P^\star,\U_{}  \rangle % \left\{  \right\}
	\\
	=&\max_{u ,v}  \hspace{-1pt} 
	\left\{ \begin{array}{l}
		\langle u,a \rangle+ \langle v,b \rangle \\
		 +\max_{w \leq 0}  Nw - \frac{N}\lambda e^{-1+\lambda w} \sum_{i,j} e^{ - \lambda (C_{i,j} - u_i - v_j) }
	\end{array}  \hspace{-2pt} \right\}
	\nonumber
%\end{split}
\end{align}

Observe that the expression of $P^\star(u,v,w)$ in \eqref{eq:optim_P_MK_lbd_N}, 
that becomes $P^\star = N e^{\lbd w} Q_{\lbd}(u,v)$ using definition \eqref{eq:dual_sinkhorn_nonnormalized_matrix},
is scaled by the Lagrangian variable $w$
which corresponds to the constraint %on the magnitude of $P^\star$ 
$\dotp{P^\star}{\U} \le N$.
We consider now whether or not this equality holds.

%\paragraph{Case 1: $w<0$}
\medskip
\noindent \textbf{\bf Case 1: $\dotp{P^\star}{\U} = N$.}
Let us first consider the case where the constraint is saturated, that is when $w<0$ due to the complementary slackness property. 
The maximum of function 
%$f(w) = Nw - \frac{N}\lambda e^{-1+\lambda w} x$ 
$f(w) = Nw - \frac{N}\lambda e^{\lambda w} \dotp{Q_\lbd}{\U}$ %(u,v)
which is concave ($\partial_w^2 f(w) < 0$), 
is obtained for $w^\star$ subject to
$$e^{\lambda w^\star} = \frac{1}{\dotp{Q_\lbd}{\U}} = {\left( \sum_{i,j} e^{- \lambda (C_{i,j} - u_i - v_j)-1}\right)}^{-1} \le 1.$$
One can check that the equality $\sum_{i,j} P_{i,j}^\star = N$ is indeed satisfied.
%Then, we have:
In addition, the maximum of $f$ verifies
\begin{equation*}
\begin{split}
	f(w^\star)
	&= N w^\star - \frac{N}\lambda =  \frac{N}\lambda (\lambda w^\star-1) = \frac{N}\lambda \log e^{\lambda w^\star -1} 
	\\
	%&= - \frac{N}\lambda\log \left(\sum_{i,j} e^{- \lambda(C_{i,j} - u_i - v_j)}\right)
	&= - \frac{N}{\lambda} \log \dotp{Q_\lbd}{\U} - \frac{N}\lambda 
\end{split}
.
\end{equation*}
%}
%
%We finally get
The problem~\eqref{calcul_intermediaire0} becomes
\begin{equation}
\begin{split}
	& \hspace*{0mm}
	\MK_{\lambda,\leq N}(a,b)\\
	& \quad = \max_{u,v} \;\; %\max_{u \in \R^{M},v \in \R^{M}} 
	\langle u,a \rangle +  \langle v,b \rangle  -  \frac{ N}\lambda {  \log}  \sum_{i,j}  e^{ - \lambda (C_{i,j} - u_i - v_j) } %\left\{  \right\} .
	.
\end{split}
\end{equation}

From the definition of the Legendre-Fenchel transformation, this implies that 
\eq{
	\MK_{\lambda,\leq N}(u,v)=\left(\frac{N}\lambda  {  \log}\left( \sum_{i,j} e^{ - \lambda (C_{i,j} - u_i - v_j) }  \right)\right)^*
.}
As these functions are convex, proper and lower semi-\-continuous, we have that $\MK^{**}_{\lambda,N}=\MK_{\lambda,N}$ which concludes the proof for the case $\dotp{Q_\lbd(u,v)}{\U} \ge 1$.

%\paragraph{Case 2: $w=0$} 
\medskip
\noindent \textbf{\bf Case 2: $\dotp{P^\star}{\U} < N$.}
Now we consider the case where the constraint is not saturated, \emph{i.e.} $w=0$.
The expression of $P^\star(u,v,w)$ in \eqref{eq:optim_P_MK_lbd_N} becomes $P^\star = N Q_{\lbd}(u,v)$.
% using definition \eqref{eq:dual_sinkhorn_nonnormalized_matrix},
Going back to relation \eqref{calcul_intermediaire0}, we have directly 
\eq{
	\MK_{\lambda,\leq N}(u,v) 
	=\max_{u,v}  \;\; \langle u,a \rangle+ \langle v,b \rangle 
- \frac{N}\lambda \dotp{Q_\lbd(u,v)}{\U} %\sum_{i,j} e^{ -1 - \lambda (C_{i,j} - u_i - v_j) } \left\{\right\}
}
which concludes the proof for the case $\dotp{Q_\lbd(u,v)}{\U} \le 1$.

%Going back to relation \eqref{calcul_intermediaire0}, the  unconstrained optimal value $w^*$ is still given by 
%$$
%	e^{\lambda w^*} = \frac{1}{\sum_{i,j} e^{-1 - \lambda (C_{i,j} - u_i - v_j)}}
%$$
%which involves 
%$$
%	w^*=-\frac{1}\lambda \log\left({\sum_{i,j} e^{-1- \lambda (C_{i,j} - u_i - v_j)}}\right).
%$$
%Hence, $w^*\ge 0$ as soon as ${\sum_{i,j} e^{1- \lambda (C_{i,j} - u_i - v_j)}} \le 1.$
%In this case, the optimal value is therefore projected to $w^*=0$. Relation \eqref{calcul_intermediaire0} give us the following expression:
%$$\max_{u \in \R^{M},v \in \R^{M}}  \left\{ \langle u,a \rangle+ \langle v,b \rangle 
%- \frac{N}\lambda \sum_{i,j} e^{ -1 - \lambda (C_{i,j} - u_i - v_j) } \right\},$$

\end{proof}

\subsection{Proof of proposition \ref{prop:lipschitz_dual}}
\label{sec:proof_lipschitz_dual}

\begin{proof}
The derivative $\nabla \MK^{*}_{\lbd,\le N} (X) $ with $X=(u;v)$  is lipschitz continuous \emph{iff} there exists $L_{\MK^*}>0$ such that 
\eq{
	\lVert  \nabla \MK^{*}_{\lbd,\le N}(X)-\nabla \MK^{*}_{\lbd,\le N}(X') \rVert\leq L_{\MK^*}\lVert X-X'\rVert.
}
We denote as ${\cal U}$ the set of vectors $X=(u;v) \in \R^{2M}$ such that 
%$\sum_{i,j}e^{-1-\lambda (C_{i,j} -u_i - v_j)}\geq 1$. 
$\dotp{Q_\lbd(u,v)}{\U} > 1$ (where $Q_\lbd$ is defined in Eq.~\eqref{eq:dual_sinkhorn_nonnormalized_matrix}).
%In the same way, $V$ is the set of $X$ \emph{s.t.} $\dotp{Q_\lbd(u,v)}{\U} \le 1$. %$\sum_{i,j}e^{-1-\lambda (C_{i,j} -u_i - v_j)}\leq 1$.
We denote ${\cal V = U}^c$ the complement of $\cal U$ in $\R^{2M}$.
{Observe that the set $\cal V$ is convex, as it corresponds to a sublevel set of the convex function $\MK^{*}_{\lbd,\le N}$.
% defined in \eqref{eq:dual_sinkhorn_inegality}.
}

Due to the expression of the gradient in \eqref{eq:dual_derivative_inegality} that is different on sets $\cal U$ and $\cal V$, we will consider the following three cases. 

%\paragraph{Case 1}
\bigskip
\noindent \textbf{\bf Case 1.}
Let $X,X'\in \cal U$.
As  $\nabla \MK^{*}_{\lbd,\le N}$ is derivable in the set $\cal U$, it is a lipschitz function \emph{iff} the norm of the Hessian matrix $\H$ of $\MK^{*}_{\lbd,\le N}$ is bounded. 
Denoting $\{\mu_i\}_{i=1}^{2M}$ the eigenvalues of $\H$, its $\ell_2$ norm is defined as $\lVert\H\rVert=\max_i |\mu_i|$. 
Moreover, as $\MK^{*}_{\lbd,\le N}$ is convex, all eigenvalues are non negative.
Thus, we have that the norm of $\H$ is bounded by its trace: $\lVert\H\rVert\leq \Tr(\H)=\sum_i \mu_i=\sum_i \H_{ii}$.

The Hessian matrix $\H$ of $\MK^{*}_{\lbd,\le N}$ is defined as:
\eq{
	%\H=\begin{bmatrix}\H^{11}&\H^{12}\\(\H^{12})^\Trp&\H^{22}\end{bmatrix},
	\H=\begin{bmatrix}\H^{11}&\H^{12}\\\H^{21}&\H^{22}\end{bmatrix},
	\quad \text{with } \quad \H^{mn}=\nabla_m \nabla_n^\Trp \MK^{*}_{\lbd,\le N}.
}
%with $ \H^{11}=\nabla_1 \nabla_1 \MK^{*}_{\lbd,\le N}$, $\H^{12}=\nabla_2\nabla_1  \MK^{*}_{\lbd,\le N}$ and $\H^{22}=\nabla_2\nabla_2 \MK^{*}_{\lbd,\le N}$.
Combining Equations \eqref{eq:dual_derivative_inegality} and \eqref{eq:dual_sinkhorn_nonnormalized_matrix}, we have
\begin{equation*}
\begin{split}
	& \hspace*{-1mm} (\nabla_1 \MK^{*}_{\lbd,\le N}(u,v))_i
	=\partial_{u_i} \MK^{*}_{\lbd,\le N}(u,v)
	\\
	%= { N} P^\star(u_i,v) \U_{M}\\
	& \quad = { N} \left( \frac{Q_\lbd (u,v) \U_{}}{\dotp{Q_\lbd (u,v)}{\U}} \right)_i
	= {N} \frac{\sum_l  e^{- \lambda (C_{i,l} - u_i - v_l) }}{\sum_{k,l}  e^{- \lambda (C_{k,l} - u_k^{} - v_l) }}.
\end{split}
\end{equation*}
Hence the diagonal elements of the matrix  read
\begin{equation*}
\begin{split}
	\H^{11}_{ii}&=\partial^2_{u_i} \MK^{*}_{\lbd,\le N}(u,v)\\&=\lambda {N} \frac{\sum_l e^{- \lambda (C_{i,l} - u_i-v_l) }{\sum_{k\neq i,l}  e^{- \lambda (C_{k,l} - u_k^{} -v_l) }}}{ {\left(\sum_{k,l}  e^{- \lambda (C_{k,l} - u_k^{}- v_l) } \right)}^2 },\\
	\H^{22}_{jj}&=\partial^2_{v_j} \MK^{*}_{\lbd,\le N}(u,v)\\&=\lambda {N} \frac{\sum_k^{} e^{- \lambda (C_{k,j} - u_k^{}-v_j) }{\sum_{k,l\neq j}  e^{- \lambda (C_{k,l} - u_k^{} -v_l) }}}{ {\left(\sum_{k,l}  e^{- \lambda (C_{k,l} -u_k^{}- v_l) } \right)}^2 }.
\end{split}
\end{equation*}
Computing the trace of the matrix $\H$, we obtain
\eq{
	\lVert\H\rVert\leq \Tr(\H)=\sum_i \H^{11}_{ii}+\sum_j\H^{22}_{jj}\leq 2\lambda N.
}

%\paragraph{Case 2} 
\bigskip
\noindent \textbf{\bf Case 2.}
We now consider $X,X'\in \cal V$.
In this case, we have for $X=(u,v)$:
\eq{
	\MK^*_{\lbd,\le N} (u,v)\hspace{-0.05cm}  =\hspace{-0.05cm}  N \dotp{Q_\lbd(u,v)}{\U}\hspace{-0.05cm}  =\hspace{-0.05cm} \frac{N}\lambda \sum_{i,j} e^{-1-\lambda (C_{i,j} -u_i - v_j)}.
}
As the second partial derivative with respect to $u_i$ reads $\partial_{u_i}^2\MK^*_{\lbd,\le N} (u,v)=N\lambda\sum_j e^{-1-\lambda (C_{i,j} -u_i - v_j)}$,
the trace of the Hessian matrix is: 
\begin{equation*}
\begin{split}
	\Tr(\H)
	&\hspace{-1pt}=\hspace{-2pt}N \lambda\hspace{-2pt}\left(\hspace{-1pt}\sum_{i,j} \hspace{-1pt}e^{-1-\lambda (C_{i,j} -u_i - v_j)}\hspace{-2pt}+\hspace{-2pt}\sum_{j,i}\hspace{-2pt} e^{-1-\lambda (C_{i,j} -u_i - v_j)}\hspace{-3pt}\right)
	\\
	& = 2\lambda N \dotp{Q_\lbd(u,v)}{\U} \leq 2\lambda N,
\end{split}
\end{equation*}
since $(u,v)\in \cal V$.

%\paragraph{Case 3}
\bigskip
\noindent \textbf{\bf Case 3.} We consider $X\in \cal U$ and $X'\in \cal V$. 
As $\cal V$ is a convex set and $U$ its complement, we denote as $Y$ the vector that lies in the segment $[X; X']$ and belongs to $\partial V$, the boundary of $\cal V$, % and $\cal V$. 
that is satisfying $\dotp{Q_\lbd(Y)}{\U} = 1$.
We thus have 
$\lVert X-Y \rVert+\lVert X'-Y\rVert=\lVert X-X'\rVert$ so that
\begin{equation}
\begin{split}
    & \hspace*{0mm} \norm{ \nabla \MK^{*}_{\lbd,\le N}(X) - \nabla \MK^{*}_{\lbd,\le N}(X') }
    \\
    & 		\quad \leq \norm{ \nabla \MK^{*}_{\lbd,\le N}(X)  - \nabla \MK^{*}_{\lbd,\le N}(Y)}  
    \\
    & 		\qquad + \norm{ \nabla \MK^{*}_{\lbd,\le N}(X') - \nabla \MK^{*}_{\lbd,\le N}(Y)} 
    \\
    & \quad \leq 2\lambda N (\norm{ X-Y} + \norm{X'-Y})
    	= 2\lambda N \norm{X-X'}
\end{split}
\end{equation}
which concludes the proof.

\end{proof}

\subsection{Proof of proposition \ref{prop:prox_bidual}}
\label{proof:prox_bidual}

\begin{proof}
We are interested in the proximity operator of $g$, which convex conjugate is
$
	g^*(q) = \frac{N}{\lambda} \dotp{e^{\lambda (q-c) - \U}}{\U}
$.
First, notice that the proximity operator of $g$ can be computed easily from the proximity operator of $g^*$ through Moreau's identity:
%$$\text{prox}_{\tau f} (p)= p-\tau \text{prox}_{ f^*/\tau} (p/\tau).$$
\eq{
	\Prox_{\tau g} (p)+\tau \Prox_{g^*/\tau} (p/\tau) = p 
	\quad \; \forall \, \tau >0, \forall \,p.
}
We now recall that the Lambert function $W$ is defined as:
$$
z = we^w \Leftrightarrow w = W(z)
$$
where $w$ can take two real values for $z\in ]-\frac1e,0]$, and only one on $]0,\infty[$, as illustrated in Figure \ref{fig:lambert}.
As $z$ will always be positive in the following, we do not consider complex values.
\begin{figure}[!htb]
	\centering
	\includegraphics[width=0.45\textwidth]{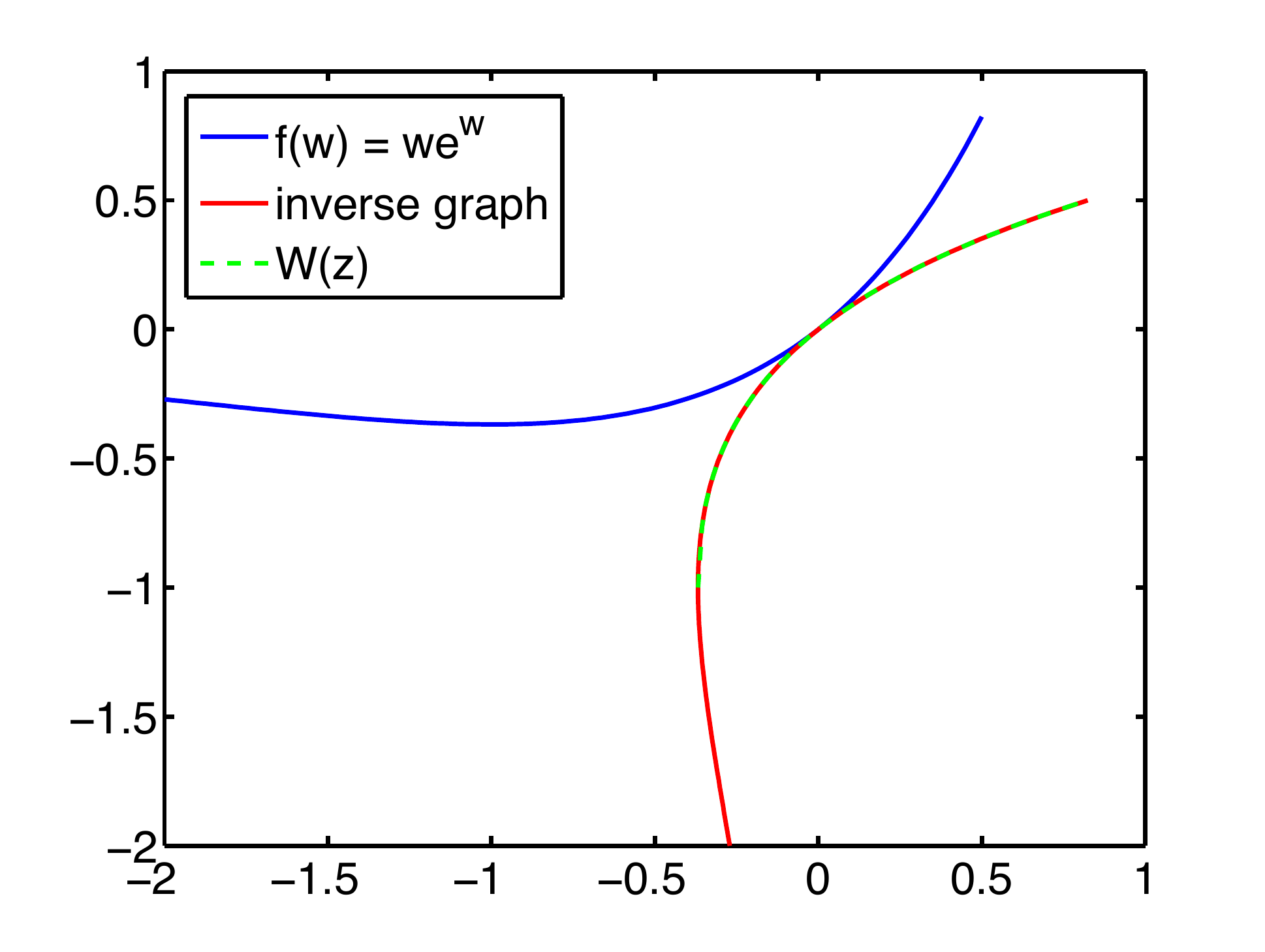}
	\caption{Graph of the Lambert function $W(z)$.\vspace*{5mm}~
	}
	\label{fig:lambert}
\end{figure}

The proximity operator of $g^*$ at point $p$ reads (as $g^*$ is convex, the $\Prox$ operator is univalued):
\eq{
\begin{split}
	\Prox_{\tau g^*}(p) & =q^\star \in  \uargmin{q} \tfrac{1}{2\tau}\norm{q-p}^2+g^*(q) 
	\\
	& =\uargmin{q} \sum_k^{}\tfrac{1}{2\tau}{(q_k^{}-p_k^{})}^2+\tfrac{N}{\lambda} e^{\lambda(q_k^{} - c_{k})-1}.
\end{split}
}
This problem is separable and can be solved independently $\forall \, k$. % \in [1,M^2]
Deriving the previous relation with respect to $q_k^{}$, the first order optimality condition gives:
\eq{
\begin{split}
&q^\star_k-p_k^{}+\tau N e^{\lambda(q^\star_k - c_{k})-1}=0
\\
\Leftrightarrow&(p_k^{}-q_k^*)e^{-\lambda q^\star_k}=\tau N e^{-\lambda c_{k}-1}
\\
\Leftrightarrow&\lambda(p_k^{}-q_k^*)e^{\lambda (p_k^{}-q_k^*)}
= \lambda\tau N e^{\lambda (p_k^{}-c_{k})-1}.
\end{split}
}
Using Lambert function, we get:
\eq{
\begin{split}
&\lambda(p_k^{}-q^\star_k)=W(\lambda\tau N e^{\lambda (p_k^{}-c_{k})-1})\\
\Leftrightarrow&q_k^*=p_k^{}-\frac1\lambda W(\lambda\tau N e^{\lambda (p_k^{}-c_{k})-1}).
\end{split}
}
%As $g^*/\tau$ is convex, the $\text{prox}$ operator is univalued, then only one possible value of the Lambert function is admissible, which is indeed the case since $z$ is strictly positive).
The proximity operator of $g^*/\tau$ thus reads
\eq{
\text{prox}_{g^*/\tau}(p) = p - \frac{1}{\lambda} W\left(  \tfrac\lbd\tau  N e^{\lambda (p-c) - 1}  \right),
}
hence
\eq{
\text{prox}_{\tau g}(p) = \frac{\tau}{\lbd} W\left(  \tfrac\lbd\tau  N e^{\lambda (\tfrac{p}\tau -c) - 1}  \right).
}

% which is in agreement with \cite{CombPesq} (Chapter 10, page 190, property xii). ça c'est en calculant d'abord le conjugé g de g*
% \cmt{Sinon, On pourrait directement invoquer \cite{CombPesq} et ce serait plié mais il faut pour ça calculer g ...}

\end{proof}